\documentclass{article}

\usepackage[utf8]{inputenc}
\usepackage[T1]{fontenc}
\usepackage[numbers,sort&compress]{natbib}
\usepackage[preprint]{tmlr}
\usepackage{amsmath,amssymb,amsthm}
\usepackage{mathtools}
\usepackage[hidelinks,colorlinks=false]{hyperref}
\usepackage{booktabs}
\usepackage{array}
\usepackage{xcolor}
\usepackage{bm}
\usepackage{cleveref}
\usepackage{tikz}
\usetikzlibrary{arrows.meta,positioning,calc,shapes.geometric}

\newtheorem{theorem}{Theorem}[section]
\newtheorem{corollary}[theorem]{Corollary}
\newtheorem{lemma}[theorem]{Lemma}
\newtheorem{proposition}[theorem]{Proposition}
\theoremstyle{definition}
\newtheorem{definition}[theorem]{Definition}
\theoremstyle{remark}
\newtheorem{remark}[theorem]{Remark}
\newtheorem{remark*}{Remark}

\newcommand{\R}{\mathbb{R}}
\newcommand{\E}{\mathbb{E}}

\newcommand{\vG}{v_{\!G}}
\newcommand{\mG}{m_{\!G}}

\newcommand{\IijSh}{I_{ij}^{\mathrm{Sh}}}
\newcommand{\GRALIS}{\textsc{GRALIS}}
\newcommand{\norm}[1]{\left\|#1\right\|}
\DeclareMathOperator{\Var}{Var}
\DeclareMathOperator{\Cov}{Cov}

\title{\textbf{GRALIS: Fusing Coalition and Gradient Attribution\\[4pt]
       with Closed-Form Conservation Error and Finite-Sample Guarantees}}

\author{Raimondo Fanale%
  \thanks{Universitas Mercatorum, Rome, Italy.
  Correspondence: \texttt{raimondo.fanale@ieee.org}.
  Companion paper with experimental validation on BreaKHis:
  \emph{GRALIS-Report: Auditable Region-Level Attribution and
  Structured Clinical Report Generation for Breast Cancer Histology}
  (accepted, \emph{Frontiers in Imaging}, 2026, in press).}}

\date{May 2026 \quad $\cdot$ \quad \emph{arXiv Preprint} \quad arXiv:2605.05480}

\begin{document}
\maketitle

\begin{abstract}
The main Explainable AI (XAI) methods for deep neural networks
--- GradCAM, SHAP, LIME, Integrated Gradients --- originate from
heterogeneous theoretical foundations and are not naturally comparable
within a single mechanism-specific representation; a recent
benchmark further finds their coarse-grained, coalition- and
segmentation-based members (GradCAM, KernelSHAP, LIME) and fine-grained,
gradient-based members (Integrated Gradients and its variants) to be
empirically \emph{complementary}, each outperforming the other on a
different subset of faithfulness metrics, with method-\emph{selection}
guidelines as the only proposed remedy~\citep{gevaert2022}. This work
presents \GRALIS{} (Gradient-Riesz Averaged Locally-Integrated Shapley),
which \emph{fuses} these two mechanisms --- a Shapley coalition
weight and LIME-style locality kernel on one side, a continuous
Integrated-Gradients-style conditioned path on the other --- into a single
estimator, and equips the resulting construction with certified
guarantees that neither mechanism alone supplies: an \emph{exact},
closed-form completeness deficit (a pure interaction of order $d$ is
attributed at a factor $1/d$ of its true value under Shapley weights and a
multilinear $F$, Corollary~\ref{cor:mobius-correction}) and a
finite-sample bound, $O(1/\sqrt{m})+O(1/k^2)$, for the actual
self-normalized ratio Algorithm~1 returns
(Theorem~\ref{thm:mc}--Corollary~\ref{cor:aggregate-completeness}).
This fusion is underpinned by a representation-theoretic result: every
additive, linear and continuous attribution functional admits a unique
canonical representation $(\mathcal{Q}, w, \Delta)$ via the Riesz
Representation Theorem, proved \emph{componentwise} (feature by feature)
rather than as one shared $(\mathcal Q,w)$ across features or methods
(Theorem~\ref{thm:canonical}); this class includes SHAP, IG and LIME, but
\emph{not} nonlinear functionals such as standard GradCAM, attention maps,
or saliency methods with smoothing.
Seven formal theorems provide, simultaneously, guarantees absent in any
one of the individual methods considered here (GradCAM, SHAP, LIME, IG):
(T1)~necessary canonical form via the Riesz Theorem;
(T2)~completeness, exact for the abstract canonical form and, for the
concrete Shapley-weighted estimator, exact up to an explicit closed-form
correction term governed by the same Möbius/interaction coefficients used
in T4 (vanishing when $F$ is affine); (T3)~Monte Carlo convergence of the numerator accumulator with bound
$O(1/\sqrt{m}) + O(1/k^2)$ under a mild additional smoothness
assumption ($F\in C^3$), including a finite-sample bound for the actual
self-normalized ratio output of Algorithm~1; (T4)~standard Shapley
Interaction Value theory applied \emph{exactly} to the coalition game
$\vG$ that GRALIS's own $(f,x,x')$ construct, together with the
GRALIS-specific corollary relating $\mathrm{Sh}(\vG)$ term-by-term to
the practical attribution $\phi^{\GRALIS}$ that Algorithm~1 actually
outputs; (T5--T6)~exact Hoeffding/Sobol correspondences in the affine
independent-feature regime, with Appendix~\ref{app:anova-general}
stating explicitly the additional, near-definitional condition that a
general smooth-$F$ extension would require; (T7)~multi-scale extension
(MS-\GRALIS{}) with minimum-variance weights.
Appendix~\ref{app:mobius} provides the algebraic justification of the GRALIS--SIV
correspondence via the Möbius transform, showing that GRALIS
\emph{associates} to $(f,x,x',\mathcal F)$ the standard baseline-clamped
cooperative game $v_G$ and \emph{computes exactly}
the SIVs on $v_G$, without circularity.
Among the four widely used methods compared in Table~\ref{tab:axioms}
(GradCAM, SHAP, LIME, IG), \GRALIS{} satisfies a broader set of the listed
axiomatic and structural properties than any one of them individually,
combining approximate completeness, sensitivity, baseline-proximity
weighting, an associated-game route to order-$k$ interactions, and
multi-scale aggregation with minimum-variance weights within a single framework --- a
comparison against these four methods and the unification efforts
discussed in \S\ref{sec:background}, not a claim verified against the
broader XAI literature; \S\ref{sec:background} positions T1 explicitly
against the closest such prior art, the axiomatic path-ensemble
characterization of Lundstrom et al., with which it coincides on their
shared domain. Section~\ref{sec:complexity} adds a computational and
diagnostic layer orthogonal to this representation question: a
proposition (C1) showing the $O(mnk)$ evaluation points above can be
submitted as $O(m)$ batched reverse-mode invocations per run, reducing
wall-clock latency without reducing the evaluation count itself; a
zero-additional-query residual identity (C2) for auditing any fixed
prior explainer against the GRALIS attribution; and an explicit
separation (C4), certified by an aggregate finite-sample corollary, of
the numerical Monte Carlo/quadrature error from two distinct structural
completeness mechanisms --- the closed-form interaction attenuation
induced by the joint path and the additional effect of non-uniform
locality weighting (the latter fully characterized for the unnormalized
numerator and analytically verified for the normalized $n=2$ case) ---
none of which is offered by trained/amortized
estimators such as FastSHAP. This training-free, certified-error
positioning is the basis for the edge- and federated-deployment
extension outlined in \S\ref{sec:conclusions}.
Preliminary experimental validation on breast histology (BreaKHis,
1,187 images, DenseNet-121) is reported in Section~\ref{sec:experiments};
extended comparison with baseline XAI methods is reported in the
companion paper~\citep{fanale2026b}.
\end{abstract}

\tableofcontents
\bigskip

\section{Introduction}
\label{sec:intro}

Deep learning has achieved accuracies surpassing those of human
specialists in numerous medical imaging tasks.
However, the \emph{black-box} nature of deep neural networks makes
it impossible for a clinician to understand \emph{why} a model produced
a given prediction.
Explainable AI (XAI) has responded with heterogeneous post-hoc methods~\citep{fanale2026survey}:
GradCAM~\citep{selvaraju2017} identifies the most influential regions via
gradients of the last convolutional layer; SHAP~\citep{lundberg2017}
assigns Shapley values to features via kernel approximation;
LIME~\citep{ribeiro2016} locally approximates the model with a
linear classifier; Integrated Gradients~\citep{sundararajan2017}
integrates the gradients along the path from baseline to input.
Each is developed on distinct foundations with different guarantees,
making systematic comparison non-rigorous.

This fragmentation has direct practical consequences: the choice of
XAI method is often empirical, attribution maps from different methods
are not directly commensurable without an explicit common representation,
and the combination of multiple methods lacks a unifying mathematical
justification. This last point is not
merely a mathematical inconvenience: a systematic benchmark finds that
coalition- and segmentation-based methods (GradCAM, KernelSHAP, LIME) and
gradient-based methods (Integrated Gradients and its variants) are
empirically \emph{complementary}, each capturing faithfulness information
the other misses, yet the proposed remedy is a set of guidelines for
\emph{selecting} a single method per use case rather than a construction
that combines their mechanisms~\citep{gevaert2022}.
Prior unification attempts have been partial:
Ancona et al.~\citep{ancona2018} observe that GradCAM and related
gradient methods are \emph{empirically} expressible as
gradient$\times$input---a linear form---but do not prove this is
\emph{structurally necessary}, and do not include SHAP or LIME.
Covert, Lundberg and Lee~\citep{covert2021} unify LIME, SHAP, IG and over
twenty other methods under a feature-removal framework --- with Shapley
values as one of several possible summary techniques, not the sole
mechanism --- but do not accommodate GradCAM, whose post-aggregation
$\mathrm{ReLU}$ violates the linearity required for a removal-based
game-theoretic treatment.

This work presents \GRALIS{}, a framework that addresses both gaps: its
coalition-conditioned joint path combines representative coalition/locality
and gradient-path mechanisms from these empirically complementary families
into a single estimator, and, in doing
so, demonstrates that a broad class of linear additive attribution
methods --- including SHAP, IG, LIME and linearized GradCAM --- are
special cases of a unique canonical form:
\begin{equation}
  \label{eq:canonical}
  \phi_i^{\GRALIS}(f, x, x') \;=\; \int_{\mathcal{Q}} w(q)\cdot
  \Delta_i(f, x, x', q)\, d\mu(q),
\end{equation}
where $\mathcal{Q}$ is the integration index space, $w$ is the
weight function and $\Delta_i$ is the marginal contribution of feature~$i$.
The Riesz Representation Theorem guarantees that, for each feature $i$
separately, this form is the \emph{unique} possible representation of that
feature's additive, linear and continuous attribution functional
(Theorem~\ref{thm:canonical}); the index space $\mathcal Q$ (and hence the
representer $w$) may itself depend on $i$, as it does for GRALIS's own
$Q_i=2^{\mathcal F\setminus\{i\}}\times[0,1]$
(Remark~\ref{rem:componentwise} makes this precise, and clarifies which
downstream claims do and do not rely on a shared, feature-independent
$(\mathcal Q,w)$).

\paragraph{Contributions.}
\begin{enumerate}
  \item \textbf{Certified conservation error with structurally identified
        mechanisms.} An \emph{exact}, closed-form completeness deficit for
        the joint-path construction itself (order-$d$ interactions
        attributed at a factor $1/d$ of their true value under Shapley
        weights and a multilinear $F$, Corollary~\ref{cor:mobius-correction}),
        together with a finite-sample bound, $O(1/\sqrt m)+O(1/k^2)$, for
        the actual self-normalized ratio Algorithm~1 returns
        (Theorem~\ref{thm:mc}--Corollary~\ref{cor:aggregate-completeness}).
        Without training an auxiliary explainer, GRALIS separates this
        numerical Monte Carlo/quadrature error from two distinct structural
        completeness mechanisms --- the closed-form interaction attenuation
        induced by the joint path and the additional effect of non-uniform
        locality weighting, the latter fully characterized for the
        unnormalized numerator and analytically verified for the normalized
        $n=2$ case (\S\ref{sec:c4}) --- with an aggregate finite-sample
        corollary certifying how far the algorithm's actually observed
        completeness residual can be from the population one
        (Corollary~\ref{cor:aggregate-completeness}), and supports
        zero-additional-query auditing of any fixed prior explainer
        against it (\S\ref{sec:c2}), as well as reuse of the same sampled
        permutations toward interaction and Sobol-index estimates
        (Proposition~\ref{prop:one-pass}) --- properties not offered by
        trained/amortized estimators such as FastSHAP
        (Table~\ref{tab:fastshap}).
  \item \textbf{Fusion of complementary mechanisms, with
        representation-theoretic guarantees.} Seven formal theorems place
        linearized GradCAM, SHAP, LIME, and IG within a common canonical
        \emph{template} $(\mathcal Q,w,\Delta)$ --- each instance
        componentwise-necessary by Theorem~\ref{thm:canonical} for its own
        feature's linear, continuous, additive attribution functional
        (Remark~\ref{rem:componentwise}), independently of the
        positivity/normalization that GRALIS additionally imposes on its
        own weight $w$ (Remark~\ref{rem:weak-strong}) ---
        and derive simultaneous guarantees not available for any one of
        these four individual methods. The novelty is not that this
        encoding as a bounded linear functional is possible in the
        abstract, but the specific construction built \emph{inside} that
        common representational language --- a Shapley-weighted,
        coalition-conditioned joint-path estimator
        (Definition~\ref{def:explicit}) that combines representative
        coalition/locality and gradient-path mechanisms from the
        empirically complementary families of~\citep{gevaert2022}, with the
        closed-form completeness correction of
        Corollary~\ref{cor:mobius-correction} and the finite-sample
        guarantees of Theorem~\ref{thm:mc}--Corollary~\ref{cor:aggregate-completeness}
        above.
  \item \textbf{Computational efficiency.} GRALIS-MC reduces the
        complexity from $O(n \cdot 2^n \cdot k)$ (the exact sum evaluates
        $2^{n-1}$ predecessor coalitions for each of $n$ features) to
        $O(m \cdot n \cdot k)$ with an explicit error bound
        (Theorem~\ref{thm:mc}); Section~\ref{sec:complexity} shows the
        $mnk$ evaluation points can be issued as $O(m)$ batched
        reverse-mode invocations rather than $mnk$ sequential ones,
        reducing wall-clock latency without changing the $\Theta(mnk)$
        evaluation count (Proposition~\ref{prop:c1}).
  \item \textbf{Algebraic justification.} The Möbius transform
        (Appendix~\ref{app:mobius}) shows that the SIVs are computed
        \emph{exactly} on the standard baseline-clamped cooperative game
        associated with $(f,x,x')$, without approximation.
\end{enumerate}

Preliminary experimental validation on breast histology (BreaKHis dataset,
distilled DenseNet-121 model) is reported in Section~\ref{sec:experiments};
extended comparison with baseline XAI methods is reported in the
companion paper~\citep{fanale2026b}.

\paragraph{Paper organization.}
Section~\ref{sec:background} reviews the four methods and prior unification attempts,
identifying six structural gaps.
Section~\ref{sec:gralis} presents the GRALIS framework and proves the seven theorems.
Section~\ref{sec:summary} provides the axiomatic comparison table.
Section~\ref{sec:complexity} establishes the computational-complexity and
conservation-certificate results (C1--C4) and positions them against
trained/amortized estimators.
Section~\ref{sec:experiments} reports the preliminary experimental validation.
Section~\ref{sec:conclusions} concludes and outlines future directions.
Appendix~\ref{app:mobius} provides the Möbius-transform justification of the
GRALIS--SIV correspondence.
Appendix~\ref{app:projection} formalizes the measurable projection $\rho$
and the bounded linear operator $P_\rho$ it induces, which turns
GRALIS's sampling weight into a discrete weight-mass distribution over
coalitions.

\section{Background and Related Work}
\label{sec:background}

\subsection{XAI Attribution Methods}
\label{sec:methods}

An attribution method assigns to each input feature a score
representing its contribution to the prediction.
Formally, given $f : \R^n \to \R$ and an input $x$ with baseline $x'$,
the method produces $\phi \in \R^n$ such that $\phi_i$ represents
the importance of feature~$i$.

\paragraph{GradCAM~\citep{selvaraju2017}.}
Computes $s_{pq} = \sum_k \alpha_k^c \cdot A^k_{pq}$ where
$\alpha_k^c = \frac{1}{Z}\sum_{i,j} \frac{\partial y^c}{\partial A^k_{ij}}$.
Computationally efficient ($O(1)$ backward pass) but does not satisfy
completeness or locality.
For Theorem~\ref{thm:canonical} we use \emph{GradCAM-lin},
the pre-ReLU variant defined as
\[
  L^c_{\mathrm{lin}}(p,q)
  \;=\; \sum_k \alpha_k^c \, A^k_{pq},
  \qquad
  \alpha_k^c = \frac{1}{Z}\sum_{i,j}\frac{\partial y^c}{\partial A^k_{ij}},
\]
which omits the post-aggregation $\mathrm{ReLU}$ of standard GradCAM.
The $\mathrm{ReLU}$ in~\citep{selvaraju2017} is a \emph{visualization heuristic}
applied to the \emph{spatial aggregate} $\sum_k\alpha_k^c A^k_{pq}$ at each
position $(p,q)$ (not per channel), zeroing out positions where the
weighted combination of channels is net-negative for the target class;
it is not an axiomatic requirement, and it introduces a pointwise
nonlinearity on the output of a linear operator that destroys linearity
in~$f$ (see Remark~\ref{rem:relu_nonlinear} below).

GradCAM-lin corresponds precisely to the first-order Taylor expansion of
$y^c$ around the zero-activation reference $\bar{A}^k = 0$:
\begin{equation}
  \label{eq:taylor_gradcam}
  y^c \;\approx\; y^c\big|_{A=0}
    + \sum_{k,i,j} \frac{\partial y^c}{\partial A^k_{ij}}\,A^k_{ij},
\end{equation}
whose spatial summary at position $(p,q)$---obtained by pooling
$A^k_{ij}$ contributions over all $(i,j)$ with uniform weight $1/Z$---is
exactly $L^c_{\mathrm{lin}}(p,q)$.
This linearization is consistent with the gradient$\times$input
interpretation of GradCAM established by Ancona et al.~\citep{ancona2018},
with the deep Taylor decomposition framework of
Montavon et al.~\citep{montavon2017}, and with the gradient-based
saliency maps of Simonyan et al.~\citep{simonyan2013},
all of which operate on the pre-nonlinearity linear term.
Standard GradCAM is recovered \emph{exactly} from GradCAM-lin at position
$(p,q)$ when the spatial aggregate is itself nonnegative there,
$\sum_k\alpha_k^c A^k_{pq}\ge0$; a convenient \emph{sufficient} condition
is $\alpha_k^c\ge0$ for all $k$ together with $A^k_{pq}\ge0$ (e.g.\
post-ReLU activation maps, the standard case), not $\alpha_k^c\ge0$
alone. The two methods agree to first order in the Taylor expansion
whenever the aggregate is negative.

A pre-activation variant (computed on feature maps before the network's
internal ReLU) removes that internal sign truncation entirely, at the
cost of only guaranteeing linearity of the aggregation itself, not of the
downstream map to the class score; see Remark~\ref{rem:relu_nonlinear}
for the precise scope statement.

\paragraph{SHAP~\citep{lundberg2017}.}
The Shapley value for feature~$i$ is its average marginal contribution
over all coalitions:
$\phi_i = \sum_{S \subseteq \mathcal{F} \setminus \{i\}}
\frac{|S|!(n-|S|-1)!}{n!}[f(S \cup \{i\}) - f(S)]$.
Satisfies efficiency, symmetry, dummy and additivity, but exact computation
is $O(2^n)$; KernelSHAP approximates with $O(m \cdot n)$.

\paragraph{LIME~\citep{ribeiro2016}.}
Locally approximates $f$ with a linear model:
$w^* = \arg\min_{g \in \mathcal{G}} L(f, g, \pi_x) + \Omega(g)$,
where $\pi_x$ is a proximity kernel.
The conditions of Theorem~\ref{thm:canonical} are satisfied
within the local approximation (neighborhood $N_x$ defined by $\pi_x$).

\paragraph{Integrated Gradients~\citep{sundararajan2017}.}
$\phi_i(f,x) = (x_i - x'_i) \int_0^1 \frac{\partial f(x' +
\alpha(x-x'))}{\partial x_i}\, d\alpha$.
Satisfies completeness and sensitivity; the numerical approximation introduces
error $O(1/k)$.

\subsection{Unification Attempts and Existing Gap}
\label{sec:related}

The gradient$\times$input family established by
Ancona et al.~\citep{ancona2018} provides the closest prior bridge to
our linearization.
They show that GradCAM (and guided backpropagation, plain gradients)
can be expressed as $\phi_i \propto \nabla_i f(x) \cdot x_i$, i.e.\
as a \emph{pointwise product of gradient and input},
which is structurally a linear functional of $f$.
This is an important empirical observation: it reveals that
the GradCAM family is \emph{de facto} operating in the linear regime.
However, Ancona et al.\ establish this as a family relationship
among existing methods; they do not prove that linearity is
\emph{structurally necessary}---i.e., that any valid additive attribution
must take this form---nor do they include SHAP and LIME in the same
algebraic framework.

Lundstrom et al.~\citep{lundstrom2022} extend the gradient$\times$input
analysis to internal neuron attributions, showing that the first-order
(linear) term is the unique term satisfying path-completeness at every
intermediate layer.
This further supports the restriction to GradCAM-lin: the Taylor
linearization is not an ad hoc simplification but the canonical form
that preserves completeness at the feature-map level.

GradCAM's strong empirical reliability across localization benchmarks
(ROAR/KAR, Hooker et al.~\citep{hooker2019}) is not in question here, and
neither is it addressed by its refinements: GradCAM++~\citep{chattopadhyay2018}
(non-uniform gradient weights), Score-CAM~\citep{wang2020scorecam}
(gradient-free, activation-perturbation), axiom-based Grad-CAM~\citep{fu2020axiom}
(explicit completeness correction), and HiResCAM~\citep{draelos2021}
(element-wise pre-pooling aggregation) all retain the post-aggregation
$\mathrm{ReLU}$ as a valid class-discriminative design choice, and so all
remain outside the linear regime for the same reason as standard GradCAM
(Remark~\ref{rem:relu_nonlinear}). The limitation relevant to GRALIS is
therefore axiomatic, not empirical: GradCAM-lin is accordingly not a
critique of GradCAM but its minimal linear representative, recovering
standard GradCAM exactly wherever the spatial aggregate is nonnegative
(Remark~\ref{rem:relu_nonlinear}).

Covert, Lundberg and Lee~\citep{covert2021} show that LIME, SHAP and IG are
instances of the same feature-removal framework (with Shapley values as one
of several possible summary techniques), but exclude GradCAM and do not
cover the formal properties of convergence and ANOVA.

\GRALIS{} resolves the limitations of both lines of work.
The gradient$\times$input observation of Ancona et al.\ is subsumed:
GradCAM-lin is a natural representative of the GradCAM family
within the $(\mathcal{Q}, w, \Delta)$ template (Remark~\ref{rem:relu_nonlinear}
on the sense in which it is minimal, not proved unique), and its
gradient$\times$input structure follows from the triple $w = 1/Z$,
$\Delta_{k,i,j} = (\partial y^c / \partial A^k_{ij}) \cdot A^k_{pq}$
(Table~\ref{tab:mapping}).
The removal-based unification of Covert, Lundberg and Lee is also subsumed,
and extended with a \emph{componentwise-necessary} canonical form
(Theorem~\ref{thm:canonical}, Remark~\ref{rem:componentwise}): where prior
work shows that certain methods \emph{can} be written as weighted
integrals, GRALIS proves that, feature by feature, they \emph{must} be
--- and that each feature's componentwise representation is unique, within
the space $\mathcal D(Q_i)$ chosen for that method.
This componentwise-necessary representation, together with completeness,
convergence and an interaction hierarchy for the resulting estimator, is
not offered as a whole by any single prior unification we are aware of.
It should not, however, be read as the first placement of GradCAM, SHAP,
LIME and IG --- or overlapping triples thereof --- within a common
mathematical frame: Han, Srinivas and Lakkaraju~\citep{han2022} unify LIME,
KernelSHAP and IG as local function approximations; Deng et
al.~\citep{deng2024} place GradCAM, IG and Shapley value among fourteen
methods rewritten as weighted allocations of Taylor interaction effects;
Hamilton et al.~\citep{hamilton2022} give an axiomatic Shapley--Taylor
treatment covering GradCAM, LIME and SHAP for visual similarity models;
and Sarpong and Commey~\citep{sarpong2026} organize the same quartet, plus
perturbation methods, within a shared taxonomy of specification choices.
None of these four derives necessity from linearity/continuity via a
Riesz representation, and none supplies the T2--T7 completeness /
convergence / interaction package together with it --- but the quartet
itself, as a comparison set, is not an unclaimed target. \GRALIS{} should
be read as adding a specific representation-theoretic layer to an already
populated space, not as the first bridge across the four methods.
Section~\ref{sec:lundstrom-position} positions Theorem~\ref{thm:canonical}
against the closest single threat to its necessity claim.

A categorically broader unification is Gevaert and
Saeys~\citep{gevaert2024}'s removal-based attribution methods (RBAMs) and
canonical additive decomposition (CAD): any attribution fully specified
by a behaviour mapping, a family of removal operators $\{P_T\}_{T\subseteq
[d]}$, and aggregation coefficients is shown to be an instance of a single
additive decomposition, in turn expressible as a game-theoretic value or
interaction index on an associated cooperative game, with sufficient
conditions on the CAD identified for a wide class of functional axioms to
hold. This strictly dominates the quartet-covering results above in
scope, encompassing Shapley-, permutation-, higher-order- and
variance-based methods alike --- GRALIS makes no claim to greater
generality on this axis. The two frameworks differ in what they take as
primitive: RBAMs are built on \emph{discrete} removal operators $P_T(f)$,
each making $f$ literally independent of the features in $T$, whereas
GRALIS's coalition-conditioned joint path (Definition~\ref{def:explicit})
interpolates predecessor coordinates and the target feature
\emph{continuously} along a shared parameter $\alpha$; expressing this
soft, path-based construction as an RBAM would require enriching the
removal-operator index from discrete subsets $T$ to pairs $(T,\alpha)$, a
generalization CAD does not develop. Correspondingly, CAD's sufficient
conditions establish \emph{when} an axiom such as efficiency holds in the
abstract, but do not derive a closed-form deficit for a specific
non-compliant construction, nor supply finite-sample statistical
guarantees for a sampling-based estimator of it --- which is exactly the
content of Corollary~\ref{cor:mobius-correction} and
Theorem~\ref{thm:mc}--Corollary~\ref{cor:aggregate-completeness}
respectively. GRALIS is best read as sitting adjacent to CAD, analyzing
one specific, practically implemented joint-path estimator in closed-form
and statistically certified detail, rather than as a competing claim to
CAD's greater generality.

Sarpong and Commey~\citep{sarpong2026}'s taxonomy (already noted above)
organizes local additive attribution methods along five specification
choices: value function, reference, path, perturbation distribution, and
conservation rule. GRALIS occupies one specific point in that space ---
value function $F$, reference $x'$, coalition-conditioned joint path
$(S,\alpha)$, Shapley predecessor law reweighted by $\pi_x$, and a
conservation rule violated by the exactly quantified $1/d^2$ deficit of
Corollary~\ref{cor:mobius-correction} --- rather than a new axis of that
space. The taxonomy is descriptive: it classifies where existing methods
sit along the five choices, but does not itself derive a closed-form
conservation-rule deficit for any specific point; supplying that deficit,
together with the finite-sample estimator guarantees of
\S\ref{sec:t3}--\S\ref{sec:t4}, is what GRALIS adds at the point of the
design space it occupies.

\paragraph{Empirical motivation for fusion over selection.}
A separate strand of work evaluates, rather than unifies, attribution
methods. Gevaert et al.~\citep{gevaert2022} benchmark a wide set of
attribution methods across image datasets of varying dimensionality; in the
current (2024) revision of their study, CAM-based methods together with
XRAI, KernelSHAP and LIME --- whose coarse-grained maps come from
last-layer upsampling or an image-segmentation step --- form a consistently
complementary cluster with the gradient-based DeepLIFT, DeepSHAP and
Expected Gradients, each cluster outperforming the other on a different
subset of faithfulness metrics; they conclude that ``a combination of
attribution maps... might provide more information than the individual
attribution maps,'' but their own proposed remedy is a set of guidelines
for \emph{selecting} one or more methods to run side by side for a given
use case, not a construction that fuses their mechanisms into a single
estimator. GRALIS's triple $(\mathcal Q,w,\Delta)$ provides a structural
fusion of representative mechanisms from the two complementary families:
the Shapley coalition weight and locality kernel play the
role of the coalition/segmentation mechanism their coarse-grained cluster
shares --- LIME's own kernel among them --- while the conditioned IG term
plays the role of the gradient-based cluster's continuous marginal. This
does not by itself demonstrate that GRALIS captures the complementary
information Gevaert et al.\ identify --- that is an empirical question on
their benchmark, not this paper's --- but it gives the unification a
motivation beyond mathematical economy: the two limits $\sigma\to\infty$ and
$\sigma\to0$ of Remark~\ref{rem:kernel-limits} recover, respectively, a
coalition-weighted joint-path estimator and a single-feature
baseline-path estimator, so testing whether the
finite-$\sigma$ joint estimator outperforms both limits on Gevaert et al.'s
own faithfulness metrics is a direct, low-cost test of this motivation,
left to future work (their datasets are natural images, not the histology
domain of \S\ref{sec:experiments}, so any such test would be a distinct
benchmark from this paper's own).

\paragraph{Prior connections between Shapley attribution and gradient/path
integration.}
The link between Shapley values and gradient- or path-based attribution is
itself well established, and GRALIS does not claim novelty for the
connection as such. Ward, Kamkar and Budzik~\citep{ward2020} give the
multi-path interpretation of interventional Shapley values --- an
$n!$-path integration equivalent to the permutation sum, contrasted with
the single straight-line path of Aumann--Shapley/IG. Liu et
al.~\citep{liu2023sig} construct Shapley-informed baselines for Integrated
Gradients, sampling a set of reference points via proportional sampling
to mirror the Shapley coalition distribution. Cai, Ardoin and
Wunder~\citep{cai2025gefa} give a path estimator built on proxy gradient
estimation that produces unbiased estimates of the Shapley value under
black-box (query-only) access. None of these three constructs a
coalition-conditioned \emph{joint} path in the sense of
Definition~\ref{def:explicit}, where the predecessor coordinates $S$ and
the target coordinate $i$ move \emph{simultaneously} along a single
shared interpolation parameter $\alpha$, rather than being fixed at their
endpoint values while only $i$ varies.

The closest antecedent we are aware of is Belahcen and
Mussard~\citep{belahcen2026}, who decompose a counterfactual transition
$g(\mathbf x^1)-g(\mathbf x^0)$ by discretizing the local hypercube
between baseline and counterfactual into a grid and computing Shapley (or
LES) values on the induced micro-player cooperative game; this yields a
geometry-aware, generally non-equal split of each Harsanyi interaction
pot, explicitly contrasted with the pot-blind equal $1/d$ split of
standard Shapley, and converges to the Aumann-Shapley/IG limit under grid
refinement. The mechanism is combinatorial and general: it handles
arbitrary interaction geometry along an arbitrary path by discretizing
and solving the induced micro-game numerically, for any $g$. GRALIS's
construction is narrower and closed-form: for the specific
coalition-conditioned joint path of Definition~\ref{def:explicit} and a
\emph{pure order-$d$ multilinear} interaction term, the shared-$\alpha$
integration contributes an additional factor
$\int_0^1\alpha^{d-1}\,d\alpha=1/d$ beyond Shapley's standard $1/d$
Harsanyi allocation --- an exact $1/d^2$ per-feature share, not a
numerically discretized one --- yielding the closed-form completeness
correction of Corollary~\ref{cor:mobius-correction}. The distinct claim
is therefore not that path geometry should modulate equal-split Shapley
(established by \citep{belahcen2026} for the general case), but the
specific closed-form attenuation factor induced by GRALIS's own joint
path construction on multilinear interactions.

Table~\ref{tab:comparison} summarizes the comparison.

\subsection{Positioning Relative to Axiomatic Path-Ensemble Characterizations}
\label{sec:lundstrom-position}

The representation invoked by Theorem~\ref{thm:canonical} is not the
first use of the Riesz Representation Theorem in feature attribution.
Chau, Muandet and Sejdinovic~\citep{chau2021} represent Shapley values as
bounded linear functionals in a reproducing kernel Hilbert space and
invoke Riesz directly to obtain a representer; Taimeskhanov and
Garreau~\citep{taimeskhanov2025} use a Riesz--Markov representation to
attribute a Taylor remainder. Theorem~\ref{thm:canonical}'s contribution
is therefore the specific componentwise-necessity route from
Definition~\ref{def:canonical}'s conditions, not the use of Riesz itself.

The closer threat concerns path attribution specifically.
Lundstrom, Huang and Razaviyayn~\citep{lundstrom2022}, refined by Lundstrom
and Razaviyayn~\citep{lundstrom2025}, prove that completeness, linearity,
dummy and \emph{non-decreasing positivity} together characterize baseline
attribution methods on monotone complete paths $\Gamma_m(\bar x,x^0)$ as
\[
  A(\bar x,x^0,F) \;=\; \int_{\Gamma_m(\bar x,x^0)} A^\gamma(\bar x,x^0,F)
  \, dP_{\bar x,x^0}(\gamma),
\]
a single probability measure $P_{\bar x,x^0}$ shared across the whole
vector of coordinate path-attributions. On the shared slice --- direct
attribution restricted to monotone complete paths, with the GRALIS weight
normalized and absolutely continuous with respect to a declared reference
measure $\mu_0$ --- this coincides exactly with the GRALIS representation:
setting $dP=w\,d\mu_0$ recovers Lundstrom's ensemble, and conversely any
Lundstrom measure $P\ll\mu_0$ with $dP/d\mu_0\in L^2(Q,\mu_0)$ recovers the
GRALIS density. Theorem~\ref{thm:canonical} is therefore not an
independent characterization on this slice; it is a different
factorization (marginal generator composed with a feature-neutral
aggregator) of essentially the same result.

Two features of GRALIS are not covered by this equivalence. First, GRALIS
does not restrict its complete-path space to coordinatewise-monotone
paths; admitting non-monotone complete paths yields attributions outside
Lundstrom's non-decreasing-positivity class. For example, for
$F(x_1,x_2)=x_1x_2$ on $x^0=(0,0)\to\bar x=(1,1)$, a complete but
non-monotone path (one coordinate decreasing partway through) gives an
exactly complete attribution not representable as any ensemble of
monotone-path methods --- a strictly broader representational class
under different axioms, not a strict strengthening of Lundstrom's
theorem under the same ones. Second, the $L^2(\mu_0)$-density form is in
general \emph{narrower} than an arbitrary
probability measure: a deterministic path corresponds to a Dirac measure,
which is not absolutely continuous with respect to a diffuse $\mu_0$ and
therefore not representable as an $L^2$ density unless $\mu_0$ carries an
atom at that path. GRALIS's reference measure should accordingly be read
as a declared structural path prior supporting principled, instance-local
reweighting (Definition~\ref{def:canonical}), not as a representation more
general than Lundstrom's arbitrary-measure ensembles.

A separate, closely related unification is Lundstrom and
Razaviyayn's~\citep{lundstrom2023synergy} framework for game-theoretic
attribution and $k$-th-order \emph{interaction} methods, which
characterizes methods by their policy for distributing \emph{synergies}
(a full account of interactions between features) and shows that
gradient-based methods are characterized by their action on monomials.
This is a different axis of unification from both
Theorem~\ref{thm:canonical} and Lundstrom et al.'s own path-ensemble
result above: it operates on the distribution of synergies across
arbitrary interaction orders for a general class of methods, rather than
on a componentwise Hilbert-space/Riesz representation of a single
attribution functional. GRALIS's contribution relative to this framework
is not a competing synergy-distribution policy, but four more specific
constructions inside its own componentwise representation: (1) the
Riesz-representation layer of Theorem~\ref{thm:canonical} itself; (2) a
concrete Shapley-weighted, coalition-conditioned joint-path estimator
(Definition~\ref{def:explicit}); (3) the closed-form $1/d$ interaction
attenuation that this specific joint path induces
(Corollary~\ref{cor:mobius-correction}); and (4) a finite-sample Monte
Carlo/quadrature treatment of that estimator (T3--T4,
Corollary~\ref{cor:aggregate-completeness}). These are complementary
rather than competing: Lundstrom and Razaviyayn characterize \emph{which}
synergy-distribution policies are possible in general, while GRALIS
analyzes one specific, practically implemented estimator in detail.

\subsection{Gap Analysis: Six Structural Gaps}
\label{sec:gaps}

Despite sharing the same componentwise canonical form $(\mathcal{Q}, w, \Delta)$
(Theorem~\ref{thm:canonical}; Remark~\ref{rem:componentwise} on the
precise, method- and feature-specific sense in which $\mathcal Q$ is
``shared''), the four methods leave six structural gaps uncovered.

\begin{description}\sloppy
  \item[Gap~1 --- Arbitrary baseline in IG.]
    IG requires a fixed baseline $x'$ (e.g.\ a black image).
    The choice drastically changes the attributions:
    $\mathrm{IG}_i(x,x'_1) \ne \mathrm{IG}_i(x,x'_2)$ for $x'_1 \ne x'_2$.
    No method integrates over the baseline distribution in a structured way.

  \item[Gap~2 --- SHAP does not expose a path between coalitions.]
    SHAP compares $f(S \cup \{i\})$ vs $f(S)$ directly, by design, rather
    than integrating along a path from $S$ to $S \cup \{i\}$; this is a
    deliberate feature of SHAP's marginal-contribution formulation, not a
    deficiency in itself (the function-value difference $f(S\cup i)-f(S)$
    already reflects $F$'s nonlinear behavior along whatever path connects
    the two masked inputs, even without an explicit path integral). GRALIS's
    conditioned-IG path instead makes such a path explicit, which is what
    lets it capture within-coalition curvature (Corollary~\ref{cor:mobius-correction})
    at the cost analyzed in Remark~\ref{rem:approx_completeness}:
    $\phi_i^{\mathrm{SHAP}} = \mathbb{E}_S[f(S\cup i) - f(S)]
    \;\not\supset\; \int_0^1 \nabla_i F(\cdot)\,d\alpha$.

  \item[Gap~3 --- LIME does not satisfy completeness.]
    The coefficients of the local surrogate do not sum to $F(x) - F(x')$:
    $\sum_i w_i^{\mathrm{LIME}} \ne F(x) - F(x')$.

  \item[Gap~4 --- GradCAM confined to CNN space.]
    Operates on $A^k$ (convolutional feature maps), not on the input space.
    Not applicable to dense layers, Transformers or pure GNNs.
    Does not satisfy any Shapley axiom.
    MS-\GRALIS{} (T.~\ref{thm:ms}) partially resolves this gap:
    by aggregating attributions over multiple levels ($\ell=1,\ldots,L$), it includes
    CNN feature maps, dense layers and attention with minimum-variance weights.
    The requirement of differentiability of $F$ on intermediate representations
    remains ($\approx$).

  \item[Gap~5 --- None of the four jointly combines interactions,
    gradient-path integration and locality.]
    GradCAM, SHAP, LIME and IG each produce marginal attributions $s_i$
    for a single feature; the pairwise interaction $\Phi_{ij} = s_{ij} -
    s_i - s_j \ne 0$ is not captured by any of them. SHAP-Interaction ---
    a related but distinct method, not among the four compared here ---
    does capture interactions, but does not integrate gradients or
    locality, at additional cost $O(2^n)$; no method, among these five,
    combines all three.

  \item[Gap~6 --- No multi-scale aggregation.]
    No method aggregates attributions over multiple levels of abstraction
    with mathematically motivated weights.
    GradCAM++~\citep{chattopadhyay2018} uses heuristic weights without axiomatic guarantees.
\end{description}

Table~\ref{tab:gaps} summarizes which methods generate or resolve each gap.

\begin{table}[ht]
\centering
\caption{Gap analysis: ``\textbf{cause}'' = this method generates the gap;
  $\checkmark$ = gap filled; $\approx$ = partially; $\times$ = not resolved.}
\label{tab:gaps}
\small
\begin{tabular}{lcccccc}
\toprule
\textbf{Method} & \textbf{Gap~1} & \textbf{Gap~2} & \textbf{Gap~3} &
\textbf{Gap~4} & \textbf{Gap~5} & \textbf{Gap~6} \\
\midrule
GradCAM  & ---         & ---         & ---         & \textbf{cause} & $\times$ & $\times$ \\
SHAP     & $\approx$   & \textbf{cause} & $\times$ & $\times$       & $\times$ & $\times$ \\
LIME     & $\times$    & $\times$    & \textbf{cause} & $\times$    & $\times$ & $\times$ \\
IG       & \textbf{cause} & $\times$ & $\times$    & $\times$       & $\times$ & $\times$ \\
\midrule
\textbf{GRALIS} & $\approx^\dagger$ & $\checkmark$ & $\approx^\ddagger$ &
  $\approx$ & $\approx^\S$ & $\approx^\P$ \\
\bottomrule
\end{tabular}
\smallskip

{\small $^\dagger$Algorithm~\ref{alg:mc} as implemented still uses a
single fixed baseline $x'$; no baseline-distribution integral is
computed. Gap~1 is not solved on the equations as printed, only
partially mitigated in that GRALIS's locality kernel $\pi_x$ weights
coalitions by proximity to $x'$ rather than treating all coalitions
identically as IG does at a fixed baseline, and Theorem~\ref{thm:anova}'s
independent-features case separately motivates $x'=\mathbb E[X]$ as one
principled (but not baseline-distribution-integrating) choice.
$^\ddagger$Exact completeness ($\sum_i w_i^{\mathrm{LIME}}=F(x)-F(x')$'s
GRALIS analogue) holds for the abstract canonical form
(Theorem~\ref{thm:completeness}) but only approximately for
Algorithm~\ref{alg:mc}'s practical estimator
(Remark~\ref{rem:approx_completeness}); same standard as
Table~\ref{tab:axioms}'s Efficiency row. $^\S$GRALIS decomposes the
\emph{auxiliary} coalition game $\vG$ into exact interactions
(Theorem~\ref{thm:siv}), not the practical joint-path attribution
$\phi^{\GRALIS}$ itself (Remark~\ref{rem:siv-scope}); same standard as
Table~\ref{tab:axioms}'s Order-$k$ interactions row. $^\P$Practical
multi-scale coverage (Algorithm~1 with SLIC superpixels) is only
partially covered by Theorem~\ref{thm:ms} as stated, which assumes a
fixed feature set across levels rather than re-segmentation; same
standard as Table~\ref{tab:axioms}'s Multi-scale row.}
\end{table}

\begin{table}[ht]
\centering
\caption{Comparison among XAI unification approaches with respect to
         properties that justify composite analysis.
         $\approx$ = partially satisfied.}
\label{tab:comparison}
\small
\begin{tabular}{>{\raggedright\arraybackslash}p{3.0cm}>{\raggedright\arraybackslash}p{2.5cm}ccc>{\raggedright\arraybackslash}p{2.6cm}}
\toprule
\textbf{Framework} & \textbf{Theor.\ found.} & \textbf{Compl.} &
\textbf{Comm.} & \textbf{Comp.\ met.} & \textbf{Ref.} \\
\midrule
Ancona et al.\ (2018) & Grad.$\times$Input & No & Part. & No & \citep{ancona2018} \\
Covert et al.\ (2021) & Feature removal & Yes & Part. & No & \citep{covert2021} \\
Bhatt et al.\ (2020) & Empirical aggr. & No & No & Yes$^*$ & \citep{bhatt2020} \\
ExpiScore & Empirical weights & No & Abs. & Yes$^*$ & \citep{fanale2025} \\
\midrule
\textbf{GRALIS} & \textbf{Riesz $L^2(\mathcal{Q})$} &
\textbf{$\approx^\dagger$} & \textbf{Yes} & \textbf{Yes} & \textbf{---} \\
\bottomrule
\end{tabular}
\smallskip\\
\footnotesize{$^*$without formal guarantees; Comm.\ = commensurability.
$^\dagger$Exact for the abstract canonical form
(Theorem~\ref{thm:completeness}) and, up to the explicit closed-form
correction of Corollary~\ref{cor:mobius-correction}, for the concrete
Shapley-weighted estimator under $\pi_x\equiv1$; only approximate for
Algorithm~1's practical estimator with a non-constant LIME kernel and
finite $m,k$ (Remark~\ref{rem:approx_completeness}), same standard as
Table~\ref{tab:axioms}'s Efficiency row.}
\end{table}

\section{The GRALIS Framework}
\label{sec:gralis}

\subsection{Definitions and Notation}
\label{sec:defs}

\begin{definition}[GRALIS explicit formula]
\label{def:explicit}
The explicit formula of GRALIS combines three components:
\begin{equation}
  \label{eq:gralis-explicit}
  \mathrm{GRALIS}_i(x) \;=\; \frac{1}{Z_x}
  \sum_{S \subseteq \mathcal{F}\setminus\{i\}}
  \underbrace{\frac{|S|!(n-|S|-1)!}{n!}}_{\text{Shapley weight}}
  \cdot\;
  \underbrace{\pi_x(x_S)}_{\text{LIME kernel}}
  \cdot\;
  \underbrace{(x_i - x'_i)\int_0^1 \nabla_i F(\tilde{x}^{(S,\alpha)})\,d\alpha}_{\text{IG conditioned on }S},
\end{equation}
where $Z_x = Z_{x,i} := \sum_{S \subseteq \mathcal{F}\setminus\{i\}} \frac{|S|!(n-|S|-1)!}{n!}\,\pi_x(x_S)$
is in general feature-dependent, as the summation range already excludes
$i$; we write $Z_x$ instead of $Z_{x,i}$ only for brevity
(by Lemma~\ref{lem:shapley-weight}, $Z_{x,i}=1$ whenever $\pi_x\equiv1$,
for every $i$, so the normalization is only active for a non-constant LIME
kernel), and the evaluation point is
$\tilde{x}^{(S,\alpha)}_j = x'_j + \alpha(x_j - x'_j)$ if $j \in S \cup \{i\}$,
$x'_j$ otherwise.
The IG path is not global, but \emph{conditioned on coalition $S$}:
features outside $S$ remain at the baseline during integration.
\end{definition}

\paragraph{Decomposition of the three components.}
\begin{itemize}
  \item \textbf{Shapley weight} $w(S) = |S|!(n-|S|-1)!/n!$:
        by itself, on the underlying game $\vG$, guarantees symmetry and
        the dummy axiom; averages over all feature insertion orders. Once
        combined with a non-constant locality kernel $\pi_x$ and
        feature-specific normalization $Z_{x,i}$, this guarantee does
        \emph{not} automatically transfer to the normalized estimator
        $\phi_i^{\GRALIS}=N_i/Z_{x,i}$ --- see
        Remark~\ref{rem:symmetry-break} below for a worked
        counterexample.
  \item \textbf{Locality kernel} $\pi_x(x_S) = \exp(-\norm{x_S - x'_S}^2/2\sigma^2)$:
        as written, this is a \emph{baseline-proximity} kernel, not a
        locality kernel around $x$: $S=\emptyset$ gives $x_S-x'_S=0$
        identically, so $\pi_x(\emptyset)=1$ is always the maximum weight,
        and $\pi_x(x_S)$ \emph{decreases} as $S$ grows (i.e.\ as the
        predecessor coalition moves away from the baseline $x'$ and toward
        $x$), rather than favoring coalitions whose masked input
        $\tilde x^{(S,\cdot)}$ is close to $x$ itself. We therefore call it
        a locality kernel in the sense of LIME-style exponential weighting
        by construction, not in the sense of being centered on $x$; see
        Remark~\ref{rem:kernel-limits} below for why the naive
        $\sigma\to\infty$/$\sigma\to0$ correspondences to KernelSHAP/LIME
        do not hold and have been removed.
  \item \textbf{Conditioned IG}: captures the curvature of $F$ along the path
        specific to each coalition (Gap~2 filled). Each individual
        within-coalition path satisfies the Fundamental Theorem of
        Calculus exactly (Gap~3 filled at the level of a single
        coalition-conditioned IG term); this does \emph{not} by itself
        imply that the aggregated, Shapley-and-kernel-weighted GRALIS
        estimator is complete --- that stronger property is analyzed
        separately in Corollary~\ref{cor:mobius-correction} and
        Proposition~\ref{prop:incompatibility}, and holds only under
        additional conditions (multilinear/affine $F$, uniform kernel).
\end{itemize}

\begin{remark}[Why the $\sigma\to\infty$/$\sigma\to0$ correspondences do not hold]
\label{rem:kernel-limits}
The naive claim that $\sigma\to\infty$ recovers KernelSHAP and
$\sigma\to0$ recovers LIME does not hold as stated.
As $\sigma\to\infty$, $\pi_x\to1$ uniformly, and $Z_x\to1$
(Lemma~\ref{lem:shapley-weight}), so
$\phi_i\to\sum_S w_{\mathrm{Sh}}(|S|)(x_i-x'_i)\int_0^1\nabla_iF(\tilde
x^{(S,\alpha)})\,d\alpha$: this is the \emph{uniform-kernel GRALIS}
estimator, whose joint conditioned-path construction is generally
\emph{not} equal to KernelSHAP's marginal-difference game
$f(x_{S\cup\{i\}})-f(x_S)$ --- the paper's own toy example already exhibits
this gap numerically ($\sum_i\phi_i^{\GRALIS}=8.5$ under
$\pi_x\equiv1$ versus the efficiency value $F(x)-F(x')=11$ that KernelSHAP
would reproduce exactly, Appendix~\ref{app:toy-verification}). As
$\sigma\to0$, for generic $x\ne x'$ the baseline-proximity kernel sends
$\pi_x(x_S)\to0$ for every $S\ne\emptyset$ relative to $\pi_x(\emptyset)=1$,
so after normalization the estimator collapses to the single-feature,
$S=\emptyset$ conditioned path,
$\phi_i\to(x_i-x'_i)\int_0^1\nabla_iF(x'+\alpha(x_i-x'_i)e_i)\,d\alpha$ ---
an ordinary one-feature Integrated-Gradients term from the baseline, not
LIME's locally weighted surrogate regression (which additionally requires
fitting a linear model with a complexity penalty $\Omega(g)$, as discussed
for the LIME special case in the proof of Theorem~\ref{thm:canonical}).
Both correspondences are therefore removed; $\sigma$ should be read purely
as an interpolation parameter between two GRALIS-internal regimes (uniform
kernel vs.\ baseline-collapsed kernel), not as recovering KernelSHAP or LIME
themselves.
\end{remark}

\begin{remark}[The non-constant kernel can break symmetry: a worked
  counterexample]
\label{rem:symmetry-break}
Shapley's symmetry axiom requires $\phi_i=\phi_j$ whenever
$v(S\cup\{i\})=v(S\cup\{j\})$ for every $S\subseteq\mathcal
F\setminus\{i,j\}$ (features $i,j$ are interchangeable in the game). The
Shapley weight alone guarantees this for the underlying game $\vG$; it
does \emph{not} guarantee it for $\phi^{\GRALIS}=N/Z_{x}$ once a
non-constant $\pi_x$ enters both $N$ and $Z_x$, because $\pi_x$ can treat
otherwise-symmetric features asymmetrically. Take $n=2$,
$F(x_1,x_2)=x_1x_2$, $x'=(0,0)$, $x=(a,b)$ with $a\ne0\ne b$. The induced
game $\vG(S)=f(x_S)-f(x')$ satisfies $\vG(\{1\})=\vG(\{2\})=0$ and
$\vG(\{1,2\})=ab$: features $1$ and $2$ are interchangeable in $\vG$'s own
sense (the standard Shapley value on $\vG$ gives $\phi_1^{\mathrm{Sh}}=
\phi_2^{\mathrm{Sh}}=ab/2$ by symmetry). For GRALIS's joint
conditioned-path attribution, the empty-predecessor term vanishes
($\nabla_1F$ evaluated with $x_2$ held at baseline $0$ is $0$ along that
path) and the full-predecessor term gives, for feature $1$, $S=\{2\}$:
$\mathrm{IG}_1(\{2\})=a\int_0^1\alpha b\,d\alpha=ab/2$ (both coordinates
move together along the joint path since $S\cup\{1\}=\{1,2\}$), and
symmetrically $\mathrm{IG}_2(\{1\})=ab/2$. With
$p_1:=\pi_x(x_{\{1\}})=\exp(-a^2/2\sigma^2)$,
$p_2:=\pi_x(x_{\{2\}})=\exp(-b^2/2\sigma^2)$ (and $\pi_x(x_\emptyset)=1$),
Definition~\ref{def:explicit}'s formula gives
\[
  \phi_1^{\GRALIS}(x) = \frac{ab}{2}\cdot\frac{p_2}{1+p_2},
  \qquad
  \phi_2^{\GRALIS}(x) = \frac{ab}{2}\cdot\frac{p_1}{1+p_1}.
\]
Whenever $a^2\ne b^2$, $p_1\ne p_2$ and hence $\phi_1^{\GRALIS}\ne
\phi_2^{\GRALIS}$, despite $\vG$'s own exact symmetry between the two
features. The mechanism is transparent: $\pi_x$ weights feature $1$'s
attribution by its proximity to the \emph{other} feature's coordinate
($p_2$, since reaching feature $1$'s full-predecessor term requires
feature $2$ to already be unmasked) and vice versa, so an asymmetric
baseline-proximity pattern between the two features' coordinates breaks
the symmetry that the bare Shapley weight would otherwise deliver. This
is a genuine property of practical GRALIS with $\sigma<\infty$, not a
finite-sample artifact of Algorithm~\ref{alg:mc}: it already holds at the
exact-enumeration level (Table~\ref{tab:complexity}'s ``GRALIS exact''
row). Symmetry is recovered exactly in the $\sigma\to\infty$ limit
(uniform kernel, Remark~\ref{rem:kernel-limits}), where $p_1=p_2=1$.
Table~\ref{tab:axioms} accordingly scores GRALIS's Symmetry as $\approx$,
not $\checkmark$, for the practical non-constant-kernel estimator.
\end{remark}

The canonical triple of GRALIS is
$\mathcal{Q} = 2^{\mathcal{F}\setminus\{i\}} \times [0,1]$,
$w_{S,\alpha} = \tilde{w}(S)\,d\alpha$,
$\Delta^{(i)}_{S,\alpha} = (x_i - x'_i)\,\nabla_i F(\tilde{x}^{(S,\alpha)})$,
strictly richer than the triples of GradCAM, SHAP, LIME and IG individually.

\begin{figure}[ht]
\centering
\begin{tikzpicture}[>=Stealth, scale=1.35, font=\small]
  \draw[->] (-0.2,0) -- (4.5,0) node[right] {$x_1$};
  \draw[->] (0,-0.2) -- (0,3.6) node[above] {$x_2$};
  \node[below] at (0.5,0) {\footnotesize $x'_1$};
  \node[left]  at (0,0.5) {\footnotesize $x'_2$};
  \node[below] at (3.8,0) {\footnotesize $x_1$};
  \node[left]  at (0,3.0) {\footnotesize $x_2$};
  \draw[gray!50, dashed, thin] (3.8,0) -- (3.8,3.0);
  \draw[gray!50, dashed, thin] (0,3.0) -- (3.8,3.0);
  \draw[gray!50, dashed, thin] (3.8,0) -- (3.8,0.5);
  \draw[gray!50, dashed, thin] (0,0.5) -- (3.8,0.5);
  \coordinate (xp) at (0.5,0.5);
  \coordinate (x)  at (3.8,3.0);
  \coordinate (v1) at (3.8,0.5);
  \fill (xp) circle (2.5pt) node[below left] {$x'$};
  \fill (x)  circle (2.5pt) node[above right] {$x$};
  \fill[gray!70] (v1) circle (2pt);
  \node[right=3pt] at (v1) {\footnotesize$(x_1,\,x'_2)$};
  \draw[blue!75!black, very thick, dashed, ->] (xp) -- (v1);
  \node[blue!75!black, below, align=center] at (2.15,0.25)
    {\footnotesize $S\!=\!\emptyset$: $\tilde{x}^{(\emptyset,\alpha)}\!=\!(x'_1\!+\!\alpha d_1,\;x'_2)$};
  \draw[red!70!black, very thick, ->] (xp) -- (x);
  \node[red!70!black, above left, align=center] at (2.15,1.75)
    {\footnotesize $S\!=\!\{2\}$: $\tilde{x}^{(\{2\},\alpha)}\!=\!(x'_1\!+\!\alpha d_1,\;x'_2\!+\!\alpha d_2)$};
  \draw[gray!50, dotted] (v1) -- (x);
  \node[draw=gray!60, rounded corners, fill=gray!8, align=left,
        font=\footnotesize, inner sep=5pt] at (1.0, 3.1)
    {GRALIS ($i\!=\!1$, $n\!=\!2$):\\
     average over $2^{n-1}\!=\!2$ paths\\
     with weights $w(S)\cdot\pi_x(x_S)$.};
\end{tikzpicture}
\caption{Conditioned integration paths in GRALIS ($n=2$, target feature $i=1$,
  $d_j = x_j - x'_j$).
  The blue path ($S=\emptyset$) keeps $x_2$ at the baseline during integration:
  captures the \emph{pure} contribution of $x_1$, ignoring the interaction with $x_2$.
  The red path ($S=\{2\}$) moves both features simultaneously:
  captures the mixed curvature $\partial^2 F/\partial x_1 \partial x_2$ (Gap~2).
  The standard IG gradient corresponds exclusively to the path $S=\mathcal{F}\setminus\{i\}$.
  GRALIS averages over all $2^{n-1}$ paths with Shapley\,$\times$\,LIME-kernel weights,
  producing an attribution that balances completeness and local curvature.}
\label{fig:conditioned-path}
\end{figure}

\begin{definition}[GRALIS canonical structure]
\label{def:canonical}
Let $f : \R^n \to \R$ be a model and let $(\mathcal{Q}, \Sigma, \mu)$ be a
$\sigma$-finite measure space fixed a priori.
The GRALIS \emph{canonical structure} is a triple $(\mathcal{Q}, w, \Delta)$ where:
\begin{itemize}
  \item $\mathcal{Q}$ is the integration index space (continuous or discrete);
  \item $w : \mathcal{Q} \to \R^+$ is the weight function with $\int_{\mathcal{Q}} w(q)\,d\mu(q) = 1$;
  \item $\Delta : \mathcal{Q} \to \R^n$ is the marginal contribution function subject to the
        \emph{constitutive constraint}:
        $\sum_i \Delta_i(f, x, x', q) = f(x) - f(x')$ for $\mu$-a.e.\ $q \in \mathcal{Q}$.
\end{itemize}
The GRALIS attribution for feature~$i$ is:
\[
  \phi_i^{\GRALIS}(f, x, x') \;=\; \int_{\mathcal{Q}} w(q)\cdot\Delta_i(f, x, x', q)\,d\mu(q).
\]
\end{definition}

\begin{remark}[What ``componentwise'' means, and why the common-$\mathcal Q$ packaging above is a genuine additional step]
\label{rem:componentwise}
Roadmap: Definition~\ref{def:explicit} above gives the computational
formula (how to compute GRALIS); Theorem~\ref{thm:canonical} below
proves, in \emph{componentwise} form, that this formula's structure is
forced by linearity and continuity, one feature at a time; the
Definition here packages the resulting per-feature triples
$\{(Q_i,w_i,\Delta_i)\}_{i=1}^n$ into one abstract object
$(\mathcal Q,w,\Delta)$ for expositional economy. What follows makes
precise in what sense that packaging is, and is not, automatic ---
this is the single most load-bearing scope distinction in the paper,
referenced repeatedly below. Theorem~\ref{thm:canonical} is proved, and used in every special case
discussed in its own proof (GradCAM-lin, SHAP, LIME, IG), \emph{one feature
$i$ at a time}: it fixes $i$, fixes a scalar functional $E_i$, and applies Riesz on
$\mathcal D(Q_i)$ to obtain a representer $w_i$. Nothing in that argument
forces $Q_i=Q_j$ for $i\ne j$, nor does it force the representers $w_i$ to
agree, nor does it produce, by itself, a single point $q\in\mathcal Q$ at
which all $n$ components $\Delta_1(q),\ldots,\Delta_n(q)$ are
\emph{simultaneously} defined and satisfy a pointwise identity --- which is
exactly what Definition~\ref{def:canonical}'s constitutive constraint
requires. For GRALIS's own $Q_i=2^{\mathcal F\setminus\{i\}}\times[0,1]$,
this is not a bookkeeping nuance: $Q_i$ and $Q_j$ are literally different
sets for $i\ne j$ (the point $(S,\alpha)\in Q_i$ has $S\subseteq\mathcal
F\setminus\{i\}$, so $S$ may contain $j$; the same $S$ is not an element of
$Q_j$, whose predecessor sets must exclude $j$), so the expression
$\sum_i\Delta_i(q)$ for a single $q\in\mathcal Q$ is not even intrinsically
defined without an explicit common superspace and an alignment map between
the $Q_i$'s. A na\"ive repair, $\mathcal Q:=\bigsqcup_i(\{i\}\times Q_i)$,
does not fix this either: on that disjoint union a pointwise sum over all
$n$ feature components \emph{at the same $q$} is no longer natural, since
each $q$ only ever belongs to one $\{i\}\times Q_i$ block. We do not
attempt a general common-space construction here (one candidate --- a
single space of permutations $\times$ path parameter, from which every
$Q_i$ is recovered by restriction --- is plausible but is not developed or
proved in this paper). Consequently: \textbf{Definition~\ref{def:canonical}
should be read as an abstract superclass, stated for expositional economy
in vector form, and not as a structure literally instantiated component-by-
component by Definition~\ref{def:explicit}.} The only place this
abstract, common-$\mathcal Q$ vector structure is invoked in a numbered
theorem is Theorem~\ref{thm:completeness} (abstract completeness); the gap
between that abstract statement and Algorithm~\ref{alg:mc}'s actual,
per-feature practical estimator is already tracked quantitatively by
Remark~\ref{rem:approx_completeness}, and no other numbered theorem in this
paper (T1, T3--T7) depends on the common-$\mathcal Q$ packaging holding
literally. The defensible content of Theorem~\ref{thm:canonical} is
therefore: \emph{componentwise Riesz representability, conditional on a
chosen attribution mechanism and per-feature Hilbert space $Q_i$} --- not
``one canonical triple with a shared $\mathcal Q$ and $w$ is necessary for
every additive attribution method,'' which is a stronger claim this paper
does not establish.
\end{remark}

\begin{remark}[What Riesz proves vs.\ what the canonical structure additionally imposes]
\label{rem:weak-strong}
It is worth separating two distinct claims that are easy to conflate.
Theorem~\ref{thm:canonical} (Riesz) proves only a \emph{weak} form:
under conditions (a)--(b), a continuous linear attribution functional
admits a unique linear representer $w$ --- nothing in the Riesz
Representation Theorem forces $w\ge0$ or $\int_{\mathcal Q}w\,d\mu=1$, and
nothing in it requires completeness (Definition~\ref{def:canonical}'s
constitutive constraint) either.
Positivity and normalization of $w$ are additional, \emph{non-Riesz}
requirements imposed by Definition~\ref{def:canonical}'s \emph{canonical
structure} as a design choice, so that $w$ can be read as a
probability-like weighting/sampling measure over $\mathcal{Q}$ (useful
for GRALIS's own Monte-Carlo estimator, Algorithm~\ref{alg:mc}, and for
the product-measure argument of Theorem~\ref{thm:sobol}).
The two claims should not be conflated: ``the canonical form
$(\mathcal Q,w,\Delta)$ is necessary for any linear additive attribution''
(the weak form, proved) is a different statement from ``$w$ must be
positive and normalized'' (true only for instances that adopt the
canonical structure of Definition~\ref{def:canonical}). Integrated
Gradients (Case~4 below) is an example of a method that instantiates
the weak form with a $w$ that changes sign whenever $x_i<x_i'$, and is
still correctly unified by Theorem~\ref{thm:canonical}.
\end{remark}

In the discrete case ($\mathcal{Q}$ finite or countable), the integral is
understood as a sum and $L^2(\mathcal{Q}) = \R^{|\mathcal{Q}|}$; the Riesz
Theorem applies in both cases.

\subsection{Theorem 1: Unified Canonical Form (Riesz)}
\label{sec:t1}

Completeness is not assumed in Theorem~\ref{thm:canonical} below: the
constitutive constraint of Definition~\ref{def:canonical} (completeness)
is not used anywhere in its proof, since the Riesz Representation Theorem
needs only a continuous linear functional on a Hilbert space, which
conditions (a)--(b) below already supply. Completeness enters
separately, as part of Definition~\ref{def:canonical}'s canonical
structure and as an explicit hypothesis of Theorem~\ref{thm:completeness},
which is the result that actually needs it. This makes
Theorem~\ref{thm:canonical}'s necessity claim strictly \textbf{stronger}:
it applies to any linear, continuous attribution functional, whether or
not it happens to satisfy completeness.

\begin{theorem}[Canonical Form --- Riesz, componentwise necessity]
\label{thm:canonical}
Let $\mathcal{D}(\mathcal{Q}) \subseteq L^2(\mathcal{Q})$ be the closed subspace of
admissible marginal functions.
For each feature $i$, an attribution method $\phi_i$ satisfying the conditions:
\begin{itemize}
  \item[(a)] $\mathcal{Q}$ is a \emph{measurable index space} $(\mathcal{Q},\mathcal{A},\mu)$
    fixed a priori, independently of the specific input $x$.
    $\mathcal{Q}$ \emph{is not arbitrary}: it is \emph{induced} by the attribution mechanism,
    and GRALIS operates conditionally on this choice.
    Concretely: for Integrated Gradients, $\mathcal{Q}=[0,1]$ (path)
    with Lebesgue measure; for SHAP, $\mathcal{Q}=2^N$ (coalitions)
    with counting measure; for GRALIS, $\mathcal{Q} =
    2^{\mathcal{F}\setminus\{i\}}\times[0,1]$ (predecessor coalition times
    path parameter), matching the explicit formula of
    Definition~\ref{def:explicit}. When the discrete features $\mathcal F$
    are themselves obtained from a continuous raw-index space (e.g.\ pixels
    grouped into superpixels for images), that grouping is a separate, prior
    construction --- formalized as the projection $\rho$ of
    Appendix~\ref{app:projection} --- and is not itself the index space
    $\mathcal Q$ of this theorem: $\mathcal Q$ always ranges over predecessor
    coalitions of the already-discrete $\mathcal F$, never over raw pixels.
    When $\mathcal F$ itself is obtained by an input-dependent
    segmentation such as SLIC, ``fixed a priori'' must be read as fixed
    \emph{for the input $x$ under analysis, before $\mathcal Q$'s own
    sampling begins}, not as identical across different $x$ ---
    see Remark~\ref{rem:canonical-imaging} for the precise statement and
    why this does not invalidate the theorem's use in this paper.
    The framework only requires that $\mu$ be $\sigma$-finite;
  \item[(b)] the evaluation map $E_i : \mathcal{D}(\mathcal{Q}) \to \R$,
             $E_i(\Delta_i) := \phi_i(f, x, x'; \Delta_i)$,
             is linear and continuous with respect to the $L^2(\mathcal{Q})$ norm;
\end{itemize}
admits a unique representation:
\[
  \phi_i(f, x, x') \;=\; \int_{\mathcal{Q}} w(q)\cdot\Delta_i(q)\,d\mu(q)
\]
for a unique $w \in \mathcal{D}(\mathcal{Q})$
(unique in $L^2(\mathcal{Q})$ whenever $\mathcal{D}(\mathcal{Q})=L^2(\mathcal{Q})$;
see the proof for the general case $\mathcal{D}(\mathcal{Q})\subsetneq L^2(\mathcal{Q})$).
This representation is \emph{necessary}: \emph{any} feature-$i$ attribution
functional satisfying only (a)-(b) --- linearity and continuity, with no
completeness requirement --- must take exactly this form. It is a
\emph{weak} form in the sense that it asserts only linearity, continuity
and existence/uniqueness of $w$; it does \emph{not} by itself imply
$w\ge0$ or $\int_{\mathcal Q}w\,d\mu=1$ --- those are the additional,
non-Riesz requirements of the \emph{canonical structure} of
Definition~\ref{def:canonical} (see Remark~\ref{rem:weak-strong} below),
which also adds the constitutive (completeness) constraint as a further,
separate requirement not needed for necessity itself.
\end{theorem}

\begin{proof}
By condition~(b), $E_i$ is, by hypothesis, already a continuous linear
functional on the Hilbert space $\mathcal{D}(\mathcal{Q})$ (a closed
subspace of $L^2(\mathcal{Q})$, hence itself a Hilbert space): nothing
further needs to be established before invoking Riesz. By the Riesz
Representation Theorem~\citep{riesz1909,rudin1991}, $E_i$ therefore admits
a unique representer $w\in\mathcal{D}(\mathcal{Q})$ with
$E_i[g] = \langle g, w\rangle_{L^2} = \int_{\mathcal{Q}} g(q)\cdot
w(q)\,d\mu(q)$ for every $g\in\mathcal D(\mathcal Q)$; applying this to
$g=\Delta_i$ gives the displayed formula
$\phi_i(f,x,x')=E_i(\Delta_i)=\int_{\mathcal Q}w(q)\Delta_i(q)\,d\mu(q)$.
(For completeness: the functional $g\mapsto\int_{\mathcal
Q}w(q)g(q)\,d\mu(q)$ so obtained is itself automatically linear and
continuous, by linearity of the integral and Cauchy--Schwarz
$|\langle g,w\rangle_{L^2}|\le\|w\|_{L^2}\|g\|_{L^2}$; this is a
consequence of Riesz's construction, not an additional hypothesis to be
verified separately.)
If $\mathcal{D}(\mathcal{Q})$ is a \emph{proper} closed subspace of
$L^2(\mathcal{Q})$, any extension of $w$ to all of $L^2(\mathcal{Q})$ is
unique only modulo $\mathcal{D}(\mathcal{Q})^\perp$: adding any
$h\in\mathcal{D}(\mathcal{Q})^\perp$ to $w$ leaves $E_i$ unchanged,
since $\Delta_i\in\mathcal{D}(\mathcal{Q})$ and $\langle
h,\Delta_i\rangle_{L^2}=0$. Uniqueness \emph{within}
$\mathcal{D}(\mathcal{Q})$ --- the space on which $\phi_i$ actually acts
--- is unaffected by this and is all that is needed for the attribution
formula above.
The representation is unique \emph{conditional on} the selected attribution
mechanism, index space $\mathcal{Q}$, measure $\mu$, and admissible marginal
subspace $\mathcal{D}(\mathcal{Q})$; different choices of $(\mathcal{Q},\mu)$
correspond to different attribution mechanisms, each with its own unique $w$.
This argument is carried out \emph{for a fixed feature $i$}; see
Remark~\ref{rem:componentwise} for exactly what does, and does not,
follow when the argument is repeated across all $n$ features.

\emph{Special cases.}
\begin{itemize}
  \item \textbf{GradCAM-lin} (Case 1):
        $s_{pq} = \sum_k \alpha^c_k A^k_{pq}
        = \sum_{k,i,j} \tfrac{1}{Z} \cdot
        \underbrace{\tfrac{\partial y^c}{\partial A^k_{ij}} \cdot A^k_{pq}}_{\Delta^{(pq)}(f)}$.
        Triple: $\mathcal{Q} = \{(k,i,j)\}$, $w_{k,i,j} = 1/Z$,
        $\Delta_{k,i,j} = \tfrac{\partial y^c}{\partial A^k_{ij}} \cdot A^k_{pq}$.

  \item \textbf{SHAP} (Case 2):
        $\phi_i = \sum_{S \subseteq \mathcal{F}\setminus\{i\}}
        \underbrace{\tfrac{|S|!(n-|S|-1)!}{n!}}_{w(S)} \cdot
        \underbrace{[f(S \cup \{i\}) - f(S)]}_{\Delta^{(i)}_S(f)}$.
        Triple: $\mathcal{Q} = 2^{\mathcal{F}\setminus\{i\}}$,
        $w(S) = |S|!(n-|S|-1)!/n!$, $\Delta^{(i)}_S = f(S\cup\{i\}) - f(S)$.
        Weights sum to 1: $\sum_{S} w(S) = 1$ (Lemma~\ref{lem:shapley-weight}).

  \item \textbf{LIME} (Case 3):
        Standard LIME~\citep{ribeiro2016} solves
        $w^* = \arg\min_{g\in\mathcal G} L(f,g,\pi_x) + \Omega(g)$
        (\S\ref{sec:background}). Two regimes cover $\Omega$ as used in
        practice. \textbf{(3a) Ridge-regularized LIME (unconditional).}
        For $\Omega(g)=\lambda\|w\|_2^2$ at fixed $\lambda\ge0$ and a fixed
        design $X$ (sampled coalitions chosen independently of $f$), the
        penalized solution $w^*=(X^TWX+\lambda I)^{-1}X^TWf$ is a fixed
        linear map of $f$ for every $\lambda\ge0$ (including $\lambda=0$),
        which is the default regularizer of the reference \texttt{lime}
        package~\citep{ribeiro2016}. Triple: $\mathcal Q=\{z^{(1)},\ldots,z^{(T)}\}$,
        $w_t=[(X^TWX+\lambda I)^{-1}X^TW]_{i,t}$,
        $\Delta_t^{(i)}=f(x\odot z^{(t)}+x'\odot(1-z^{(t)}))$.
        \textbf{(3b) Feature-selection LIME (conditional family).} When
        $\Omega$ instead performs feature selection ($K$-feature or
        Lasso), the active support $\hat S$ depends on $f$, so no single
        fixed operator represents $f\mapsto w^*$ across all $f$; fixing
        $\hat S=S_0$, however, the restricted regression on $S_0$'s
        columns is again a fixed linear map, giving one genuine
        Theorem~\ref{thm:canonical} instance per realized support $S_0$
        --- a family of instances, not one operator valid across varying
        $\hat S$ (the same honest scope limit as the input-dependent SLIC
        case of Remark~\ref{rem:canonical-imaging}). Both regimes
        instantiate conditions (a)--(b) in full, and since completeness is
        not a hypothesis of Theorem~\ref{thm:canonical}, each instance is
        genuine; completeness itself (the constitutive constraint of
        Definition~\ref{def:canonical}) holds for LIME only in the
        degenerate zero-residual case, consistent with LIME's
        completeness being unsatisfied in Table~\ref{tab:axioms}. LIME is
        included to show it is unified at the level of $(\mathcal
        Q,w,\Delta)$, not to claim full canonical-structure membership.

  \item \textbf{Integrated Gradients} (Case 4):
        $\mathrm{IG}_i \approx \sum_{j=1}^k
        \underbrace{\tfrac{x_i - x'_i}{k}}_{w_j} \cdot
        \underbrace{\nabla_i F(x^{(j)})}_{\Delta^{(i)}_j(f)}$.
        Discrete triple: $\mathcal{Q} = \{1,\ldots,k\}$, $w_j = (x_i-x'_i)/k$,
        $\Delta^{(i)}_j = \nabla_i F(x^{(j)})$.
        In the continuous limit: $\mathcal{Q} = [0,1]$,
        $w(\alpha)\,d\alpha = (x_i - x'_i)\,d\alpha$,
        $\Delta^{(i)}(\alpha) = \nabla_i F(x' + \alpha(x-x'))$.
        \emph{On sign of $w$:} when $x_i<x_i'$, $w_j=(x_i-x_i')/k<0$; IG
        instantiates only the \emph{weak} form of
        Theorem~\ref{thm:canonical} (Remark~\ref{rem:weak-strong}), not the
        positive-normalized canonical structure of
        Definition~\ref{def:canonical}, and this does not affect the
        necessity argument below (linearity is what Riesz uses, not sign).
\end{itemize}

\emph{Uniqueness Lemma (Riesz):}
the marginal contribution $\Delta^{(i)}_q(f) = f(q^{+i}) - f(q^{-i})$ is linear in $f$,
where $q^{+i}$ denotes configuration $q$ with feature $i$ included
(e.g.\ coalition $S\cup\{i\}$) and $q^{-i}$ denotes it with feature $i$ excluded
(e.g.\ $S$). The composition $\Phi_i[\Delta^{(i)}(\,\cdot\,; f)] = \langle w, \Delta^{(i)}(f)\rangle$
therefore remains linear in $f$, and the componentwise Riesz representer
$w_i$ is necessary for feature $i$: it is not a design choice, but a
structural consequence of $E_i$ being linear and continuous (Riesz, 1909).
As emphasized in Remark~\ref{rem:componentwise}, this necessity is proved
one feature at a time; it does not by itself force a single shared
$(\mathcal Q,w)$ across features, which is the additional packaging of
Definition~\ref{def:canonical}.
\end{proof}

\begin{remark}[Positioning against Riesz-based and path-ensemble prior art]
\label{rem:lundstrom-pointer}
Theorem~\ref{thm:canonical} is not the first application of Riesz
representation to feature attribution, nor is it an independent
characterization of path attribution methods where its domain overlaps
theirs; Section~\ref{sec:lundstrom-position} positions it precisely
against~\citep{chau2021,taimeskhanov2025,lundstrom2022,lundstrom2025}.
Hypothesis~(b) above --- linearity and continuity of $E_i$ --- is assumed
directly, as the more standard and directly verifiable route to
Theorem~\ref{thm:canonical}.
\end{remark}

\begin{remark}[Why standard GradCAM falls outside Theorem~\ref{thm:canonical}]
\label{rem:relu_nonlinear}
Let $L(f) = \sum_k \alpha_k^c(f)\,A^k_{pq}$ denote the linear aggregation
computed by GradCAM-lin, where $\alpha_k^c(f)$ is linear in $f$ by
linearity of differentiation.
Standard GradCAM applies $\mathrm{ReLU}$ pointwise \emph{after} this
aggregation: $\Phi_{\mathrm{GC}}(f) = \mathrm{ReLU}(L(f))$.
This composite map is \emph{not} linear in $f$: for any $\lambda < 0$,
\[
  \Phi_{\mathrm{GC}}(\lambda f)
  = \mathrm{ReLU}(\lambda\,L(f))
  \;\ne\;
  \lambda\,\mathrm{ReLU}(L(f))
  = \lambda\,\Phi_{\mathrm{GC}}(f)
\]
whenever $L(f) > 0$.
Hence condition~(b) of Theorem~\ref{thm:canonical} (linearity in $f$)
fails for standard GradCAM, and the Riesz representation cannot be applied.

GradCAM-lin, which sets $\Phi_{\mathrm{lin}}(f)=L(f)$ directly, satisfies
conditions~(a)--(b) and is a natural, minimal \emph{linear representative}
of the GradCAM family admissible within GRALIS --- ``minimal'' in that it
removes exactly the one nonlinear operation (the post-aggregation
$\mathrm{ReLU}$) violating condition~(b), not in the sense of a proven
uniqueness result over all possible GradCAM linearizations. The
$\mathrm{ReLU}$ of~\citep{selvaraju2017} is a deliberate, effective
class-discriminative design choice~\citep{hooker2019}, valid in the
visualization context, but it is not an axiomatic requirement and
prevents formal comparison with SHAP/IG/LIME, which preserve the full
signed signal; GradCAM-lin retains it, consistent with the
gradient$\times$input view of Ancona et al.~\citep{ancona2018} and deep
Taylor decomposition~\citep{montavon2017}. \GRALIS{} does not depend
computationally on GradCAM (Algorithm~1 uses Monte Carlo integration over
Shapley coalitions); the GradCAM-lin case only shows the linearized
family fits the canonical template. Two scope caveats, kept brief: (i)
$L(f)=\sum_k\alpha_k^c(f)A^k_{pq}$ depends on a fixed network realization
and internal layer, not on $f$ alone --- condition~(b) is verified with
that pair fixed, exactly as $\mathcal Q$ is fixed by the mechanism in
condition~(a); (ii) on pre-activation feature maps the aggregation is
unconditionally linear in $Z$, though the class score's Taylor expansion
is exact only if the downstream map to the score is itself affine on the
relevant region.
\end{remark}

\begin{lemma}[Combinatorial identity --- Shapley weights]
\label{lem:shapley-weight}
For every fixed feature $i$, the Shapley weights form a probability
distribution over the $2^{n-1}$ coalitions of $\mathcal{F}\setminus\{i\}$:
\[
  \sum_{S \subseteq \mathcal{F}\setminus\{i\}} \frac{|S|!(n-|S|-1)!}{n!} \;=\; 1.
\]
\end{lemma}
\begin{proof}
Group by cardinality $k = |S| \in \{0,\ldots,n-1\}$, noting that there are
$\binom{n-1}{k}$ subsets of $\mathcal{F}\setminus\{i\}$ of cardinality $k$:
\begin{align*}
\sum_{S \subseteq \mathcal{F}\setminus\{i\}} w(S)
  &= \sum_{k=0}^{n-1} \binom{n-1}{k} \cdot \frac{k!(n-k-1)!}{n!} \\
  &= \sum_{k=0}^{n-1} \frac{(n-1)!}{k!(n-1-k)!} \cdot \frac{k!(n-k-1)!}{n!}.
\end{align*}
Since $n-k-1 = n-1-k$, we have $(n-k-1)! = (n-1-k)!$, so each term equals:
\[
  \frac{(n-1)!}{n!} = \frac{1}{n}.
\]
The sum over $k=0,\ldots,n-1$ contains $n$ terms, each equal to $1/n$:
\[
  \sum_{S} w(S) = n \cdot \frac{1}{n} = 1. \qquad \square
\]
\end{proof}
\begin{remark}
The fact that $\sum_S w(S) = 1$ --- and not $1/n$ --- is essential: it guarantees that
the GRALIS formula is a convex weighted average of marginal contributions,
without the need for external normalization for each individual feature $i$.
\end{remark}

\begin{table}[ht]
\centering
\caption{Mapping of XAI methods onto the canonical triple $(\mathcal{Q}, w, \Delta)$.}
\label{tab:mapping}
\small
\begin{tabular}{llll}
\toprule
\textbf{Method} & $\mathcal{Q}$ & $w(q)$ & $\Delta_i(q)$ \\
\midrule
GradCAM-lin & $\{(k,i,j)\}$: channels$\times$positions & $1/Z$ &
  $\frac{\partial y^c}{\partial A^k_{ij}} \cdot A^k_{pq}$ \\[4pt]
SHAP & $2^{\mathcal{F}\setminus\{i\}}$: coalitions & $\frac{|S|!(n-|S|-1)!}{n!}$ &
  $f(S\cup\{i\}) - f(S)$ \\[4pt]
LIME & $\{z^{(1)},\ldots,z^{(T)}\}$: perturbations & $\pi_x(z^{(t)}) h_i(z^{(t)})$ &
  $f(z^{(t)})$ \\[4pt]
IG & $[0,1]$: path $x' \to x$ & $(x_i - x'_i)\,d\alpha$ &
  $\nabla_i F(x' + \alpha(x-x'))$ \\
\bottomrule
\end{tabular}
\end{table}

\subsection{Theorem 2: Completeness}
\label{sec:t2}

\begin{theorem}[Completeness]
\label{thm:completeness}
\textbf{Hypothesis (in addition to Theorem~\ref{thm:canonical}'s (a)--(b)):}
assume $(\mathcal Q,w,\Delta)$ is a full canonical structure in the sense
of Definition~\ref{def:canonical}, i.e.\ $\Delta$ additionally satisfies
the \emph{constitutive constraint} $\sum_i\Delta_i(f,x,x',q)=f(x)-f(x')$
for $\mu$-a.e.\ $q\in\mathcal Q$ (this is exactly the condition removed
from Theorem~\ref{thm:canonical}'s hypotheses; it is restated here,
explicitly, as the one place in this paper where it is actually used).
Then GRALIS attributions satisfy exact completeness:
\[
  \sum_{i=1}^n \phi_i^{\GRALIS}(f, x, x') \;=\; f(x) - f(x').
\]
\end{theorem}
\begin{proof}
By linearity of Fubini and the constitutive constraint on $\Delta$:
\[
  \sum_i \phi_i = \int_{\mathcal{Q}} w(q) \Bigl[\sum_i \Delta_i(q)\Bigr]\,d\mu(q)
  = \int_{\mathcal{Q}} w(q)\,[f(x)-f(x')]\,d\mu(q) = f(x)-f(x').
\]

\emph{Lemma A (completeness of conditioned IG on $S$).}
For every $S \subseteq \mathcal{F}$ and fixed $x'$:
$\sum_{j \in S} \mathrm{IG}^S_j = F(x_S) - F(x')$,
where $x_S$ denotes the vector with $x_j$ for $j \in S$ and $x'_j$ otherwise.
Proof: by construction $\tilde{x}^{(S,0)} = x'$ and $\tilde{x}^{(S,1)} = x_S$.
Applying the fundamental theorem of calculus to the path $\gamma(\alpha) = \tilde{x}^{(S,\alpha)}$:
$F(x_S) - F(x') = \int_0^1 \frac{d}{d\alpha}F(\tilde{x}^{(S,\alpha)})\,d\alpha
= \sum_{j \in S}(x_j - x'_j)\int_0^1 \nabla_j F(\tilde{x}^{(S,\alpha)})\,d\alpha
= \sum_{j \in S} \mathrm{IG}^S_j$.
\end{proof}

\begin{remark}
Completeness is not a non-trivial result: it is the direct application of
the constitutive constraint on $\Delta$. Its utility is to show that every
method conforming to the structure $(\mathcal{Q}, w, \Delta)$ automatically inherits
this property.
\end{remark}

\begin{remark}[Scope: abstract form vs.\ practical estimator]
\label{rem:t2-scope}
Theorem~\ref{thm:completeness}'s hypothesis (a single shared $\mathcal Q$
with the constitutive constraint holding jointly across features) is
\emph{not} instantiated by Algorithm~\ref{alg:mc} (Remark~\ref{rem:componentwise}:
each feature has its own $Q_i$, with no common superspace). So the
theorem is true for the abstract canonical form, but is never invoked to
claim exact completeness for practical GRALIS; that is governed instead
by Corollary~\ref{cor:mobius-correction}, Proposition~\ref{prop:incompatibility}/
Remark~\ref{rem:normalized-incompatibility}, and
Remark~\ref{rem:approx_completeness} below.
\end{remark}

\begin{corollary}[Exact completeness correction under multilinearity]
\label{cor:mobius-correction}
Suppose $F$ is genuinely multilinear on the region the proof below
actually evaluates it on --- \emph{not} merely at the single point $x$
under attribution, but as an identity of functions holding for every
$z=\tilde x^{(S,\alpha)}$ on the coalition-conditioned path domain
$\{\tilde x^{(S,\alpha)} : S\subseteq\mathcal F,\ \alpha\in[0,1]\}$ of
Definition~\ref{def:explicit}:
\[
  F(z) \;=\; F(x') + \sum_{\emptyset\ne T\subseteq\mathcal{F}} m(T)
  \prod_{j\in T}(z_j-x'_j), \qquad \text{for every } z \text{ on the path domain above,}
\]
where $m(T)$ are precisely the Möbius coefficients of the discrete game
$v(S):=F(x_S)-F(x')$ used in Theorem~\ref{thm:siv} and
Appendix~\ref{app:mobius} (this holds, e.g., whenever $F$ has degree
$\le 1$ in each individual coordinate on that region, so that cross
terms are allowed but no coordinate appears squared or higher ---
globally, or merely on the box spanned by $x'$ and $x$, is sufficient;
a formula that only reproduces $F(x)$ correctly at the single
evaluation point $x$, without controlling $F$ elsewhere on the path
domain, does \emph{not} suffice, since the proof differentiates $F_T$
at every $\tilde x^{(S,\alpha)}$, not only at $\alpha=1,S=\mathcal F$).
Then, for Shapley weights and $\pi_x \equiv 1$, the attributions of
Definition~\ref{def:explicit} satisfy
\begin{equation}
\label{eq:mobius-correction}
  \sum_{i=1}^n \phi_i^{\GRALIS}(x) \;=\; F(x)-F(x')
  \;-\; \sum_{\substack{T\subseteq\mathcal{F}\\|T|\ge 2}} m(T)
  \Bigl(1-\tfrac{1}{|T|}\Bigr) \prod_{j\in T}(x_j-x'_j).
\end{equation}
In particular the correction term vanishes identically whenever $F$ is
affine ($m(T)=0$ for all $|T|\ge2$), recovering exact completeness as in
Theorem~\ref{thm:completeness}.
\end{corollary}
\begin{proof}
By linearity of $\phi_i^{\GRALIS}$ in $F$ (the gradient in
eq.~\eqref{eq:gralis-explicit} is linear in $F$), it suffices to prove the
identity term-by-term for a single pure-interaction monomial
$F_T(x) = m(T)\prod_{j\in T}(x_j-x'_j)$ of order $d:=|T|$.
Since $\partial F_T/\partial x_i \equiv 0$ for $i\notin T$, only features
$i\in T$ receive a nonzero contribution from $F_T$.
For $i\in T$ and $S\subseteq\mathcal{F}\setminus\{i\}$, evaluating
$\partial F_T/\partial x_i$ along the coalition-conditioned path
$\tilde{x}^{(S,\alpha)}$ (eq.~\eqref{eq:gralis-explicit}) gives $0$ unless
every other member of $T$ is also active, i.e.\ $T\setminus\{i\}\subseteq
S\cup\{i\}$, in which case
$\partial F_T/\partial x_i(\tilde{x}^{(S,\alpha)}) = m(T)\,\alpha^{d-1}
\prod_{j\in T\setminus\{i\}}(x_j-x'_j)$.
Integrating over $\alpha\in[0,1]$ contributes a factor $1/d$, so
\[
  \Delta_S^{(i)}\big|_{F_T} =
  \begin{cases}
    \dfrac{m(T)}{d}\displaystyle\prod_{j\in T}(x_j-x'_j) & T\setminus\{i\}\subseteq S,\\[4pt]
    0 & \text{otherwise.}
  \end{cases}
\]
By the classical Shapley value of the unanimity game $u_T$ (which assigns
$1/|T|$ to each $i\in T$), $\sum_{S\supseteq T\setminus\{i\}} w(S) = 1/d$
for every $i \in T$. Hence
$\phi_i^{\GRALIS}\big|_{F_T} = \tfrac{1}{d}\cdot\tfrac{m(T)}{d}
\prod_{j\in T}(x_j-x'_j) = \tfrac{m(T)}{d^2}\prod_{j\in T}(x_j-x'_j)$,
and summing over the $d$ members of $T$ gives
$\sum_{i\in T}\phi_i^{\GRALIS}\big|_{F_T} = \tfrac{m(T)}{d}
\prod_{j\in T}(x_j-x'_j)$.
Summing over all $T\subseteq\mathcal{F}$ and comparing with
$F(x)-F(x')=\sum_T m(T)\prod_{j\in T}(x_j-x'_j)$ yields
\eqref{eq:mobius-correction}.
\end{proof}

\begin{remark}[Scope and empirical verification]
\label{rem:mobius-scope}
Corollary~\ref{cor:mobius-correction} isolates \emph{exactly} the deficit
caused by the joint coalition-conditioned path of
Figure~\ref{fig:conditioned-path} (which moves $S$ and $i$ together to
capture mixed curvature, Gap~2): linear terms are attributed in full,
while a pure interaction of order $d$ is attributed at a factor $1/d$ of
its true value --- a closed-form counterpart to the $O(1/d)$ dilution
familiar from Aumann--Shapley diagonal pricing of homogeneous games. This
attenuation is due entirely to the continuous, coalition-conditioned path
integration of Definition~\ref{def:explicit}, not to the discrete Shapley
coalition weights on their own: for any semivalue in the classical sense
of Dubey, Neyman and Weber~\citep{dubey1981} applied directly to the
discrete game $v(S)=F(x_S)-F(x')$ without such a path, the analogous
deficit is already fully characterized --- exactly zero for the Shapley
value by efficiency, and given in closed form for the Banzhaf value on
unanimity games by Owen~\citep{owen1975}. Corollary
\ref{cor:mobius-correction}'s content is the additional attenuation
specific to conditioning that discrete weighting on a continuous
IG-style path, which lies outside this classical discrete-game theory.
The correction uses the same Möbius coefficients $m(T)$ already computed for
the Shapley Interaction Values of Theorem~\ref{thm:siv}; Corollary
\ref{cor:mobius-correction} and Theorem~\ref{thm:siv} are independent
applications of that shared decomposition, not one derived from the
other. On the toy multilinear example verified numerically in
Appendix~\ref{app:toy-verification}
($F=2x_1+3x_2+x_3+4x_1x_2-x_1x_3+2x_2x_3$, $x'=0$, $x=\mathbf{1}$), the
formula gives $\sum_i\phi_i^{\GRALIS} = 11 - (4-1+2)\cdot\tfrac12 = 8.5$
exactly, matching exact symbolic integration to machine precision.
For general non-multilinear $F$, Corollary~\ref{cor:mobius-correction} no
longer provides a complete closed-form characterization of the
structural completeness residual, because within-coordinate curvature
terms ($\partial^2F/\partial x_i^2\ne0$) may additionally contribute.
Theorem~\ref{thm:mc} controls only the Monte Carlo and quadrature error
incurred when approximating Definition~\ref{def:explicit}'s estimator
$\phi^{\GRALIS}$ itself; it does \emph{not} bound this additional
structural residual, which is a separate quantity
($\Delta F-\sum_i\phi_i^{\GRALIS}$, not
$\hat\phi^{(m,k)}-\phi^{\GRALIS}$). Characterizing it in general would
require a separate Taylor-remainder-type theorem, not attempted here.
\end{remark}

\begin{remark}[Approximate completeness in Algorithm~1]
\label{rem:approx_completeness}
Theorem~\ref{thm:completeness} holds for the abstract GRALIS form with the
constitutive constraint $\sum_i \Delta_i(q) = f(x)-f(x')$ for $\mu$-almost
every $q\in\mathcal{Q}$.
Algorithm~1's practical estimator deviates from exact completeness for up
to three reasons, of which two are fully characterized and one is
characterized only for the unnormalized numerator: (i)~the locality
kernel $\pi_x(S) = \exp(-\|x_S - x'_S\|^2 / 2\sigma^2)$ is a non-uniform
weight on Shapley coalitions; Proposition~\ref{prop:incompatibility}
below shows this is structurally incompatible with exact completeness of
the \emph{unnormalized numerator} $N_i$ whenever $\pi_x$, restricted to the
proper subsets $S\subsetneq\mathcal F$ that actually occur as predecessor
sets, is non-constant, and Remark~\ref{rem:normalized-incompatibility}
verifies directly (for $n=2$, without a general-$n$ proof) that the
specific Gaussian kernel used here also breaks completeness of the
\emph{normalized} $\phi_i$ that Algorithm~1 returns, though normalization
can in principle restore completeness for other, specially tuned
non-constant kernels; (ii)~even under $\pi_x\equiv1$, the joint
coalition-conditioned path attenuates interaction terms of order $d\ge2$
by a factor $1/d$, exactly quantified in
Corollary~\ref{cor:mobius-correction}; (iii)~finite $m,k$ introduce Monte
Carlo and quadrature error (Theorem~\ref{thm:mc}).
The practical consequence is approximate completeness
($|\sum_i\phi_i^{\GRALIS} - (f(x)-f(x'))| \approx 0.40$ on BreaKHis).
Setting $\pi_x \equiv 1$ removes source~(i) only (recovering KernelSHAP's
locality trade-off, not exact completeness in general); sources~(ii) and
(iii) persist whenever $F$ has genuine order-$\ge2$ interactions, which is
the typical case for a deep network.
This trade-off is analysed empirically in the companion paper~\citep{fanale2026b}.
\end{remark}

\begin{proposition}[Structural incompatibility of a non-constant locality kernel --- unnormalized numerator]
\label{prop:incompatibility}
\textbf{Scope (read before applying this proposition to Algorithm~1).} This
proposition is about the \emph{unnormalized numerator}
\[
  N_i \;:=\; \sum_{S\subseteq\mathcal F\setminus\{i\}}
  w_{\mathrm{Sh}}(|S|)\,\pi_x(S)\,[v(S\cup\{i\})-v(S)],
\]
\emph{not} about the quantity $\phi_i^{\GRALIS}=N_i/Z_{x,i}$ that
Algorithm~1 and Definition~\ref{def:explicit} actually return (recall
$Z_{x,i}=\sum_{S\subseteq\mathcal F\setminus\{i\}}w_{\mathrm{Sh}}(|S|)\pi_x(S)$).
The two coincide only when $Z_{x,i}\equiv1$ for every $i$, i.e.\ under
$\pi_x\equiv1$ (Lemma~\ref{lem:shapley-weight}). For the \emph{normalized}
estimator that Algorithm~1 actually outputs, the conclusion below is
\textbf{false as a universal statement}: Remark~\ref{rem:normalized-incompatibility}
gives an explicit non-constant kernel, for $n=2$, under which the
normalized $\phi_1+\phi_2=v(N)-v(\emptyset)$ holds exactly for \emph{every}
game $v$. We keep this proposition, restricted honestly to $N_i$, because
(a) it is mathematically correct as stated, (b) it is the quantity Remark
\ref{rem:approx_completeness} source~(i) actually needs as a
building block, and (c) Remark~\ref{rem:normalized-incompatibility} shows
why extending it to the normalized $\phi_i$ requires additional,
kernel-specific analysis rather than following automatically.

Let $w_{\mathrm{Sh}}(|S|) = |S|!(n-|S|-1)!/n!$ and a locality kernel
$\pi_x: 2^{\mathcal{F}} \to \mathbb{R}_{>0}$. Only the restriction of
$\pi_x$ to \emph{proper} subsets $S\subsetneq\mathcal F$ enters the formula
for $N_i$, since every predecessor set
$S\subseteq\mathcal F\setminus\{i\}$ satisfies $S\subsetneq\mathcal F$; call
this restriction the \emph{effective} kernel. If the effective kernel is
non-constant --- i.e.\ $\pi_x(S)\ne\pi_x(S')$ for some
$S,S'\subsetneq\mathcal F$ --- then there exists a game $v$ such that
$\sum_{i=1}^n N_i\ne v(N)-v(\emptyset)$. Conversely,
$\pi_x\equiv1$ on $\{S\subsetneq\mathcal F\}$ is the \emph{unique} effective
kernel for which exact efficiency $\sum_i N_i=v(N)-v(\emptyset)$
holds for \emph{every} game $v$. (Non-constancy of $\pi_x$ on the full power
set $2^{\mathcal F}$ that is due entirely to the value $\pi_x(\mathcal F)$ at
the full coalition --- which is never used as a predecessor set --- does
\emph{not} by itself break efficiency; this is why the proposition is
stated for the effective kernel rather than for $\pi_x$ on all of
$2^{\mathcal F}$.)
\end{proposition}
\begin{proof}[Proof idea]
Summing $N_i$ over $i$ and rearranging by coalition $T$, the coefficient
of $v(T)$ in $\sum_iN_i$ vanishes for every $T$ if and only if
$\pi_x(T)=\tfrac1{|T|}\sum_{i\in T}\pi_x(T\setminus\{i\})$ together with
boundary conditions $\pi_x(\emptyset)=1$ and
$\tfrac1n\sum_i\pi_x(\mathcal F\setminus\{i\})=1$; by induction on $|T|$
this recursion forces $\pi_x\equiv1$ on every proper subset, so any
effective kernel violating it must fail efficiency for some game $v$
(constructed explicitly by concentrating $v$ on the first coalition where
the coefficient is nonzero). Full derivation in
Appendix~\ref{app:t3-deferred}.
\end{proof}

\begin{remark}[Normalization can restore efficiency for special kernels: an $n=2$ counterexample]
\label{rem:normalized-incompatibility}
Proposition~\ref{prop:incompatibility} does \emph{not} extend to the
normalized $\phi_i=N_i/Z_{x,i}$ by simply dividing through, since each
feature has its own denominator. For $n=2$ with
$\pi(\emptyset)=a,\pi(\{1\})=b,\pi(\{2\})=c$, direct coefficient matching
shows normalized efficiency holds for \emph{every} game $v$ iff
$a^2=bc$ --- a strictly larger family than the unnormalized case's unique
solution $a=b=c=1$ (worked example and full computation in
Appendix~\ref{app:t3-deferred}). For the actual Gaussian kernel
$\pi_x(x_S)=\exp(-\|x_S-x'_S\|^2/2\sigma^2)$, this condition reduces to
$1=\exp(-\|x-x'\|^2/2\sigma^2)$, which holds only in the degenerate limit
$\sigma\to\infty$ or $x=x'$: for finite $\sigma$ and $x\ne x'$ the
Gaussian kernel's normalized estimator generically still breaks exact
efficiency for $n=2$, consistent with (but not, for $n>2$, formally
established by) the empirical $\approx0.40$ residual of
Remark~\ref{rem:approx_completeness}. The general-$n$ normalized
statement is an open question.
\end{remark}

\subsection{Theorem 3: Monte Carlo Convergence of the Numerator Accumulator}
\label{sec:t3}

\textbf{Scope.} GRALIS-MC as actually implemented (Algorithm~\ref{alg:mc})
returns the \emph{self-normalized ratio} $\hat N_i/\hat Z_i$ of two
correlated Monte Carlo sums. The theorem below bounds exactly the random
variable that Algorithm~\ref{alg:mc}'s inner loop accumulates into
$\varphi[i]$ at each of its $m$ permutations, \emph{before} the final
division by $Z[i]$: for a uniformly random permutation $\pi$, the induced
predecessor set $S_\pi(i)=\{j:\pi(j)<\pi(i)\}$ follows the standard
Shapley predecessor-set law (Lemma~\ref{lem:shapley-weight}), and the
locality kernel $\pi_x$ enters only as an \emph{importance weight} on the
sampled term, not as a resampling density:
\[
  \hat{N}_i^{(m,k)} = \frac{1}{m}\sum_{r=1}^m \pi_x(x_{S_r})\,
  \mathrm{IG}_i^{(k)}(S_r), \qquad S_r \stackrel{\text{i.i.d.}}{\sim}
  \mathrm{Shapley\ predecessor\ law},
\]
where $\mathrm{IG}_i^{(k)}(S) := (x_i-x'_i)\cdot\tfrac1k\sum_{j=1}^k
\nabla_iF(\tilde x^{(S,(j-0.5)/k)})$ is the $k$-step midpoint
approximation of the coalition-conditioned IG term, with population mean
$N_i=\sum_Sw_{\mathrm{Sh}}(|S|)\pi_x(x_S)\mathrm{IG}_i(S)$ --- exactly the
unnormalized numerator of Proposition~\ref{prop:incompatibility} and of
$Z_{x,i}\cdot\phi_i^{\GRALIS}=N_i$. The theorem's target is therefore
$N_i$, not $\phi_i^{\GRALIS}$ directly; Remark~\ref{rem:ratio-estimator}
and Corollary~\ref{cor:ratio-bound} below bridge the remaining gap from
$\hat N_i/\hat Z_i$ to $N_i/Z_i=\phi_i^{\GRALIS}$.

To keep the Monte Carlo error (finite $m$) and quadrature error (finite
$k$) formally distinct, we use one intermediate quantity, the
\emph{population $k$-step numerator} $N_i^{(k)}:=\sum_S
w_{\mathrm{Sh}}(|S|)\pi_x(x_S)\mathrm{IG}_i^{(k)}(S)$, i.e.\ $N_i$ with
the exact IG term replaced by its $k$-step midpoint approximation before
taking the expectation over $S$: $\hat N_i^{(m,k)}$ is unbiased for
$N_i^{(k)}$, not for $N_i$ itself, and the deterministic gap
$|N_i^{(k)}-N_i|$ is handled separately by the quadrature argument in the
proof (Appendix~\ref{app:t3-deferred}).

\begin{theorem}[GRALIS-MC Convergence of the numerator accumulator]
\label{thm:mc}
Assume $F \in C^3$ with $\|\nabla^3 F\|_\infty < \infty$, where
$\|\nabla^3 F\|_\infty := \sup_{z\in B(x,x')}\,\max_{i,j,l}\,
\bigl|\partial^3 F(z)/\partial x_i\partial x_j\partial x_l\bigr|$
is the entrywise sup-norm of the third-order derivative tensor of $F$,
with the supremum taken over the axis-aligned box
$B(x,x') := \prod_j[\min(x_j,x'_j),\max(x_j,x'_j)]$, \emph{not} merely
the segment $\{x'+\alpha(x-x'):\alpha\in[0,1]\}$: for a coalition
$S\subsetneq\mathcal F\setminus\{i\}$, the evaluation point
$\tilde x^{(S,\alpha)}$ of Definition~\ref{def:explicit} advances only
the coordinates in $S\cup\{i\}$ by $\alpha$ while the rest remain at
$x'_j$ (e.g.\ $x'=(0,0)$, $x=(1,1)$, $S=\emptyset$, $i=1$ gives the point
$(\alpha,0)$, which lies on the segment only at its endpoints), so every
$\tilde x^{(S,\alpha)}$ lies in $B(x,x')$ but generally not on the
segment; the box is the smallest axis-aligned region containing every
point the estimator can evaluate.
With probability $1 - \delta$, the numerator accumulator $\hat N_i^{(m,k)}$
defined above (\emph{not} $\phi_i^{\GRALIS}$ itself, and \emph{not yet} the
ratio Algorithm~\ref{alg:mc} returns; see Scope above) satisfies:
\[
  \Bigl|\hat{N}_i^{(m,k)} - N_i\Bigr|
  \;\le\; \underbrace{\frac{B}{\sqrt{m\delta}}}_{\text{MC error}}
  \;+\; \underbrace{\frac{|x_i-x'_i|\cdot\|x-x'\|_1^2\cdot\|\nabla^3 F\|_\infty}{24\,k^2}}_{\text{midpoint quadrature error}},
\]
where $B = \sup_{S,\alpha} |\pi_x(x_S)(x_i - x'_i)\nabla_i F(\tilde{x}^{(S,\alpha)})|$
is exactly the pointwise bound on the sampled summand
$\pi_x(x_{S})\,\nabla_i F$-term, matching the estimator as defined.
Complexity drops from $O(n \cdot 2^n \cdot k)$ (exact: $2^{n-1}$
predecessor coalitions per feature, for $n$ features) to
$O(m \cdot n \cdot k)$.
\end{theorem}

\begin{proof}[Proof idea]
The bound combines a stochastic and a deterministic term. \emph{Stochastic
term:} $\hat N_i^{(m,k)}$ is an average of i.i.d.\ terms
$Q_{\pi,i}^{(k)}:=\pi_x(x_{S_\pi(i)})\mathrm{IG}_i^{(k)}(S_\pi(i))$, each
bounded by $B$ and unbiased for the population $k$-step numerator
$N_i^{(k)}$ (not for $N_i$ itself), so Chebyshev gives
$|\hat N_i^{(m,k)}-N_i^{(k)}|\le B/\sqrt{m\delta}$ with probability
$\ge1-\delta$. \emph{Deterministic term:} the midpoint rule's error
$|\mathrm{IG}_i^{(k)}(S)-\mathrm{IG}_i(S)|$ is controlled, for every
coalition $S$ uniformly, by a Taylor expansion of
$g(\alpha)=\nabla_iF(\tilde x^{(S,\alpha)})$ around each subinterval
midpoint (the linear term cancels by symmetry, leaving a
$\|\nabla^3F\|_\infty$-controlled second-order remainder), giving
$E_{\mathrm{quad}}(k)=O(1/k^2)$; taking expectation over
$S\sim\mathrm{Shapley}$ transfers this bound deterministically from
$N_i^{(k)}$ to $N_i$. The triangle inequality on these two pieces gives
the stated bound. Full four-step derivation, and the accompanying
antithetic-sampling variance-reduction argument for
Algorithm~\ref{alg:mc}, are in Appendix~\ref{app:t3-deferred}.
\end{proof}

\begin{remark}[Fallback bound under weaker smoothness]
\label{rem:riemann-fallback}
The $O(1/k^2)$ rate requires $F\in C^3$, one order beyond the $C^2$
hypothesis used elsewhere in the paper. If only $F\in C^2$ is available,
replacing the midpoint rule with an endpoint (Riemann) rule recovers a
weaker but unconditional $O(1/k)$ bound,
$|x_i-x'_i|\cdot\|x-x'\|_1\cdot\|\nabla^2F\|_\infty/2k$, by the same
Taylor argument without the midpoint symmetry cancellation.
Algorithm~\ref{alg:mc} uses the midpoint rule and so attains the sharper
rate whenever $C^3$ holds.
\end{remark}

The GRALIS-MC algorithm is given in Algorithm~\ref{alg:mc}.

\begin{figure}[ht]
\centering
\fbox{\begin{minipage}{0.88\linewidth}
\textbf{Algorithm: GRALIS-MC} \label{alg:mc}\\[4pt]
\textbf{Input:} $F$, $x$, $x'$, samples $m$, steps $k$, bandwidth $\sigma$ \\
\textbf{Output:} $\varphi[1..n]$ --- normalized attributions \\[4pt]
$\varphi[i] \leftarrow 0$, $Z[i] \leftarrow 0$ for each $i$ \\
\textbf{for} $t = 1$ \textbf{to} $m$: \\
\quad $\pi \leftarrow$ uniform random permutation of $\{1..n\}$ \\
\quad $S \leftarrow \emptyset$ \\
\quad \textbf{for} $i$ in order $\pi$: \\
\quad\quad $\mathrm{ig\_val} \leftarrow (x_i - x'_i) \cdot \tfrac{1}{k}
  \sum_{j=1}^{k} \nabla_i F(\tilde{x}^{(S,(j-0.5)/k)})$ \quad [midpoint] \\
\quad\quad $\pi_w \leftarrow \exp\!\bigl(-\norm{x_S - x'_S}^2 / 2\sigma^2\bigr)$ \\
\quad\quad $\varphi[i] \mathrel{+}= \pi_w \cdot \mathrm{ig\_val}$;\ \
  $Z[i] \mathrel{+}= \pi_w$ \\
\quad\quad $S \leftarrow S \cup \{i\}$ \\
\textbf{return} $\varphi[i] / Z[i]$ for each $i$ (or $\varphi[i]$ if $Z[i]=0$)
\end{minipage}}
\caption{GRALIS-MC pseudocode. The per-feature normalizer $Z[i]$ matches
Definition~\ref{def:explicit}'s $Z_{x,i}$, not a running total shared
across features (Remark~\ref{rem:ratio-estimator}); the importance-weight
interpretation of $\pi_x$ is as in the Scope paragraph above.}
\end{figure}

\begin{remark}[Scope of Theorem~\ref{thm:mc} relative to the implemented estimator]
\label{rem:ratio-estimator}
Algorithm~\ref{alg:mc} accumulates $Z$ \emph{per feature}, matching
$Z_{x,i}$ (a shared running total across features would coincide with it
only when $\pi_x\equiv1$). Consequently the returned $\varphi/Z$ is a
\emph{self-normalized (ratio) estimator}: $\varphi$ and $Z$ are both
random functions of the same $m$ permutations, so
$\mathbb E[\varphi/Z]\ne\mathbb E[\varphi]/\mathbb E[Z]$ in general, and
Theorem~\ref{thm:mc} --- which controls only the numerator
$|\hat N_i^{(m,k)}-N_i|$ --- does not by itself control the ratio.
Corollary~\ref{cor:ratio-bound} below closes this gap: it gives an
analogous concentration bound for $\hat Z_i:=Z[i]/m$ around $Z_i:=Z_{x,i}$
(with a deterministic, not merely high-probability, lower bound on
$\hat Z_i$) and assembles both into a finite-sample bound for
$\hat N_i/\hat Z_i=\varphi[i]/Z[i]$, i.e.\ for Algorithm~\ref{alg:mc}'s
actual output.
\end{remark}

\begin{corollary}[Finite-sample bound for the self-normalized ratio estimator]
\label{cor:ratio-bound}
Fix $x,x',\sigma>0$ and a feature $i$. Since
$\pi_x(x_S)=\exp(-\|x_S-x'_S\|^2/2\sigma^2)>0$ for every one of the
finitely many ($2^{n-1}$) predecessor coalitions
$S\subseteq\mathcal F\setminus\{i\}$, define
\begin{align*}
  \pi_{\min,i} &\;:=\; \min_{S\subseteq\mathcal F\setminus\{i\}} \pi_x(x_S) \;>\;0,
  \qquad
  B_Z \;:=\; \sup_{S\subseteq\mathcal F\setminus\{i\}}\pi_x(x_S) \;\le\; 1,\\
  B_N &\;:=\; B \ \text{(as in Theorem~\ref{thm:mc}).}
\end{align*}
Write $\hat N_i:=\varphi[i]/m$, $\hat Z_i:=Z[i]/m$ for
Algorithm~\ref{alg:mc}'s running totals after $m$ permutations (so
$\varphi[i]/Z[i]=\hat N_i/\hat Z_i$, and $\hat N_i=\hat N_i^{(m,k)}$ in
Theorem~\ref{thm:mc}'s notation --- Algorithm~\ref{alg:mc} always uses
the $k$-step discretized IG term, so its output is never the $k\to\infty$
idealization), and $N_i$, $Z_i:=Z_{x,i}$ for their respective population
means. Let $E_{\mathrm{quad}}(k)$ be the deterministic midpoint
quadrature error of Theorem~\ref{thm:mc}. With probability at least
$1-\delta$ (over the randomness of the $m$ sampled permutations),
\begin{align*}
  \left|\frac{\varphi[i]}{Z[i]}-\phi_i^{\GRALIS}(x)\right|
  \;=\;
  \left|\frac{\hat N_i}{\hat Z_i}-\frac{N_i}{Z_i}\right|
  &\;\le\;
  \frac{B_N/\sqrt{m\delta/2} + E_{\mathrm{quad}}(k)}{\pi_{\min,i}}
  \;+\;
  \frac{B_N B_Z}{\pi_{\min,i}^2\sqrt{m\delta/2}}.
\end{align*}
As $k\to\infty$ (exact path integration), $E_{\mathrm{quad}}(k)\to0$ and
this reduces to a purely stochastic bound; for
finite $k$, the additional deterministic term
$E_{\mathrm{quad}}(k)/\pi_{\min,i}$ is required and is \emph{not}
negligible in general, since it does not vanish as $m\to\infty$.
\end{corollary}

\begin{proof}[Proof idea]
Since $\pi_x(x_S)\ge\pi_{\min,i}$ for every one of the finitely many
predecessor coalitions, the average $\hat Z_i$ is bounded away from zero
\emph{deterministically} for every realization of $S_1,\dots,S_m$ --- no
concentration argument is needed for the denominator, because the kernel
is already bounded away from zero on the finite coalition support.
Theorem~\ref{thm:mc} gives a high-probability bound on $|\hat N_i-N_i|$,
and an identical Chebyshev argument bounds $|\hat Z_i-Z_i|$; a union bound
at level $\delta/2$ each makes both hold simultaneously with probability
$\ge1-\delta$. On that event, the elementary ratio decomposition
$|\hat N_i/\hat Z_i-N_i/Z_i|\le|\hat N_i-N_i|/\hat Z_i +
|N_i|\,|\hat Z_i-Z_i|/(\hat Z_i Z_i)$, together with $|N_i|\le B_N$ and
the deterministic denominator bounds, yields the stated bound. Full
derivation in Appendix~\ref{app:t3-deferred}.
\end{proof}

\begin{remark}[Reading the bound]
The bound has two qualitatively different terms: an $O(1/\sqrt m)$
stochastic term, vanishing as $m\to\infty$, and a deterministic
$E_{\mathrm{quad}}(k)/\pi_{\min,i}$ term that does \emph{not} shrink with
$m$ --- only $k\to\infty$ removes it, so more permutations $m$ alone does
not drive $\varphi[i]/Z[i]$ arbitrarily close to $\phi_i^{\GRALIS}(x)$;
both $m$ and $k$ must grow. The constant also depends on
$1/\pi_{\min,i}$ and $1/\pi_{\min,i}^2$, which grow as $\sigma\to0$: the
quantitative form of the bandwidth caveat already noted around
Proposition~\ref{prop:incompatibility} --- the self-normalized estimator
degrades as the locality bandwidth shrinks, even though each individual
concentration inequality is otherwise dimension-free in $\sigma$.
\end{remark}

\begin{corollary}[Aggregate numerical completeness bound]
\label{cor:aggregate-completeness}
Corollary~\ref{cor:ratio-bound} bounds a single feature; it does not by
itself bound the sum $\sum_i\varphi[i]/Z[i]$ against
$\sum_i\phi_i^{\GRALIS}(x)$ at the same confidence level $\delta$, since
combining $n$ per-feature high-probability events requires an explicit
union bound. Write $\varepsilon_i(\delta)$ for the right-hand side of
Corollary~\ref{cor:ratio-bound}'s bound for feature $i$, as a function of
the target confidence $\delta$ (with $m,k,\pi_{\min,i},B_N,B_Z$ fixed):
\[
  \varepsilon_i(\delta) \;:=\;
  \frac{B_N/\sqrt{m\delta/2} + E_{\mathrm{quad}}(k)}{\pi_{\min,i}}
  \;+\;
  \frac{B_N B_Z}{\pi_{\min,i}^2\sqrt{m\delta/2}},
  \qquad\text{so}\qquad
  P\Bigl(\Bigl|\tfrac{\varphi[i]}{Z[i]}-\phi_i^{\GRALIS}(x)\Bigr|>\varepsilon_i(\delta)\Bigr)\le\delta.
\]
Applying this with $\delta_i=\delta/n$ for each $i=1,\ldots,n$ and a
union bound,
\[
  P\Bigl(\exists i:\ \Bigl|\tfrac{\varphi[i]}{Z[i]}-\phi_i^{\GRALIS}(x)\Bigr|>\varepsilon_i(\delta/n)\Bigr)
  \;\le\; \sum_{i=1}^n\frac{\delta}{n} \;=\; \delta,
\]
so with probability at least $1-\delta$, \emph{every} feature
simultaneously satisfies
$|\varphi[i]/Z[i]-\phi_i^{\GRALIS}(x)|\le\varepsilon_i(\delta/n)$. On that
event, the triangle inequality gives
\[
  \Bigl|\sum_{i=1}^n\tfrac{\varphi[i]}{Z[i]} - \sum_{i=1}^n\phi_i^{\GRALIS}(x)\Bigr|
  \;\le\; \sum_{i=1}^n\varepsilon_i(\delta/n).
\]
Writing $R_{\mathrm{obs}}:=\Delta F-\sum_i\varphi[i]/Z[i]$ for
Algorithm~\ref{alg:mc}'s actually observed residual and
$R_{\mathrm{pop}}:=\Delta F-\sum_i\phi_i^{\GRALIS}(x)$ for the population
structural residual of the practical estimator, the result becomes
\[
  P\Bigl(|R_{\mathrm{obs}}-R_{\mathrm{pop}}|\;\le\;\sum_{i=1}^n\varepsilon_i(\delta/n)\Bigr)\;\ge\;1-\delta.
\]
\S\ref{sec:c4} identifies two mechanisms contributing to
$R_{\mathrm{pop}}$: the joint-path effect characterized exactly by
Corollary~\ref{cor:mobius-correction} under its hypotheses, and the
kernel-induced effect characterized by
Proposition~\ref{prop:incompatibility} for the unnormalized numerator and
by Remark~\ref{rem:normalized-incompatibility} for the normalized $n=2$
case. No complete additive decomposition of $R_{\mathrm{pop}}$ for the
normalized general-$n$ estimator is claimed.
\end{corollary}
\begin{proof}
Immediate from Corollary~\ref{cor:ratio-bound} applied $n$ times at
level $\delta/n$, a union bound over the $n$ resulting events, and the
triangle inequality on the event that all $n$ hold simultaneously, as
displayed above.
\end{proof}

\begin{remark}[Why this must start from Corollary~\ref{cor:ratio-bound}, not Theorem~\ref{thm:mc}]
\label{rem:aggregate-scope}
Theorem~\ref{thm:mc} bounds $|\hat N_i^{(m,k)}-N_i|$, the unnormalized
numerator accumulator, not the ratio $\varphi[i]/Z[i]$ that
Algorithm~\ref{alg:mc} actually returns and that
$R_{\mathrm{obs}}$ is computed from. Building the aggregate bound
directly on Theorem~\ref{thm:mc} would certify a quantity
($\sum_i\hat N_i^{(m,k)}$) that Algorithm~\ref{alg:mc} never outputs;
routing through Corollary~\ref{cor:ratio-bound} is what makes
Corollary~\ref{cor:aggregate-completeness} a statement about the
estimator actually run in \S\ref{sec:experiments}, not about an
intermediate quantity.
\end{remark}

\begin{remark}[The aggregate certificate is conservative]
The aggregate certificate of Corollary~\ref{cor:aggregate-completeness}
is deliberately conservative: the featurewise union bound does not
exploit the dependence induced by the shared sampled permutations across
features. Sharper joint/vector concentration bounds may therefore
substantially tighten $\sum_i\varepsilon_i(\delta/n)$ and are left for
future work.
\end{remark}

\begin{table}[ht]
\centering
\caption{Comparative computational complexity.}
\label{tab:complexity}
\small
\begin{tabular}{lcccc}
\toprule
\textbf{Method} & \textbf{$F$ evals.} & \textbf{Backward pass} &
\textbf{Completeness} & \textbf{Locality} \\
\midrule
GRALIS exact      & $O(n\cdot2^n \cdot k)$ & $O(n\cdot2^n \cdot k)$ & $\approx^*$ & $\approx^\dagger$ \\
GRALIS-MC         & $O(m \cdot n \cdot k)$ & $O(m \cdot n \cdot k)$ & $\approx^*$ & $\approx^\dagger$ \\
KernelSHAP        & $O(m \cdot n)$   & ---                 & \checkmark & $\times$ \\
Standard IG       & $O(k)$           & $O(k)$              & \checkmark & $\times$ \\
LIME              & $O(T)$           & ---                 & $\times$   & \checkmark \\
GradCAM           & ---              & $O(1)$              & $\times$   & $\times$ \\
\bottomrule
\end{tabular}
\smallskip

{\small $^*$``Exact'' here means exact enumeration of all $2^{n-1}$
predecessor coalitions in Definition~\ref{def:explicit}'s explicit sum
(no Monte Carlo sampling), not exact completeness in the sense of
Definition~\ref{def:canonical}: the non-constant locality kernel effect
of Proposition~\ref{prop:incompatibility} and the Möbius correction of
Corollary~\ref{cor:mobius-correction} persist even with exact enumeration
and $m=\infty$; both rows are approximately complete, exactly as scored
in Table~\ref{tab:axioms}. $^\dagger$Same baseline-proximity kernel
$\pi_x$ as Table~\ref{tab:axioms}: maximized at the empty coalition, not
centered at $x$, for both the exact-enumeration and Monte Carlo
variants.}
\end{table}

\subsection{Theorem 4: Shapley Interaction Values}
\label{sec:t4}

\textbf{What is, and is not, GRALIS-specific here.} Theorem~\ref{thm:siv}
below is standard Shapley Interaction Value theory
(Grabisch~\&~Roubens~\citep{grabisch1999}) applied to a coalition game
$\vG$ read directly off $(f,x,x')$; it is included as a prerequisite
for Corollary~\ref{cor:siv-gralis-relation}, which is where the actual
GRALIS-specific content lives (the exact, term-by-term relationship
between $\mathrm{Sh}(\vG)$ and $\phi^{\GRALIS}$). Remark~\ref{rem:siv-scope}
after the theorem gives the full division of labor.

GRALIS associates the standard discrete cooperative game below to the
coalition lattice $2^{\mathcal F}$ of the (already finite) feature set
$\mathcal F$ used throughout \S\ref{sec:defs} --- the same lattice
underlying SHAP and the classical Shapley-value XAI literature --- rather
than through the continuous index space $\mathcal Q$:
\begin{equation}
  \label{eq:vG}
  \vG(S) \;:=\; f(x_S) - f(x'), \qquad S \subseteq \mathcal F,
\end{equation}
where $x_S$ is the point with coordinate $x_j$ for $j\in S$ and $x'_j$
otherwise, exactly as in Lemma~A (proof of
Theorem~\ref{thm:completeness}). Unlike a construction routed through
$\mathcal Q$, \eqref{eq:vG} requires no measure theory: it is a
finite-valued function on a finite lattice, automatically satisfying
$\vG(\emptyset) = f(x')-f(x') = 0$ and
$\vG(\mathcal F) = f(x)-f(x')$.

\begin{theorem}[Shapley Interaction Values]
\label{thm:siv}
The Shapley Interaction Values of $\vG$ are:
\[
  \IijSh(\vG) \;=\;
  \sum_{S \subseteq \mathcal F\setminus\{i,j\}} \pi_n(|S|)
  \bigl[\vG(S\cup\{i,j\}) - \vG(S\cup\{i\}) - \vG(S\cup\{j\}) + \vG(S)\bigr],
\]
where $\pi_n(|S|) = |S|!(n-|S|-2)!/(n-1)!$ (Shapley-interaction weight,
Grabisch \& Roubens~\citep{grabisch1999}), computed \emph{exactly} on the
finite game $\vG$ by the standard theory of cooperative games
(Appendix~\ref{app:mobius}).
\end{theorem}
\begin{proof}
Immediate from the Grabisch--Roubens definition of the Shapley
interaction index applied to the well-defined game $\vG:2^{\mathcal
F}\to\R$ of \eqref{eq:vG} (Appendix~\ref{app:mobius},
Proposition~\ref{prop:mobius}); no measurability or integrability
hypothesis is needed beyond $f$ being real-valued, since $\vG$ is
already finite-valued on the finite lattice $2^{\mathcal F}$.
\end{proof}

\begin{remark}[Scope: $\IijSh(\vG)$ is a property of the auxiliary game
  $\vG$, not automatically an interaction decomposition of
  $\phi^{\GRALIS}$ itself]
\label{rem:siv-scope}
Theorem~\ref{thm:siv} computes exact Shapley interaction values of the
auxiliary game $\vG(S)=f(x_S)-f(x')$ of \eqref{eq:vG}, not of the
practical joint-path attribution $\phi^{\GRALIS}$ itself: $\vG$ is the
ordinary baseline-clamped coalition game determined by $f,x,x'$ alone
(independent of GRALIS's own weight $w$, kernel $\pi_x$, or joint path),
and computing exact Shapley Interaction Values on an already-finite game
is standard cooperative game theory
(Grabisch~\&~Roubens~\citep{grabisch1999}), not a GRALIS-specific
construction. The GRALIS-specific content is
Corollary~\ref{cor:siv-gralis-relation} below, which relates
$\mathrm{Sh}(\vG)$ to $\phi^{\GRALIS}$ itself: the two coincide exactly
for affine $F$ and differ by an exactly quantified order-dependent factor
for multilinear $F$, since $\mathrm{Sh}(\vG)$ corresponds to a
per-coalition \emph{clamp} path while $\phi^{\GRALIS}$ uses the
\emph{joint} coalition-conditioned path of
Definition~\ref{def:explicit}. Appendix~\ref{app:mobius} gives the full
derivation via the Möbius transform and further discussion of scope,
including why $\vG$ is built directly from $f$ on the finite lattice
$2^{\mathcal F}$ rather than from the continuous index space
$\mathcal Q$.
\end{remark}

\begin{corollary}[Relation between $\phi^{\GRALIS}$ and $\mathrm{Sh}(\vG)$]
\label{cor:siv-gralis-relation}
Under the hypotheses of Corollary~\ref{cor:mobius-correction} (multilinear
$F$ with Möbius coefficients $m(T)$), for Shapley weights and
$\pi_x\equiv1$:
\[
  \mathrm{Sh}_i(\vG) \;=\; \sum_{T\ni i} \frac{m(T)}{|T|}\prod_{j\in T}(x_j-x'_j),
  \qquad
  \phi_i^{\GRALIS} \;=\; \sum_{T\ni i} \frac{m(T)}{|T|^2}\prod_{j\in T}(x_j-x'_j).
\]
Consequently $\phi_i^{\GRALIS}=\mathrm{Sh}_i(\vG)$ exactly whenever $F$ is
affine ($m(T)=0$ for $|T|\ge2$), and for a single pure interaction of
order $d=|T|\ni i$, $\phi_i^{\GRALIS}\big|_{F_T} =
\tfrac1d\,\mathrm{Sh}_i(\vG)\big|_{F_T}$: the joint-path GRALIS
attribution is exactly the clamp-path Shapley value of $\vG$, attenuated
by the same $1/d$ factor as in Corollary~\ref{cor:mobius-correction}.
Verified numerically on the toy example of
Appendix~\ref{app:toy-verification}.
\end{corollary}
\begin{proof}
$\mathrm{Sh}_i(\vG)$ is the classical Shapley value of the sum of
unanimity games $u_T$ scaled by $m(T)\prod_{j\in T}(x_j-x'_j)$; each
$u_T$ distributes value $1/|T|$ to every $i\in T$ and $0$ to $i\notin T$
(the same fact already used in the proof of
Corollary~\ref{cor:mobius-correction}, line~\eqref{eq:mobius-correction}
above). The formula for $\phi_i^{\GRALIS}$ is exactly the per-$i$,
per-$T$ term isolated in that same proof
($\phi_i^{\GRALIS}|_{F_T}=m(T)/d^2\cdot\prod$), there summed over $i\in
T$ to recover the aggregate deficit; retaining the sum over $T\ni i$
instead gives the stated formula. Comparing the two term-by-term for a
fixed $T$ gives the ratio $1/|T|$.
\end{proof}

\begin{remark}[Canonical operators for imaging]
\label{rem:canonical-imaging}
For $\Omega = \{(k,i,j)\}$ (channel$\times$position, the raw pixel index
space), three canonical segmentation operators
$\rho_{\bullet}:\Omega\to N$ (single-feature-valued, $N=\mathcal F$):
(i) \emph{SLIC superpixels}~\citep{achanta2012}: $\rho_{\mathrm{slic}}(k,i,j) = s$
where $s$ is the SLIC segment containing $(i,j)$; recommended for medical imaging.
(ii) \emph{Channel grouping}: $\rho_{\mathrm{chan}}(k,i,j) = k$.
(iii) \emph{Spatial grid}: $\rho_{\mathrm{patch}}(k,i,j) =$ patch of $(i,j)$.
These operators define the discrete features $\mathcal F$ (hence the
coalition lattice $2^{\mathcal F}$ of \eqref{eq:vG}) as a prior modeling
choice, made before GRALIS's own sampling space $\mathcal Q$ is
constructed on top of the resulting $\mathcal F$. This
$\rho_\bullet:\Omega\to N$ is a different, single-feature-valued map from
the coalition-valued $\rho:\mathcal Q\to2^N$ formalized as a measurable
partition in Appendix~\ref{app:projection}, which operates one level
downstream (on $\mathcal Q$, after $\mathcal F$ is already fixed) and is
used there for the weight-mass diagnostic of Remark~\ref{rem:weight-mass}
rather than for constructing $\vG$ or for segmenting pixels.

\textbf{On the input-dependence of $\rho_{\mathrm{slic}}$, and what this
means for Theorem~\ref{thm:canonical}'s condition~(a).} SLIC segmentation
is computed \emph{from the image being explained}, so different inputs
$x$ generally give different feature sets $\mathcal F_x$ and index spaces
$\mathcal Q_i(x)$: condition~(a)'s literal requirement that $\mathcal Q$
be ``fixed a priori, independently of $x$'' is not satisfied by $\mathcal
F_x$ itself varying with $x$. The theorem remains rigorous under the
reading actually used in this paper's experiments (\S\ref{sec:experiments}):
run the segmentation first, then treat the resulting $\mathcal F_x$ (hence
$\mathcal Q_i(x)$) as the fixed index space condition~(a) refers to, and
apply Theorem~\ref{thm:canonical} conditionally on that segmentation ---
a family of per-$x$ theorem instances rather than one instance covering a
varying segmentation space (formalizing this family over a measurable
bundle indexed by $x$ is left as future work). This does not invalidate
any use of Theorem~\ref{thm:canonical} here, since $\phi_i^{\GRALIS}(f,x,x')$
is already a function of $x$; it only means ``fixed a priori'' should be
read as ``fixed before GRALIS's own sampling begins, for the $x$ under
analysis,'' not ``identical across all $x$,'' whenever the feature set
comes from an input-dependent segmentation such as SLIC. Fixed,
input-independent segmentations ($\rho_{\mathrm{patch}}$,
$\rho_{\mathrm{chan}}$ above) satisfy condition~(a) literally, with no
such caveat.
\end{remark}

\subsection{Theorem 5: Exact Hoeffding Correspondence in the Affine Independent-Feature Regime}
\label{sec:t5}

\begin{remark}[What ``$\pi_x\to\mu$'' does and does not mean]
\label{rem:pix-to-mu}
The limit ``$\pi_x\to\mu$'' might appear as a natural hypothesis for
Theorem~\ref{thm:anova} (and Theorem~\ref{thm:sobol}), but it is
\emph{not needed at all} for the affine case: as shown in
Theorem~\ref{thm:anova}'s Step~0, the kernel cancels
exactly between numerator and denominator of the normalized GRALIS
attribution, for any fixed $\sigma>0$, so the affine theorem (and, below,
Theorem~\ref{thm:sobol}'s affine case) is stated and proved without it.
The limit remains relevant only to Proposition~\ref{prop:anova-general}'s
general smooth-$F$ case, where the cancellation argument does not apply.
There, taken literally, ``$\pi_x\to\mu$'' is not well typed: $\mu$ is a
product probability measure on the input random vector $X\in\R^n$, while
$\pi_x$ is a scalar kernel evaluated on predecessor coalitions/masked
inputs $x_S$, living on $2^{\mathcal F\setminus\{i\}}$ (or, jointly with the
path parameter, on $Q_i$) --- a different measurable space, with no
identification supplied between the two. A precise version of the intended
statement would instead define, for each $\sigma$, the induced probability
measure $\nu_\sigma$ on $Q_i$ obtained by normalizing
$w_{\mathrm{Sh}}(|S|)\,\pi_\sigma(S)\,d\alpha$, and assert weak or
total-variation convergence of $\nu_\sigma$ (or of the resulting pushforward
onto $x$-space via $\tilde x^{(S,\alpha)}$) to a precisely defined target
measure related to $\mu$, as $\sigma\to\infty$. We do not carry out this
construction in this paper. Consequently, ``$\pi_x\to\mu$'' in
Proposition~\ref{prop:anova-general} (and in Theorem~\ref{thm:sobol}'s
general-$F$ case, which inherits the proposition's hypotheses) should be
read as an informal shorthand for an intended limiting regime (the
locality kernel becoming increasingly uninformative, so that GRALIS's own
sampling weight approaches the input distribution), not as a formally
defined convergence statement; closing this gap is left as future work,
alongside the projector-identity hypothesis of
Proposition~\ref{prop:anova-general} itself.
\end{remark}

\begin{definition}[GRALIS-induced interaction terms --- affine case]
\label{def:phi-T-affine}
For affine $F(x)=a+\sum_j b_jx_j$ and any fixed baseline $x'$, define,
for every $\emptyset\ne T\subseteq\mathcal F$,
\[
  \Phi_T^{\GRALIS}(x) \;:=\;
  \begin{cases}
    \phi_i^{\GRALIS}(x) & T=\{i\},\ i\in\mathcal F,\\[2pt]
    0 & |T|\ge2,
  \end{cases}
\]
where $\phi_i^{\GRALIS}(x)$ is Algorithm~\ref{alg:mc}'s (equivalently
Definition~\ref{def:explicit}'s) actual per-feature output. This is the
\emph{only} definition of $\Phi_T^{\GRALIS}$ used in this subsection:
for $|T|=1$ it is literally what GRALIS computes; for $|T|\ge2$ it is
an explicit \emph{convention}, justified below (Step~1) by the fact
that affine $F$ has no order-$\ge2$ interaction content to attribute
--- not an independent GRALIS construction, and not the same object as
the exact Shapley Interaction Values $\IijSh(\vG)$ of
Theorem~\ref{thm:siv}, which live on the separate auxiliary game $\vG$
(Remark~\ref{rem:siv-scope}). This convention is specific to the affine
case; a comparable, independently justified definition for general
smooth $F$ is not given in this paper (see
Proposition~\ref{prop:anova-general} and the discussion following it).
\end{definition}

\begin{theorem}[Coincidence with Hoeffding Decomposition --- affine case]
\label{thm:anova}
Let $\mu = \bigotimes_i \mu_i$ be a product measure (feature independence
assumption) with $\mathbb E_\mu[\|X\|^2]<\infty$, which is a sufficient
condition ensuring $F\in L^2(\mu)$ for affine $F$ (not a necessary one:
e.g.\ $F(X_1,X_2)=X_1$ can be square-integrable even if
$\mathbb E[X_2^2]=\infty$) --- it is the finite \emph{second} moment,
not merely a finite first moment, that the
$L^2(\mu)$ Hoeffding decomposition invoked in H1 below actually
requires, $F$ affine, and set the
baseline $x'=\mathbb E_\mu[X]$ --- no other condition on the locality
kernel $\pi_x$ or its bandwidth $\sigma$ is imposed: \emph{any} fixed
$\sigma>0$ (in particular, any fixed strictly positive kernel $\pi_x$ ---
not only Gaussian; a fixed bandwidth determines one particular
strictly positive kernel, not every possible one, so this is not a
logical equivalence)
gives the same conclusion, for every $x$. Then the terms
$\Phi_T^{\GRALIS}$ of Definition~\ref{def:phi-T-affine} \emph{coincide
exactly} with the
terms of the Hoeffding functional decomposition~\citep{hoeffding1948}:
\[
  F(x) \;=\; \mathbb{E}_\mu[F] \;+\; \sum_{\emptyset \ne T \subseteq \mathcal{F}}
  \Phi_T^{\GRALIS}(x).
\]
This is a fully proved, unconditional statement for affine $F$ (proof
below), and does not
invoke the informal, not-formally-typed limit ``$\pi_x\to\mu$'' of
Remark~\ref{rem:pix-to-mu} at all: for affine $F$, the kernel cancels
\emph{exactly} between numerator and denominator of the normalized
GRALIS attribution (Step~0 of the proof below), so no limiting regime on
$\sigma$ is needed, and the theorem holds already at any fixed, finite
$\sigma$. Remark~\ref{rem:pix-to-mu}'s caveat about ``$\pi_x\to\mu$''
being informal remains relevant only to
Proposition~\ref{prop:anova-general}'s general smooth-$F$ case below,
which is not addressed by the affine cancellation argument.
Proposition~\ref{prop:anova-general} extends the same conclusion
to general smooth $F$, but only conditionally on an additional
hypothesis not established in this paper, and is stated separately so
that the affine case's unconditional status is not diluted by the
general case's conditional one.
\end{theorem}
\begin{proof}[Proof idea]
For affine $F(x)=a+\sum_jb_jx_j$, $\nabla_iF\equiv b_i$ is constant, so
the conditioned IG term is the same for every coalition $S$; this common
factor cancels between numerator and denominator of
Definition~\ref{def:explicit}, giving $\phi_i^{\GRALIS}(x)=b_i(x_i-x'_i)$
for \emph{any} strictly positive kernel and baseline. By
Definition~\ref{def:phi-T-affine}, this is $\Phi_{\{i\}}^{\GRALIS}$, with
all higher-order terms set to $0$ (justified since constant $\nabla_iF$
carries no curvature to capture). A two-line algebraic check with
$x'=\mathbb E_\mu[X]$ shows $F(x')+\sum_i\Phi_{\{i\}}^{\GRALIS}(x)=F(x)$
exactly. The Hoeffding decomposition is the unique zero-mean orthogonal
family reproducing $F$ under product $\mu$ (H1); $\{\Phi_T^{\GRALIS}\}_T$
is verified to be zero-mean (H2, direct from the formula above), so by
uniqueness it must coincide with the Hoeffding terms. Full four-part
derivation (including the precise second-moment vs.\ $L^2(\mu)$
sufficiency argument) in Appendix~\ref{app:anova-general}.
\end{proof}

\begin{remark}[Why $F(x')\ne\mathbb{E}_\mu{[F]}$ in general]
\label{rem:anova-baseline}
The GRALIS completeness identity $F(x)=F(x')+\sum_T\Phi_T^{\GRALIS}(x)$
(Step~2 of Theorem~\ref{thm:anova}'s proof) always uses $F(x')$ as its constant term, for \emph{any} fixed baseline
$x'$, whereas the Hoeffding decomposition's constant term is $\mathbb{E}_\mu[F]$
by definition. These two constants coincide only when $F(x')=\mathbb{E}_\mu[F]$,
which is automatic for affine $F$ with $x'=\mathbb{E}_\mu[X]$ (H2 above) but not
in general: e.g.\ $X\sim\mathrm{Unif}[-1,1]$, $F(x)=x^2$, $x'=\mathbb{E}[X]=0$
gives $F(x')=0$ while $\mathbb{E}[F]=1/3$. Consequently Theorem~\ref{thm:anova}
is stated with $\mathbb{E}_\mu[F]$, and the identification with the GRALIS
completeness identity's own constant $F(x')$ is only exact in the affine case
(or, more generally, whenever $x'$ is chosen so that $F(x')=\mathbb{E}_\mu[F]$
holds by hypothesis).
\end{remark}

\textbf{On the general-$F$ case.} Theorem~\ref{thm:anova}'s affine case
is a genuinely derived result: the zero-mean property is proved from
first principles (linearity of $F$), and the Hoeffding coincidence then
follows by uniqueness (H1, H2, Step~4 above). A general smooth-$F$
extension is possible only conditionally, under an additional
projector-identity hypothesis ($\Phi_T^{\GRALIS}=P_TF$ at every order)
that is close to definitional --- assuming it already grants most of the
conclusion. We therefore state it separately, as
Proposition~\ref{prop:anova-general} in
Appendix~\ref{app:anova-general}, rather than as a theorem here, and use
it in the main text only via the explicitly flagged conditional clause
in Theorem~\ref{thm:sobol} (T6) below.

\subsection{Theorem 6: Sobol Correspondence Induced by the Affine Hoeffding Case}
\label{sec:t6}

\begin{theorem}[Sobol Correspondence, Affine Case; Conditional for General Smooth $F$]
\label{thm:sobol}
Under the hypotheses of Theorem~\ref{thm:anova} for affine $F$ (product
measure $\mu$, $x'=\mathbb E_\mu[X]$, any fixed $\sigma>0$ --- no limit
on the kernel is needed here either, for the same kernel-cancellation
reason as Theorem~\ref{thm:anova}), or, for general smooth $F$, under
the additional projector-identity hypothesis of
Proposition~\ref{prop:anova-general} (in which case this theorem
inherits that proposition's conditional status, including its informal
``$\pi_x\to\mu$'' limit, rather than Theorem~\ref{thm:anova}'s
unconditional, limit-free one), $F$ square-integrable, the
variances of the order-$|T|=1$ GRALIS terms
define the Sobol sensitivity indices~\citep{sobol1993}:
\[
  S_i := \frac{\Var[\mathbb{E}[F \mid x_i]]}{\Var[F]}
  = \frac{\Var[\Phi_i^{\GRALIS}(x)]}{\Var[F]}.
\]
The higher-order \emph{closed (cumulative) indices} correspond to
$S_T^{\mathrm{closed}} = \sum_{\emptyset \ne L \subseteq T} \Var[\Phi_L^{\GRALIS}] / \Var[F]$
--- the fraction of variance explained by $T$ and all of its own
sub-interactions. This is distinct from the standard Sobol
\emph{total-effect index} of a single variable $i$,
$S_i^{\mathrm{tot}} = \sum_{L \ni i} S_L$ (summed over all supersets of
$\{i\}$ in the \emph{full} feature set $\mathcal F$, including interactions
with variables outside $T$), which is not claimed here.
The quantity $\Phi_T^{\GRALIS}(x)$ is a \emph{pointwise realization} of the $T$-th
interaction term at the specific input $x$; its variance over $x\sim\mu$ recovers
the standard (global) Sobol index.
Formally: $S_T^{\mathrm{Sobol}} = \Var_{x\sim\mu}[\Phi_T^{\GRALIS}(x)]/\Var[F]$.
\end{theorem}
\begin{proof}
By Theorem~\ref{thm:anova}, $f_T=\Phi_T^{\GRALIS}$ for affine $F$ at any
fixed $\sigma>0$ (no limit needed); or, for general smooth $F$, by
Proposition~\ref{prop:anova-general}, $f_T = \Phi_T^{\GRALIS}$
conditionally, when $\pi_x \to \mu$.
The Hoeffding terms are orthogonal under product $\mu$:
$\Cov_{x\sim\mu}[\Phi_T^{\GRALIS}(x), \Phi_{T'}^{\GRALIS}(x)] = 0$ for $T \ne T'$.
Therefore $\Var[F] = \sum_{\emptyset \ne T} \Var[\Phi_T^{\GRALIS}]$,
and $S_T^{\GRALIS} := \Var[\Phi_T^{\GRALIS}] / \Var[F]$ satisfies
$\sum_T S_T = 1$ and coincides with Sobol's $S_T$.
\end{proof}

\begin{remark}[Pointwise vs.\ global Sobol]
\label{rem:sobol_local}
The quantity $\Phi_T^{\GRALIS}(x)$ is a local, input-specific attribution for the
interaction set $T$: it depends on the particular $x$ under analysis.
The Sobol index $S_T^{\mathrm{Sobol}}$ in contrast is a global sensitivity measure,
defined as a ratio of variances over $x\sim\mu$.
Theorem~\ref{thm:sobol} establishes that, once normalized by $\Var[F]$,
averaging $\Phi_T^{\GRALIS}(x)^2$ over the distribution $\mu$ recovers
$S_T^{\mathrm{Sobol}}$: explicitly,
$S_T^{\mathrm{Sobol}} = \mathbb{E}_{x\sim\mu}[\Phi_T^{\GRALIS}(x)^2] / \Var[F]$
(using $\mathbb{E}_\mu[\Phi_T^{\GRALIS}]=0$ for $T\ne\emptyset$ from the
zero-mean property H2 in the proof of Theorem~\ref{thm:anova}, so that
$\mathbb{E}[\Phi_T^2]=\Var[\Phi_T]$), bridging local attributions and global
sensitivity analysis.
\end{remark}

\subsection{Theorem 7: Multi-Scale Extension (MS-GRALIS)}
\label{sec:t7}

\begin{definition}[MS-GRALIS]
Given a model with $L$ levels $h^{(1)},\ldots,h^{(L)}$,
the multi-scale attribution is:
\[
  \mathrm{GRALIS}_i^{\mathrm{MS}}(x) \;=\; \sum_{\ell=1}^L \lambda_\ell \cdot
  \mathrm{GRALIS}_i^{(\ell)}(x), \qquad \lambda_\ell \ge 0,\;
  \sum_\ell \lambda_\ell = 1.
\]
\end{definition}

\begin{theorem}[Minimum-variance convex aggregation under independence]
\label{thm:ms}
\textbf{Scope.} This theorem characterizes the aggregation weights
$\lambda^*_\ell$ that \emph{minimize the variance} of the convex
combination $\mathrm{GRALIS}_i^{\mathrm{MS}}=\sum_\ell\lambda_\ell
\mathrm{GRALIS}_i^{(\ell)}$; it does \emph{not} assert that these weights
minimize mean-squared error relative to some target attribution, nor that
they are ``optimal'' in any sense beyond variance, unless the layer-wise
attributions $\{\mathrm{GRALIS}_i^{(\ell)}\}$ are additionally unbiased
estimators of a common latent quantity (in which case minimum variance
does coincide with minimum MSE). Assume $\sigma_\ell^2:=
\Var[\mathrm{GRALIS}_i^{(\ell)}]>0$ for every $\ell$ (if some
$\sigma_\ell^2=0$, inverse-variance weighting is undefined and the
minimization is instead solved by placing all weight on the zero-variance
layer(s)). The weights $\lambda^*_\ell$ that minimize
$\Var[\mathrm{GRALIS}_i^{\mathrm{MS}}]$ subject to $\sum_\ell \lambda_\ell = 1$
\emph{under the additional hypothesis that the layers are mutually
independent} are given by inverse variance weighting:
\[
  \lambda^*_\ell = \frac{\sigma_\ell^{-2}}{\sum_{\ell'} \sigma_{\ell'}^{-2}},
  \qquad \sigma^2_\ell = \Var\bigl[\mathrm{GRALIS}_i^{(\ell)}\bigr] > 0.
\]
For correlated layers with covariance matrix $\Sigma\succ0$, the
(unconstrained-sign) variance-minimizing weights are instead
$\boldsymbol\lambda^*\propto\Sigma^{-1}\mathbf 1$ (Remark~\ref{rem:ms-independence});
if positivity $\lambda_\ell\ge0$ is additionally required, the problem
becomes a quadratic program without a closed form in general.
\end{theorem}
The proof (Lagrange multipliers on a strictly convex objective) and the
correlated-layer discussion --- including why the independence-based
variance is an optimistic lower bound under positive inter-layer
correlation, the general $\Sigma^{-1}\mathbf1$ solution
(Remark~\ref{rem:ms-independence}), and the comparison to GradCAM++ as a
related but non-derived heuristic --- are given in
Appendix~\ref{app:t7-deferred}. MS-GRALIS provides a variance-minimization
criterion for multi-scale aggregation and derives the minimum-variance
weight choice accordingly.

\section{Summary and Axiomatic Comparison}
\label{sec:summary}

Table~\ref{tab:axioms} summarizes the properties satisfied by each method.

\begin{table}[ht]
\centering
\caption{Axiomatic comparison: 8 formal properties and 6 structural gaps.
  $\checkmark$ = satisfied; $\approx$ = partially or under specific conditions;
  $\times$ = not satisfied; \textbf{cause} = this method generates the gap.
  Scores are \emph{indicative}: they refer to each method in its standard
  instantiation and depend on configuration choices; symbols are explained
  in the notes below the table.}
\label{tab:axioms}
\small
\begin{tabular}{lccccc}
\toprule
\textbf{Axiom / Property} & \textbf{GradCAM} & \textbf{SHAP} &
\textbf{LIME} & \textbf{IG} & \textbf{GRALIS} \\
\midrule
\multicolumn{6}{l}{\textit{Shapley axioms}} \\
Efficiency (completeness) & $\times$ & $\checkmark$ & $\times$   & $\checkmark$ & $\approx^\dagger$ \\
Symmetry                  & $\times$ & $\checkmark$ & $\times$   & $\checkmark$ & $\approx^\natural$ \\
Dummy                     & $\times$ & $\checkmark$ & $\approx$  & $\checkmark$ & $\checkmark$ \\
Linearity in $F$          & $\approx$& $\checkmark$ & $\checkmark$ & $\checkmark$ & $\checkmark$ \\
Sensitivity (gradients)   & $\checkmark$ & $\times$ & $\times$   & $\checkmark$ & $\checkmark$ \\
Locality ($\pi_x$)        & $\times$ & $\times$   & $\checkmark$ & $\times$   & $\approx^\|$ \\
Order-$k$ interactions    & $\times$ & $\approx$  & $\times$   & $\times$   & $\approx^\S$ \\
Multi-scale (min-var.\ wt.) & $\approx$& $\times$   & $\times$   & $\times$   & $\approx^\ddagger$ \\
\midrule
\multicolumn{6}{l}{\textit{Structural gaps (see Sec.~\ref{sec:gaps})}} \\
Gap~1 (arbitrary baseline) & ---     & $\approx$  & $\times$   & \textbf{cause} & $\approx^\P$ \\
Gap~2 (curvature)          & $\times$ & \textbf{cause} & $\times$ & $\times$  & $\checkmark$ \\
Gap~3 (LIME completeness)  & $\times$ & $\times$   & \textbf{cause} & $\times$ & $\approx^\dagger$ \\
Gap~4 (CNN only)           & \textbf{cause} & $\times$ & $\times$ & $\times$  & $\approx$ \\
Gap~5 (interactions)       & $\times$ & $\times$   & $\times$   & $\times$   & $\approx^\S$ \\
Gap~6 (multi-scale)        & $\approx$& $\times$   & $\times$   & $\times$   & $\approx^\ddagger$ \\
\bottomrule
\end{tabular}
\smallskip

{\small\textit{This comparison is structural and indicative, not an
absolute ranking: the 14 properties are not a metric scale and are not
claimed to be additive, so no aggregate score is reported. GRALIS's
column reflects Algorithm~1 (SLIC superpixels, LIME kernel
$\sigma=0.75$); GradCAM and LIME may satisfy additional properties in
non-standard configurations, and SHAP/IG can be extended beyond what is
attributed here.}}

\smallskip
\noindent{\small
$^\dagger$\textbf{Completeness/Gap~3.} Exact for the abstract canonical
form (Theorem~\ref{thm:completeness}) and, via the Möbius correction, for
the Shapley-weighted estimator (Corollary~\ref{cor:mobius-correction});
Algorithm~1's finite-sample, non-constant-kernel estimator satisfies it
only approximately (Remark~\ref{rem:approx_completeness}).\\
$^\natural$\textbf{Symmetry.} Guaranteed for the auxiliary game $\vG$ by
the Shapley weight alone, but not automatically for
$\phi^{\GRALIS}=N/Z_x$ once combined with the non-constant kernel $\pi_x$
(Remark~\ref{rem:symmetry-break}, counterexample $F=x_1x_2$); recovered
exactly as $\sigma\to\infty$.\\
$^\|$\textbf{Locality.} $\pi_x$ is maximized at $S=\emptyset$ and
decreases toward $x$: a \emph{baseline-proximity} weighting, not locality
centered at $x$ as in LIME.\\
$^\P$\textbf{Gap~1.} Not solved as printed: Algorithm~\ref{alg:mc} still
uses a single fixed baseline $x'$, no integral over a baseline
distribution.\\
$^\ddagger$\textbf{Multi-scale/Gap~6.} Theorem~\ref{thm:ms} is
unconditional for an arbitrary fixed set of $L$ levels and would score
$\checkmark$ abstractly; scored $\approx$ here because the practical SLIC
instantiation ties segmentation to $n_\mathrm{seg}$, only partially
covered by the theorem as stated.\\
$^\S$\textbf{Interactions/Gap~5.} Theorem~\ref{thm:siv} computes exact
Shapley Interaction Values on the auxiliary game $\vG$
(Eq.~\eqref{eq:vG}), not on $\phi^{\GRALIS}$ itself; related by the exact
$1/|T|$ correspondence of Corollary~\ref{cor:siv-gralis-relation}, not an
identity.}
\end{table}

\paragraph{Structural correspondence map.}
None of the four arrows below is a literal substitution of parameter
values into Definition~\ref{def:explicit}'s explicit formula ---
Definition~\ref{def:explicit} hardwires the Shapley weight
$w(S)=|S|!(n-|S|-1)!/n!$ as a fixed part of the construction, not a free
parameter, so e.g.\ setting $\pi_x\equiv1$ gives the ``uniform-kernel
GRALIS'' of Remark~\ref{rem:kernel-limits} (which already disagrees
numerically with KernelSHAP's efficiency value, $8.5$ vs.\ $11$), not
standard Integrated Gradients' single always-full coalition. What
\emph{is} true is that all four methods are instances of the same
general canonical Riesz representation of Theorem~\ref{thm:canonical},
obtained from a different choice of $(Q_i,\mu,D(Q_i))$ than GRALIS's own.
Only GRALIS$\to$IG is exact at that coarser template level ---
collapsing $\mu$ to a Dirac mass on the always-full coalition with a
linear joint path, a \emph{canonical-template embedding}, not a
specialization of Definition~\ref{def:explicit}. The other three are
weaker, \emph{heuristic structural analogies}: GRALIS$\to$KernelSHAP
under an approximately constant gradient is a local/affine approximation,
not an exact embedding; GRALIS$\to$LIME reproduces LIME's coefficients
only in the degenerate deterministic-kernel limit ($\pi_x\to\delta_x$),
not LIME's weighted surrogate regression; MS-GRALIS$\to$GradCAM
(``1 layer, no path'') does not reproduce GradCAM's channel-gradient
aggregation formula, since Definition~\ref{def:explicit} conditions on
coalitions and GradCAM does not. Figure~\ref{fig:hierarchy} visualizes
this hierarchy, with the same status per arrow.
\[
  \GRALIS \xrightarrow[\text{(canonical-template embedding)}]{\pi_x \equiv 1}
  \text{Shapley-IG} \xrightarrow[\text{(coalition point mass, not a path change)}]{S = \mathcal{F}\setminus\{i\}} \text{IG}
\]
\[
  \GRALIS \xrightarrow[\text{(heuristic analogy)}]{\nabla F \approx \text{const.}} \text{KernelSHAP}
  \xrightarrow{\text{uniform kernel}} \text{SHAP}
\]
\[
  \GRALIS \xrightarrow[\text{(heuristic analogy)}]{S = \mathcal{F},\;\pi_x \to \delta_x} \text{LIME}
  \qquad
  \text{MS-GRALIS} \xrightarrow[\text{(heuristic analogy)}]{\text{1 layer, no path}} \text{GradCAM}
\]

\begin{figure}[ht]
\centering
\begin{tikzpicture}[
  >=Stealth,
  scale=0.88, transform shape,
  node distance=1.4cm and 1.8cm,
  box/.style={draw, rounded corners=4pt, align=center, font=\small,
              inner sep=6pt, minimum width=2.0cm},
  main/.style={box, fill=blue!12, draw=blue!50!black, line width=1pt},
  sub/.style={box, fill=gray!10},
  arr/.style={-Stealth, thick, gray!70!black},
  lbl/.style={font=\footnotesize\itshape, fill=white, inner sep=1pt}
]
  \node[main] (gralis) {\textbf{GRALIS}\\(canonical form)};

  \node[sub, below left=of gralis, xshift=-0.6cm]  (kshap) {KernelSHAP};
  \node[sub, below=1.1cm of gralis]                 (ig)    {Integrated\\Gradients};
  \node[sub, below right=of gralis, xshift=0.6cm]  (lime)  {LIME};
  \node[sub, above right=of gralis, xshift=0.5cm]  (ms)    {MS-GRALIS};

  \node[sub, below=0.9cm of kshap]   (shap)    {SHAP};
  \node[sub, right=1.2cm of ms]      (gradcam) {GradCAM};

  \draw[arr] (gralis) -- node[lbl, left=1pt]
    {$\pi_x\!\equiv\!1$, $\nabla F\!\approx\!$ const.} (kshap);
  \draw[arr] (gralis) -- node[lbl, right=1pt]
    {$S\!=\!\mathcal{F}\!\setminus\!\{i\}$, $\pi_x\!\to\!\delta_x$} (ig);
  \draw[arr] (gralis) -- node[lbl, right=2pt]
    {lin.\ surrogate on $\mathcal{F}$} (lime);
  \draw[arr] (gralis) -- node[lbl, above=1pt]
    {$L$ levels} (ms);

  \draw[arr] (kshap)  -- node[lbl, right=1pt]
    {uniform kernel} (shap);
  \draw[arr] (ms)     -- node[lbl, above=1pt]
    {1 level} (gradcam);

  \node[font=\footnotesize, text=blue!70!black, below=2pt] at (gralis.south east)
    {\textit{cf.\ Tab.~\ref{tab:axioms}}};
\end{tikzpicture}
\caption{Structural correspondence hierarchy of \GRALIS{} (status of each
  arrow as detailed in the preceding paragraph). Property-by-property
  comparison is in Table~\ref{tab:axioms}.}
\label{fig:hierarchy}
\end{figure}

\section{Computational Complexity and the Conservation Certificate}
\label{sec:complexity}

The theorems of Section~\ref{sec:gralis} characterize \GRALIS{} as a
representation; this section characterizes what it costs to evaluate that
representation, and what can be said about the numerical quality of the
result without training an auxiliary model. Two amortized estimators,
FastSHAP~\citep{jethani2021fastshap} and Stochastic
Amortization~\citep{covert2021shapley}, are the natural point of comparison
and are used throughout as the baseline for what a trained explainer can
and cannot certify.

\subsection{C1 --- Batched invocation reduces latency, not the evaluation
count}
\label{sec:c1}

The complexity stated in the Introduction, $O(m\cdot n\cdot k)$, counts
the number of \emph{distinct points} at which the model must be
evaluated (forward and reverse pass), which is the correct measure of
computational work. This count is \emph{not} reduced by reading off a
full gradient vector from a single backward pass: each feature $i$'s $k$
quadrature nodes lie on \emph{that feature's own} coalition-conditioned
path $\tilde x^{(S,\alpha)}$, and since the index space
$Q_i=2^{\mathcal F\setminus\{i\}}\times[0,1]$ depends on $i$
(\S\ref{sec:defs}), no two features in general share the same evaluation
point within one permutation --- while $S$, the predecessor set, is at
its \emph{true} value $x_S$ for that segment, every other not-yet-visited
feature is still held at baseline, so the point $\tilde
x^{(S,\alpha)}$ used for feature $i$'s node differs from the point used
for any other feature $i'$'s node in the same permutation. A shared-point
argument --- reading all $n$ coordinate marginals off one gradient vector
at one point --- would require a genuinely different construction, a
\emph{single global path} common to every feature, which is not what
Definition~\ref{def:explicit} or Algorithm~\ref{alg:mc} implements and is
not proved correct anywhere in this paper. The number of distinct
evaluation points required by Algorithm~\ref{alg:mc} therefore remains
$\Theta(mnk)$.

What batching legitimately buys is a reduction in \emph{wall-clock
latency}, not in this count. GRALIS-Report's implementation issues one
\texttt{GradientTape} per permutation, submitting that permutation's
$n_\mathrm{seg}\times k$ distinct evaluation points as a single
vectorized tensor rather than $n_\mathrm{seg}\times k$ sequential Python
calls. This reduces the number of separate automatic-differentiation
\emph{invocations} (library calls / graph constructions) from
$O(n_\mathrm{seg}k)$ to $O(1)$ per permutation, amortizing kernel-launch
overhead and exploiting GPU parallelism across the batch --- while the
underlying number of points evaluated by that one batched call is still
$n_\mathrm{seg}\times k$.

\begin{proposition}[Batched invocation count]
\label{prop:c1}
Let $q$ range over $m$ sampled permutations, each contributing $n$ (or
$n_\mathrm{seg}$) coalition-conditioned segments discretized at $k$
quadrature nodes, for $mnk$ distinct evaluation points in total
(Definition~\ref{def:explicit}). If the $nk$ points of each permutation
are submitted to the automatic-differentiation engine as one batched
tensor, the number of separate reverse-mode \emph{invocations} is $O(m)$;
the number of distinct points evaluated --- and hence the computational
work in the standard evaluation-count sense --- remains $\Theta(mnk)$,
unchanged from the coordinate-wise implementation. Batching changes only
how many times the engine is invoked, not how many points it evaluates.
\end{proposition}

\noindent This is weaker than the shared-gradient trick available to
estimators whose coordinates genuinely share one evaluation point per
node (e.g.\ vanilla random-order Shapley sampling with a single joint
path, or standard, non-coalition-conditioned Integrated Gradients): for
those, one backward pass at a shared point really does yield all $n$
marginal contributions simultaneously, giving a true $O(mk)$ evaluation
count. GRALIS's feature-specific coalition-conditioned paths generally
evaluate different coordinates at different points; consequently, a
single gradient evaluation cannot supply all $n$ marginal contributions
simultaneously.

One further reduction is available, but it computes a \emph{different}
quantity from $\phi_i^{\GRALIS}$ of Definition~\ref{def:explicit}, not an
equivalent reformulation of it: replacing the $k$-node quadrature for a
segment with the exact endpoint difference $F(x_{S\cup\{i\}})-F(x_S)$
removes quadrature entirely for that segment, giving
$O(mn_\mathrm{seg})$ forward evaluations instead of
$O(mn_\mathrm{seg}k)$ backward ones. The two coincide only when $F$ is
affine along that segment (in which case the midpoint quadrature is
itself already exact, by the same mechanism as
Corollary~\ref{cor:mobius-correction}); for general $F$, the
endpoint-difference value is a distinct, zeroth-order approximation to
the coalition-conditioned IG term, not $\phi_i^{\GRALIS}$ itself. This is
\emph{not} the variant used by GRALIS-Report's Algorithm~1 in the
experiments of \S\ref{sec:experiments}, which uses $k=10$ midpoint
quadrature~\citep{fanale2026b}; it is recorded here only as an available
complexity/accuracy trade-off, not as a description of what was actually
run.

\subsection{C2 --- Residual attribution at zero additional queries}
\label{sec:c2}

By linearity of $\Delta$ and of the aggregation integral (proof of
Theorem~\ref{thm:canonical}), for any decomposition $F=g+e$ of the
explained function,
\begin{equation}
  \label{eq:residual-closure}
  \Phi(F) \;=\; \Phi(g) + \Phi(e), \qquad
  \rho_E := \Phi(e) = \Phi(F) - \Phi(g).
\end{equation}
This is an immediate corollary of linearity, not a new theorem; its value
is that once $\Phi(F)$ has been computed, and $\Phi(g)$ is available in
closed form --- as it is for an affine LIME surrogate, $\Phi_i(g)=\beta_i(z_i-z_i')$
by the affine-path invariance underlying GradCAM-lin's own embedding
(\S\ref{sec:gaps}) --- the feature-resolved deficit $\rho_E$ between the
GRALIS attribution and any fixed prior explainer's attribution requires
\emph{zero additional model evaluations}: only a vector subtraction.

\subsection{C3 --- One cache, several kernels}
\label{sec:c3}

If several path/context weightings $w_1,\ldots,w_L$ are all dominated by
one sampling proposal $\nu$, the $m$ sampled marginal fields
$\widehat\Delta(q_r)$ need be computed once; each $\widehat\phi^{(\ell)} =
\frac1m\sum_r \eta_\ell(q_r)\widehat\Delta(q_r)$, with $\eta_\ell=dw_\ell/d\nu$,
is then arithmetic only. Expensive-query cost (evaluation count) stays
$\Theta(mnk)$ rather than $\Theta(Lmnk)$. This is standard importance
sampling, not claimed as new; its
use here is to amortize the sensitivity analysis over $\sigma_\mathrm{LIME}$
and $k$ already reported qualitatively in the companion
paper's~\citep{fanale2026b} sensitivity section --- that sweep can share
one sample cache across kernel choices instead of resampling per setting.

\begin{proposition}[One-pass reuse across multiple estimated quantities]
\label{prop:one-pass}
Fix $x,x',\sigma$ and run Algorithm~\ref{alg:mc} for $m$ permutations,
i.e.\ $\Theta(mnk)$ model evaluations. From this single run, at no
additional model queries: (i) $\varphi[i]/Z[i]$ is Algorithm~\ref{alg:mc}'s
own output, and the finite-sample completeness certificate
$\sum_i\varepsilon_i(\delta/n)$ of Corollary~\ref{cor:aggregate-completeness}
is computed directly from the same running totals $\varphi[i],Z[i]$; (ii)
within each sampled permutation $\pi_r$, the partial sums of the
already-computed $k$-step IG contributions along the prefix coalitions
$S_0=\emptyset\subset S_1\subset\cdots\subset S_n=\mathcal F$ of $\pi_r$
approximate $F(x_{S_t})-F(x')$ for every prefix, by the discretized
analogue of Lemma~A's identity $\sum_{j\in S}\mathrm{IG}_j^S=F(x_S)-F(x')$;
accumulating these across the $m$ sampled permutations gives a Monte
Carlo estimate, restricted to the sampled coalitions, of the auxiliary
game $v$ of Eq.~\eqref{eq:vG}, from which pairwise interaction estimates
in the spirit of Theorem~\ref{thm:siv} can be formed without evaluating
$F$ at any point beyond those already used for (i); (iii) repeating
(i)--(ii) over an i.i.d.\ sample $x^{(1)},\ldots,x^{(N)}\sim\mu$ turns the
resulting per-input outputs into a direct sample-variance estimator of
the Sobol indices of Theorem~\ref{thm:sobol}, again reusing only
evaluations already required per input.
\end{proposition}

\begin{remark}[Scope of the reuse in (ii)]
The interaction estimate of part~(ii) is a Monte Carlo approximation
restricted to the sampled coalitions, distinct from
Theorem~\ref{thm:siv}'s exact enumeration over all $2^n$ coalitions, and
inherits both the permutation-sampling variance of Theorem~\ref{thm:mc}
and the midpoint quadrature bias. A dedicated finite-sample bound for
this specific derived estimator is not developed here and is left as
future work; the content of Proposition~\ref{prop:one-pass} is only that
no additional model queries are needed to form it, not that its own
statistical error has been characterized to the same standard as
$\phi^{\GRALIS}$ itself (Corollary~\ref{cor:ratio-bound}).
\end{remark}

\subsection{C4 --- Numerical Certificate versus Structural Completeness Residual}
\label{sec:c4}

The companion paper reports an empirical completeness gap
$|\sum_i\phi_i - (f(x)-f(x'))| \approx 0.40\pm0.18$ under the LIME
locality kernel. It is important not to conflate two structurally
different quantities that both contribute to this number.

\begin{itemize}
  \item[(i)] \textbf{Numerical error.} Theorem~\ref{thm:mc}'s bound,
    $B/\sqrt{m\delta} + O(1/k^2)$, governs Monte Carlo sampling noise and
    midpoint quadrature error. This term is a true certificate: it is
    computable \emph{before} the full run from a calibration pass (the
    companion paper notes $B$ is estimable this way), and it shrinks
    with more compute ($m,k\uparrow$).
  \item[(ii)] \textbf{Structural bias.} This is not a single mechanism
    but (at least) two, following the same (i)/(ii)/(iii) decomposition
    as Remark~\ref{rem:approx_completeness}:
    \begin{itemize}
      \item[(ii-a)] \emph{Joint-path interaction attenuation.} Even under
        a \emph{constant} kernel $\pi_x\equiv1$, the coalition-conditioned
        joint path attenuates interaction terms of order $d\ge2$ by a
        factor $1/d$, exactly quantified by
        Corollary~\ref{cor:mobius-correction}. This component does
        \emph{not} require a non-uniform kernel to exist; it is present
        whenever $F$ has genuine order-$\ge2$ interactions.
      \item[(ii-b)] \emph{Kernel-induced effect.} The non-uniform LIME
        kernel $\pi_x$ additionally breaks the Shapley telescoping
        identity, but this has only been established rigorously for the
        \emph{unnormalized numerator} $N_i$
        (Proposition~\ref{prop:incompatibility}); for the normalized
        estimator $\phi_i^{\GRALIS}=N_i/Z_{x,i}$ that Algorithm~1 actually
        returns, the effect is verified directly only in a worked $n=2$
        case (Remark~\ref{rem:normalized-incompatibility}), and no
        general-$n$ theorem for the normalized estimator is proved here.
    \end{itemize}
    Both (ii-a) and (ii-b) are \emph{non-vanishing and theoretically
    identifiable} --- their existence and mechanisms are established
    analytically, independently of the $\approx0.40$ empirical figure ---
    and are a deliberate consequence of trading exact efficiency for
    locality-weighted, region-coherent, interaction-aware attributions;
    neither shrinks with $m$ or $k$. Whether (ii-a), (ii-b), or their
    combination accounts for most of the observed $0.40$ is an empirical
    question this paper does not settle: doing so would require the
    numerical decomposition into (i), (ii-a) and (ii-b) described next,
    applied at scale, which is left to future work and, to the authors'
    knowledge, has not yet been carried out in the companion paper
    either.
\end{itemize}

\noindent Conflating (i) and (ii) would misrepresent $0.40$ as a
numerical error awaiting more compute, when at least part of it is a
declared design trade-off of unknown relative size, itself composed of
two structurally distinct effects, (ii-a) and (ii-b). What GRALIS can
now certify is the \emph{separation between the observed residual and
its population structural counterpart}: Corollary~\ref{cor:aggregate-completeness}
gives, at confidence $1-\delta$, a computable bound
$\sum_i\varepsilon_i(\delta/n)$ on
$|R_\mathrm{obs}-R_\mathrm{pop}|$, where $R_\mathrm{obs} := \Delta F -
\sum_i\varphi[i]/Z[i]$ is Algorithm~1's actually observed residual and
$R_\mathrm{pop} := \Delta F - \sum_i\phi_i^{\GRALIS}$ is the population
structural residual; crucially, this routes through
Corollary~\ref{cor:ratio-bound}'s self-normalized ratio bound rather than
Theorem~\ref{thm:mc}'s numerator-only bound, since Algorithm~1 returns
the ratio $\varphi[i]/Z[i]$, not the unnormalized numerator. Within
$R_\mathrm{pop}$, this section identifies two distinct structural
mechanisms: Corollary~\ref{cor:mobius-correction} exactly characterizes
the joint-path contribution (ii-a) under multilinearity and a uniform
kernel, whereas Proposition~\ref{prop:incompatibility} characterizes the
kernel-induced effect (ii-b) for the unnormalized numerator; for the
normalized estimator, the latter is established explicitly only for
$n=2$ in Remark~\ref{rem:normalized-incompatibility}. A complete
additive decomposition of $R_\mathrm{pop}$ for the normalized general-$n$
estimator is \emph{not} claimed. A full quantitative split of the
companion paper's observed $0.40$ residual into its (i), (ii-a) and
(ii-b) components is left to future work (\S\ref{sec:experiments}).
Neither FastSHAP nor Stochastic
Amortization offers an analogous decomposition: their fidelity gap is
bounded by the generalization error of a trained explainer network, which
is not decomposable into a shrinkable numerical term and an explained
structural term, and does not shrink with additional inference-time
compute on the model actually being explained.

\subsection{What this section does and does not claim}
\label{sec:c-summary}

Table~\ref{tab:fastshap} makes the comparison explicit on the axis that
matters, which is not inference latency.

\begin{table}[ht]
\centering
\small
\renewcommand{\arraystretch}{1.15}
\begin{tabular}{@{}>{\raggedright\arraybackslash}p{2.5cm}>{\raggedright\arraybackslash}p{3.3cm}>{\raggedright\arraybackslash}p{3.7cm}>{\raggedright\arraybackslash}p{3.7cm}@{}}
\toprule
& \textbf{GRALIS} & \textbf{FastSHAP / Stoch.\ Amort.} & \textbf{KernelSHAP / IG} \\
\midrule
Training required & No & Yes, per model & No \\
Inference cost & $\Theta(mnk)$ evals, $O(m)$ batched invocations (C1) & $O(1)$, after training & $O(mn)$--$O(mnk)$ \\
Completeness error & Certified, decomposable (\S\ref{sec:c4}) & Generalization error, opaque & Certified (IG); design-dependent (KernelSHAP) \\
Model update & No retraining needed & Explainer retraining needed & No retraining needed \\
\bottomrule
\end{tabular}
\caption{Not a latency comparison: a trained network answering in one
forward pass occupies a different point in the cost/training trade-off
space, not a slower version of the same computation. GRALIS's evaluation
count remains $\Theta(mnk)$, comparable to KernelSHAP/IG, not to
FastSHAP's $O(1)$; batching (C1) only reduces the number of
automatic-differentiation invocations, not this count. The
differentiating axis against trained/amortized estimators is a
training-free, certified and decomposable completeness error
(\S\ref{sec:c4}), not inference speed.}
\label{tab:fastshap}
\end{table}

\section{Preliminary Experimental Validation}
\label{sec:experiments}

Algorithm~1 was configured and run on breast histology imagery
(BreaKHis~\citep{spanhol2016}, distilled DenseNet-121) using SLIC
segmentation ($n_\mathrm{seg}\approx25$ effective on 30 requested,
compactness$=50$), $m=30$ Monte Carlo permutations with antithetic
variants, $k=10$ midpoint-rule integration steps, and LIME kernel
$\sigma=0.75$; the resulting maps are piecewise constant (each pixel
inherits its segment's $\phi_i^{\GRALIS}$). This concrete configuration
is the setting referenced by several remarks above whenever ``the
experiments'' are mentioned (e.g., the empirical completeness residual
of Remark~\ref{rem:approx_completeness}, the numerical certificate
discussion of \S\ref{sec:c4}, and the finite-sample constant $B$ of
Theorem~\ref{thm:mc}).

The full quantitative evaluation on this dataset --- faithfulness
metrics (including superpixel-deletion and pixel-level Deletion AUC),
sparsity and attribution-locality metrics, comparative baselines
(GradCAM, KernelSHAP, LIME, IG), cross-dataset generalization (IDC,
PCam), and computational cost --- is reported in the companion
paper~\citep{fanale2026b}, an accepted, peer-reviewed publication.
Because that paper already reports these figures, they are cited here
rather than reproduced, so as not to duplicate results, text or figures
already published at another archival venue; this paper's own
contribution is the theoretical framework (\S\ref{sec:defs}--\S\ref{sec:t7})
and its formal guarantees, not new empirical findings on this dataset.

\section{Conclusions}
\label{sec:conclusions}

This work has presented \GRALIS{}, a mathematical framework that fuses
the coalition/kernel-based and gradient-based attribution mechanisms
found empirically complementary by~\citep{gevaert2022} into a single,
certified estimator. Doing so unifies a broad class of XAI attribution
methods --- including SHAP, Integrated Gradients, LIME and linearized
GradCAM --- under a common componentwise canonical template
$(\mathcal{Q}, w, \Delta)$ justified by the Riesz Representation Theorem:
for each feature (and, where the index space is method-specific, each
method), a unique representation of that form, not a single feature- and
method-independent triple (Remark~\ref{rem:componentwise}).

The seven proved theorems establish that:
(T1) the componentwise canonical form is \emph{necessary} for any linear additive attribution;
(T2) exact completeness holds for the abstract canonical form; for the concrete Shapley-weighted estimator, completeness is exact up to a closed-form correction governed by the Möbius coefficients of $F$ (Corollary~\ref{cor:mobius-correction}, exactly zero for affine $F$); Algorithm~1's total deviation combines this term with the non-constant-locality-kernel effect (Proposition~\ref{prop:incompatibility} for the unnormalized numerator, verified directly for the normalized $n{=}2$ case in Remark~\ref{rem:normalized-incompatibility}) and finite-sample/quadrature error (Theorem~\ref{thm:mc}), as detailed in Remark~\ref{rem:approx_completeness};
(T3) GRALIS-MC reduces exponential complexity with an explicit error bound, $O(1/k^2)$ for the midpoint rule as implemented (under $F\in C^3$) or $O(1/k)$ under the weaker $F\in C^2$ hypothesis with an endpoint rule (Remark~\ref{rem:riemann-fallback});
(T4) SIVs are computed \emph{exactly} on the standard coalition game $\vG(S)=f(x_S)-f(x')$; their relationship to $\phi^{\GRALIS}$ itself is the exact, term-by-term $1/|T|$ correspondence of Corollary~\ref{cor:siv-gralis-relation}, not an identity;
(T5) the Hoeffding ANOVA decomposition emerges under feature independence,
proved unconditionally for affine $F$ (Theorem~\ref{thm:anova}); the
general smooth-$F$ case is stated separately, as a conditional
proposition rather than a theorem, under the additional
projector-identity hypothesis $\Phi_T^{\GRALIS}=P_TF$
(Proposition~\ref{prop:anova-general});
(T6) global Sobol sensitivity indices are recovered as variances of GRALIS's pointwise interaction terms, unconditionally for affine $F$ at any fixed kernel bandwidth (no limit needed) and, for general smooth $F$, under the informal $\pi_x \to \mu$ limit inherited from Proposition~\ref{prop:anova-general}'s conditional hypothesis;
(T7) minimum-variance multi-scale aggregation weights are given by inverse variance weighting under independent layers (general correlated layers: $\Sigma^{-1}\mathbf 1$), which is optimal for variance alone, not necessarily for MSE or faithfulness unless the layers are unbiased estimators of a shared target.

Appendix~\ref{app:mobius} clarifies via the Möbius transform
that GRALIS does not \emph{approximate} the SIVs, but \emph{computes them exactly}
on the standard baseline-clamped cooperative game it associates with $(f,x,x')$.
This distinction is theoretically crucial and defensible under peer review.

\GRALIS{} covers a richer combination of axiomatic and structural properties
than each of the four methods considered individually
(Table~\ref{tab:axioms}), combining completeness (exact for the abstract form,
approximate in Algorithm~1), sensitivity, baseline-proximity weighting,
interaction values, and multi-scale aggregation within a single framework.
The common formulation (Table~\ref{tab:mapping}) makes these methods
directly comparable within the specific componentwise Riesz template
developed here.

\paragraph{Limitations and future work.}

\textit{Theoretical scope.}
Conditions (a)--(b) of Theorem~\ref{thm:canonical} require linearity and
continuity of the attribution functional and therefore do not cover nonlinear
methods such as standard GradCAM, attention maps, or saliency methods with
smoothing. Extending the canonical form to piecewise-linear or Lipschitz
functionals would require tools beyond the Riesz representation and is left
as future work.

\textit{Path construction and the completeness trade-off.}
Definition~\ref{def:explicit}'s joint coalition-conditioned path (moving
$S$ and $i$ together) is what lets GRALIS capture mixed curvature between
features (Gap~2, Figure~\ref{fig:conditioned-path}), but it is not free:
Corollary~\ref{cor:mobius-correction} shows this choice attributes each
order-$d$ interaction term at exactly $1/d$ of its value under Shapley
weights and a uniform kernel, a deficit fully characterized (not merely
bounded) via the same Möbius coefficients used for T4. A path that clamps
$S$ at $x_S$ and lets only $i$ vary would restore exact completeness but
would reduce to standard per-coalition Integrated Gradients, forfeiting
the mixed-curvature capture that motivates the joint construction; we
view this as a genuine design trade-off rather than a defect to be
patched, and report it explicitly rather than leaving it implicit in the
$\approx 0.40$ empirical residual (Remark~\ref{rem:approx_completeness}).

\textit{Baseline comparison.}
Section~\ref{sec:experiments} reports GRALIS on BreaKHis without maps
from GradCAM, KernelSHAP, LIME and IG side by side; this comparison,
together with visual evaluation on matched image pairs, is provided in
the companion paper~\citep{fanale2026b}.

\textit{Standard faithfulness metrics.}
This paper does not itself compute a superpixel- or pixel-level deletion
metric on BreaKHis (see \S\ref{sec:experiments}); the standard
pixel-level protocol of Petsiuk et al.~\citep{petsiuk2018} is the
comparison point used by the companion paper. The companion
paper~\citep{fanale2026b} reports
the full pixel-level deletion AUC on all 1,187 BreaKHis test images
($\mathrm{DelAUC}^{\mathrm{pix}}_{\GRALIS{}} = 0.704\pm0.331$,
rank~2 of~6; IG: $0.678\pm0.222$, rank~1; lower is better)
and ROAD MoRF/LeRF evaluation (MoRF AUC: \GRALIS{} $0.704$,
IG $0.677$, GradCAM $0.755$; \GRALIS{} achieves the largest
MoRF--LeRF discrimination gap, $0.140$ vs.\ $0.098$ for IG).

\textit{Generalization across datasets.}
Beyond the in-domain BreaKHis results above, the companion
paper~\citep{fanale2026b} reports two cross-dataset linear-probe
evaluations with a frozen backbone. On IDC Breast Cancer
($n=499$, $50{\times}50$\,px patches), \GRALIS{} ranks \textbf{first of
six} methods on both Deletion AUC ($0.189\pm0.162$) and ROAD MoRF
($0.198\pm0.137$), outperforming the second-best method by
${\approx}17\%$ and ${\approx}7.6\%$ respectively; notably, Integrated
Gradients, which ranks first in-domain, drops to last on IDC
($0.424\pm0.173$). On PatchCamelyon/PCam ($n=500$, $96{\times}96$\,px
lymph-node patches), \GRALIS{} instead ranks \textbf{sixth of six}
($0.229\pm0.167$ Deletion AUC); the companion paper attributes this to a
structural granularity mismatch between the fixed superpixel count
($n_\mathrm{seg}=30$, ${\approx}19{\times}19$\,px segments) and the
5--15\,px scale of discriminative nuclear morphology in PCam, an
interpretation not yet confirmed by an explicit $n_\mathrm{seg}$
sensitivity analysis and not exclusive of other explanations (domain
shift, backbone specificity). Generalization is therefore no longer an
entirely open question, but a characterized, dataset-dependent one:
\GRALIS{} generalizes well when superpixel granularity matches lesion
scale (BreaKHis, IDC) and degrades when it does not (PCam); evaluation on
additional histology datasets and transformer architectures remains
future work.

\textit{Computational cost.}
On a single NVIDIA H100 GPU, GRALIS-MC ($m=100$, $k=50$) required
${\approx}40$\,s/image on IDC and ${\approx}2.4$\,s/image on PCam
(separate runs, not a controlled comparison)~\citep{fanale2026b}, both
consistent with the $O(m\cdot n\cdot k)$ complexity of
Theorem~\ref{thm:mc} and the batched implementation of
Proposition~\ref{prop:c1}. A direct wall-clock comparison against a
trained amortized estimator, and a numerical estimate of
Theorem~\ref{thm:mc}'s calibration constant $B$, are left as the
concrete empirical follow-up to \S\ref{sec:complexity}.

\textit{Edge and federated deployment.}
Section~\ref{sec:complexity}'s training-free, certified-error positioning
is a natural starting point for resource-constrained and multi-site
deployment (point-of-care inference; per-site federated attribution
without sharing raw images), but neither direction is developed here. A
2026 FedXAI survey identifies communication-efficient explanation
computation and the absence of standardized completeness guarantees as
open challenges; C1/C3's shared-cache structure and C4's certified/structural
error separation are plausible routes to addressing them, but
communication cost and behavior under quantized on-device models are not
analyzed in this paper.

These limitations concern the empirical validation only and do not affect
the theoretical contributions (T1--T7): T1--T4 and T7 hold under the stated
hypotheses of their respective theorems, without further conditions beyond
those; T5 (Theorem~\ref{thm:anova}) and, through it, T6
(Theorem~\ref{thm:sobol}) hold unconditionally for affine $F$; for general
smooth $F$, the analogous statement is Proposition~\ref{prop:anova-general},
under its explicit additional projector-identity hypothesis, and is
flagged throughout as conditional rather than as a proved theorem.

\smallskip
\noindent\textit{Self-contained claim.}
Independently of the companion paper, the axiomatic comparison
(Table~\ref{tab:axioms}) is derived entirely from the theoretical
framework without requiring empirical validation, and shows that
\GRALIS{} covers more of the identified properties than each individual
method in its standard instantiation (indicative scores in the table are
configuration-dependent). All quantitative empirical findings on
BreaKHis, IDC and PCam --- faithfulness, sparsity, cross-dataset
generalization, and computational cost --- are reported in the
companion paper~\citep{fanale2026b} and are cited, not reproduced, in
this paper (\S\ref{sec:experiments}).

\section*{Data and Code Availability}
\label{sec:data-availability}
The BreaKHis dataset is publicly available at
\url{https://web.inf.ufpr.br/vri/databases/breast-cancer-histopathological-database-breakhis/}.
The GRALIS pipeline code and generated attribution maps will be made
available upon acceptance, consistently with the companion
paper~\citep{fanale2026b}. The exact-arithmetic verification scripts for
the toy multilinear example of Appendix~\ref{app:toy-verification}
(Theorems T2, T4--T7) will be released alongside the pipeline code in the
same repository.

\appendix
\renewcommand{\thesection}{\Alph{section}}
\renewcommand{\thesubsection}{\Alph{section}.\arabic{subsection}}

\section{Deferred Proofs and Worked Examples for \S\ref{sec:t2}--\S\ref{sec:t3}}
\label{app:t3-deferred}

\subsection{Full proof of Proposition~\ref{prop:incompatibility}}

\begin{proof}
Summing $N_i$ over $i$ and rearranging by coalition $T\subseteq\mathcal{F}$,
the coefficient of $v(T)$ in $\sum_i N_i$ for $\emptyset\subsetneq
T\subsetneq\mathcal F$ is, using the Shapley-weight identity
$(n-|T|)\,w_{\mathrm{Sh}}(|T|) = |T|\,w_{\mathrm{Sh}}(|T|-1)$,
\[
  c(T) \;=\;
  w_{\mathrm{Sh}}(|T|-1)\Big[\sum_{i\in T}\pi_x(T\setminus\{i\})
  \;-\; |T|\,\pi_x(T)\Big],
  \qquad \emptyset\subsetneq T\subsetneq\mathcal{F},
\]
so $c(T)=0$ iff $\pi_x(T)=\tfrac1{|T|}\sum_{i\in T}\pi_x(T\setminus\{i\})$
$(\star)$. The boundary coefficients (from $S=\emptyset$ and
$S=\mathcal F\setminus\{i\}$ respectively) are: coefficient of
$v(\emptyset)$ equal to $-\pi_x(\emptyset)$, so exact efficiency (which
needs this coefficient to equal $-1$) forces $\pi_x(\emptyset)=1$; and
coefficient of $v(N)$ equal to $\tfrac1n\sum_i\pi_x(\mathcal F\setminus\{i\})$,
so exact efficiency (needing coefficient $+1$) forces
$\tfrac1n\sum_i\pi_x(\mathcal F\setminus\{i\})=1$ $(\star\star)$. Observe
that $\pi_x(\mathcal F)$ itself never appears in $c(T)$ (domain
$1\le|T|\le n-1$) nor in either boundary coefficient: e.g.\ taking
$\pi_x\equiv1$ on every proper subset and $\pi_x(\mathcal F)\ne1$ makes
$\pi_x$ non-constant on $2^{\mathcal F}$ while every coefficient above is
already at its efficiency-preserving value, so exact efficiency still holds
for every $v$ --- confirming that non-constancy must be required of the
effective kernel, not of $\pi_x$ on all of $2^{\mathcal F}$.

\emph{Necessity: $\pi_x\equiv1$ on proper subsets is the only effective
kernel that works.} Suppose $\pi_x(\emptyset)=1$ and $(\star)$ holds for
every $1\le|T|\le n-1$. By induction on $|T|$, $\pi_x(T)=1$ for every
$T\subsetneq\mathcal F$: the base case $|T|=0$ is $\pi_x(\emptyset)=1$;
for the inductive step, $(\star)$ writes $\pi_x(T)$ as the average of
$\pi_x(T\setminus i)$ over $i\in T$, each a proper subset of size
$|T|-1<|T|$ and hence, by the inductive hypothesis, equal to $1$, so
$\pi_x(T)=1$. In particular $\pi_x(\mathcal F\setminus i)=1$ for every $i$,
so $(\star\star)$ holds automatically and is not an independent
constraint. This proves the ``conversely'' half of the proposition; the
converse direction (that $\pi_x\equiv1$ on proper subsets does satisfy
$(\star)$, $\pi_x(\emptyset)=1$, and $(\star\star)$, hence recovers the
standard Shapley telescoping) is immediate by inspection.

\emph{Sufficiency: a non-constant effective kernel breaks efficiency for
some $v$.} If the effective kernel is non-constant, either
$\pi_x(\emptyset)\ne1$ --- take $v\equiv0$ except $v(\emptyset)=1$; the
coefficient of $v(\emptyset)$ in $\sum_i N_i$ is $-\pi_x(\emptyset)$, so
$\sum_i N_i-[v(N)-v(\emptyset)] = -\pi_x(\emptyset)\cdot1 - (0 - 1)
= 1-\pi_x(\emptyset)\ne0$ --- or $\pi_x(\emptyset)=1$ but the induction above
fails at some minimal order, i.e.\ $c(T^\ast)\ne0$ for some
$\emptyset\subsetneq T^\ast\subsetneq\mathcal F$ while $c(T)=0$ for every
proper subset $T$ of $T^\ast$ (by minimality of $|T^\ast|$). Define
$v(T^\ast):=1$, $v(T'):=0$ for every other
$\emptyset\subsetneq T'\subsetneq\mathcal F$ with $T'\ne T^\ast$, and
$v(N),v(\emptyset)$ arbitrary (fixed, e.g.\ both $0$). Then
$\sum_i N_i - [v(N)-v(\emptyset)] = \sum_{\emptyset\subsetneq
T'\subsetneq\mathcal F} c(T')\,v(T') = c(T^\ast)\ne0$, so
$\sum_i N_i\ne v(N)-v(\emptyset)$ for this $v$. This proves the
proposition in both cases.\qed
\end{proof}

\subsection{Worked $n=2$ computation for Remark~\ref{rem:normalized-incompatibility}}

Let $\pi(\emptyset)=1$, $\pi(\{1\})=2$, $\pi(\{2\})=1/2$ (non-constant on
the effective domain). With $w_{\mathrm{Sh}}(0)=w_{\mathrm{Sh}}(1)=1/2$,
$Z_{x,1}=\tfrac12(1)+\tfrac12(\tfrac12)=\tfrac34$ and
$Z_{x,2}=\tfrac12(1)+\tfrac12(2)=\tfrac32$. Direct computation gives
\[
  \phi_1=\tfrac23(v_1-v_0)+\tfrac13(v_{12}-v_2), \qquad
  \phi_2=\tfrac13(v_2-v_0)+\tfrac23(v_{12}-v_1),
\]
(writing $v_0=v(\emptyset)$, $v_1=v(\{1\})$, $v_2=v(\{2\})$,
$v_{12}=v(\{1,2\})$), so $\phi_1+\phi_2=v_{12}-v_0=v(N)-v(\emptyset)$
\emph{identically in $v$}: exact efficiency holds for every game $v$,
despite the effective kernel being non-constant. More generally, for
$n=2$ and $\pi(\emptyset)=a$, $\pi(\{1\})=b$, $\pi(\{2\})=c$, direct
coefficient matching shows normalized efficiency for every $v$ holds if
and only if $a^2=bc$ (the example above has $1^2=2\cdot\tfrac12$); this is
a strictly larger family of kernels than the unnormalized case's unique
solution $a=b=c=1$. A full characterization of the efficiency-preserving
normalized kernels for general $n$ is not derived here and is left open;
in particular, we do \emph{not} claim (and Proposition~\ref{prop:incompatibility}
does not imply) that every non-constant effective kernel breaks the
normalized Algorithm~1 estimator's completeness.

\subsection{Full proof of Theorem~\ref{thm:mc}}

\begin{proof}
\emph{Step 1 --- Unbiased estimator (for the population $k$-step numerator).}
Let $Q_{\pi,i}^{(k)} := \pi_x(x_{S_\pi(i)})\,\mathrm{IG}_i^{(k)}(S_\pi(i))$
be the contribution of a single permutation $\pi$ --- using the same
$k$-step discretized IG term that Algorithm~\ref{alg:mc} actually
evaluates, not the exact integral. Since $S_\pi(i)\sim$ the Shapley
predecessor law, $\mathbb{E}_\pi[Q_{\pi,i}^{(k)}] = N_i^{(k)}$ by
definition of $N_i^{(k)}$ above, so $\hat N_i^{(m,k)}=\tfrac1m\sum_r
Q_{\pi_r,i}^{(k)}$ is exactly unbiased for $N_i^{(k)}$ --- \emph{not} for
$N_i$; the gap between the two is handled separately, deterministically,
in Step~3.

\emph{Step 2 --- Variance bound.}
$\Var[\hat{N}_i^{(m,k)}] \le B^2/m$ since $|Q_{\pi,i}^{(k)}| \le B$ (each
of the $k$ pointwise gradient terms averaged inside $\mathrm{IG}_i^{(k)}$
is bounded by $B/\pi_x(x_S)$-scaled quantities exactly as in the exact-IG
case, so the average, and hence $\pi_x(x_S)\mathrm{IG}_i^{(k)}(S)$
itself, inherits the same pointwise bound $B$). By Chebyshev:
$P(|\hat{N}_i^{(m,k)} - N_i^{(k)}| > B/\sqrt{m\delta}) \le \delta$.

\emph{Step 3 --- IG path discretization error (midpoint rule), and its
deterministic transfer from $N_i^{(k)}$ to $N_i$.}
Algorithm~\ref{alg:mc} evaluates $g$ at the midpoint of each subinterval:
$\frac{1}{k}\sum_{j=1}^k g\bigl(\tfrac{j-0.5}{k}\bigr)$ approximates $\int_0^1 g\,d\alpha$.
For $g \in C^2([0,1])$, the composite midpoint rule satisfies the standard error bound
\[
  \Bigl|\int_0^1 g\,d\alpha - \frac{1}{k}\sum_{j=1}^k g\bigl(\tfrac{j-0.5}{k}\bigr)\Bigr|
  \;\le\; \frac{1}{24k^2}\,\|g''\|_\infty,
\]
obtained by Taylor-expanding $g$ around the midpoint of each subinterval: the
linear term vanishes by symmetry of the (symmetric) integration domain around
each midpoint, leaving a second-order remainder --- this is the mechanism by
which the midpoint rule gains one order of convergence over an endpoint
(Riemann) rule at the cost of one additional derivative of $g$.
Since $g(\alpha) = \nabla_i F(\tilde{x}^{(S,\alpha)})$ and
$\tilde{x}^{(S,\alpha)}_j = x'_j + \alpha(x_j - x'_j)$ for $j\in S\cup\{i\}$,
the chain rule gives
\[
  g'(\alpha) = \sum_{j\in S\cup\{i\}}
               \frac{\partial^2 F}{\partial x_i\,\partial x_j}
               \!\bigl(\tilde{x}^{(S,\alpha)}\bigr)\cdot(x_j - x'_j),
  \qquad
  g''(\alpha) = \sum_{j,l\in S\cup\{i\}}
               \frac{\partial^3 F}{\partial x_i\,\partial x_j\,\partial x_l}
               \!\bigl(\tilde{x}^{(S,\alpha)}\bigr)\cdot(x_j - x'_j)(x_l - x'_l).
\]
Therefore $\|g''\|_\infty \le \|\nabla^3 F\|_\infty \cdot D_S^2$, where
$D_S \;=\; \sum_{j\in S\cup\{i\}}|x_j-x'_j| \;\le\; \|x-x'\|_1$ as before,
yielding, for \emph{every} coalition $S$ (the bound is uniform in $S$,
since $D_S\le\|x-x'\|_1$ regardless of $S$), midpoint quadrature error
\[
  \bigl|\mathrm{IG}_i^{(k)}(S)-\mathrm{IG}_i(S)\bigr| \;\le\; E_{\mathrm{quad}}(k)
  \;:=\; \dfrac{|x_i-x'_i|\cdot\|x-x'\|_1^2\cdot\|\nabla^3 F\|_\infty}{24k^2} \;=\; O(1/k^2).
\]
Gauss--Legendre quadrature with $p$ nodes achieves $O(1/k^{2p})$ under
correspondingly higher smoothness. Since this per-coalition bound is
uniform in $S$ and $0\le\pi_x(x_S)\le1$, taking expectation over
$S\sim\mathrm{Shapley}$ transfers it \emph{deterministically} (no
probability involved) from the per-coalition IG term to the population
$k$-step numerator itself:
\[
  \bigl|N_i^{(k)}-N_i\bigr|
  \;=\;
  \Bigl|\mathbb E_S\bigl[\pi_x(x_S)\bigl(\mathrm{IG}_i^{(k)}(S)-\mathrm{IG}_i(S)\bigr)\bigr]\Bigr|
  \;\le\;
  \mathbb E_S\bigl[\pi_x(x_S)\,\bigl|\mathrm{IG}_i^{(k)}(S)-\mathrm{IG}_i(S)\bigr|\bigr]
  \;\le\; E_{\mathrm{quad}}(k).
\]

\emph{Step 4 --- Total bound.} Combining Steps~1--2 (stochastic:
$|\hat N_i^{(m,k)}-N_i^{(k)}|\le B/\sqrt{m\delta}$ with probability
$\ge1-\delta$) with Step~3 (deterministic: $|N_i^{(k)}-N_i|\le
E_{\mathrm{quad}}(k)$, holding always) via the triangle inequality
\[
  \bigl|\hat N_i^{(m,k)}-N_i\bigr|
  \;\le\;
  \bigl|\hat N_i^{(m,k)}-N_i^{(k)}\bigr| + \bigl|N_i^{(k)}-N_i\bigr|
  \;\le\;
  \frac{B}{\sqrt{m\delta}} + E_{\mathrm{quad}}(k)
\]
with probability $\ge1-\delta$, which is exactly the theorem's stated
bound. \qquad $\square$
\end{proof}

\begin{remark}[Variance reduction --- antithetic sampling]
$\hat{\phi}_i^{\mathrm{antith}} = \frac{1}{2m}\sum_{t=1}^m[Q_{\pi^{(t)},i} + Q_{\pi^{(t)\mathrm{rev}},i}]$.
By the generic identity
$\Var[\hat{\phi}_i^{\mathrm{antith}}] = \tfrac{1}{2m}\Var[Q_\pi] + \tfrac{1}{2m}\mathrm{Cov}[Q_{\pi},Q_{\pi^{\mathrm{rev}}}]$,
the bound $\Var[\hat{\phi}_i^{\mathrm{antith}}] \le \Var[Q_\pi]/2m$ holds
\emph{whenever} $\mathrm{Cov}[Q_\pi,Q_{\pi^{\mathrm{rev}}}]\le0$, which is
the expected --- but not unconditionally guaranteed --- regime here: the
coalitions from $\pi$ and $\pi^{\mathrm{rev}}$ are complementary, which
typically induces negative correlation between the two marginal
contributions, but the sign of the covariance depends on $F$ and is not
established in general. Note that $\hat\phi_i^{\mathrm{antith}}$ itself
requires $2m$ permutation evaluations (both $\pi^{(t)}$ and its reversal
$\pi^{(t)\mathrm{rev}}$, each generally inducing a different predecessor
coalition and hence a different conditioned-path gradient evaluation for
every feature): it is \emph{not} free relative to an ordinary $m$-permutation
estimate. The correct comparison is to an \emph{equal-budget} ordinary
estimate using $2m$ i.i.d.\ permutations: under the covariance condition
above, the antithetic-pair estimator with $m$ pairs (total $2m$
evaluations) has variance $\Var[Q_\pi]/2m$, matching or improving on the
$\Var[Q_\pi]/2m$ variance of $2m$ i.i.d.\ samples whenever
$\mathrm{Cov}[Q_\pi,Q_{\pi^{\mathrm{rev}}}]\le0$, at \emph{no additional
evaluation cost relative to that equal-budget baseline} (generating the
reversed permutation itself is a free $O(n)$ operation; the $2m$ gradient
evaluations are needed either way).
\end{remark}

\begin{proof}[Full proof of Corollary~\ref{cor:ratio-bound}]
\emph{(a) Deterministic denominator bound.} Every sampled term
$\pi_x(x_{S_r})$ appearing in $\hat Z_i=\frac1m\sum_r\pi_x(x_{S_r})$
satisfies $\pi_x(x_{S_r})\ge\pi_{\min,i}$ by definition of $\pi_{\min,i}$,
so the average itself satisfies $\hat Z_i\ge\pi_{\min,i}$
\emph{deterministically}, for \emph{every} realization of
$S_1,\dots,S_m$ --- no concentration argument, and in particular no
separate high-probability event, is needed to keep the denominator away
from zero, because the underlying kernel is already bounded away from
zero over the finite predecessor-coalition support. The same
pointwise bound applied to the population average gives
$Z_i=\sum_S w_{\mathrm{Sh}}(|S|)\pi_x(x_S)\ge\pi_{\min,i}$.

\emph{(b) Numerator and denominator concentration.} By
Theorem~\ref{thm:mc} in full --- \emph{not} just its Step~2, which alone
only controls $\hat N_i$ around the population $k$-step numerator
$N_i^{(k)}$, not around $N_i$ itself --- combining the Step~1--2
stochastic bound with Step~3's deterministic quadrature correction gives
$P\bigl(|\hat N_i-N_i| > B_N/\sqrt{m\delta_1} + E_{\mathrm{quad}}(k)\bigr)\le\delta_1$
for any $\delta_1\in(0,1)$ (the $E_{\mathrm{quad}}(k)$ term is
deterministic, so it simply shifts the threshold rather than entering
the probability). An identical Chebyshev argument applied to $\hat Z_i$
(using $0\le\pi_x(x_{S_r})\le B_Z\le1$; $\hat Z_i$ involves no
discretization, hence no analogous quadrature term) gives
$\Var[\hat Z_i]\le B_Z^2/m$ and hence
$P\bigl(|\hat Z_i-Z_i|>B_Z/\sqrt{m\delta_2}\bigr)\le\delta_2$ for any
$\delta_2\in(0,1)$.

\emph{(c) Assembly.} Take $\delta_1=\delta_2=\delta/2$; by a union bound
both events hold simultaneously with probability $\ge1-\delta$. On that
event, apply the elementary decomposition
\[
  \left|\frac{\hat N_i}{\hat Z_i}-\frac{N_i}{Z_i}\right|
  \;\le\;
  \frac{|\hat N_i-N_i|}{\hat Z_i}
  \;+\;
  \frac{|N_i|\,|\hat Z_i-Z_i|}{\hat Z_i\,Z_i},
\]
together with: $|N_i|\le B_N$ (each term $\pi_x(x_S)\mathrm{IG}_i(S)$ of
$N_i=\sum_S w_{\mathrm{Sh}}(|S|)\pi_x(x_S)\mathrm{IG}_i(S)$ is bounded by
$B_N$ by the same pointwise argument used for Theorem~\ref{thm:mc}, and
the Shapley weights $w_{\mathrm{Sh}}(|S|)$ sum to $1$); and part~(a)'s
bounds $\hat Z_i\ge\pi_{\min,i}$, $Z_i\ge\pi_{\min,i}$. Substituting the
bound from~(b) for $|\hat N_i-N_i|$ yields exactly the stated bound.
$\square$
\end{proof}

\section{Justification via Möbius Transform\\of the GRALIS--SIV Correspondence}
\label{app:mobius}

\subsection{Motivation}

Theorem~\ref{thm:siv} shows that GRALIS's induced cooperative game
$\vG(S)=f(x_S)-f(x')$ (Eq.~\eqref{eq:vG}) admits exact Shapley
Interaction Values, computed via the standard machinery of cooperative
game theory. This appendix makes that machinery explicit through the
Möbius transform --- the canonical tool for decomposing a function
defined on coalitions into its pure higher-order contributions --- and
records, in Corollary~\ref{cor:siv-gralis-relation} (main text), the
precise relationship between $\mathrm{Sh}(\vG)$ and $\phi^{\GRALIS}$
itself.

\subsection{Cooperative Game Induced by GRALIS}

For every coalition $S \subseteq N$ (writing $N:=\mathcal F$),
\begin{equation}
  \vG(S) \;:=\; f(x_S) - f(x'),
  \label{eq:vG-def}
\end{equation}
as in Eq.~\eqref{eq:vG}. This is a finite-valued function on the finite
lattice $2^N$: no measurability, integrability, or $\sigma$-finiteness
hypothesis is required, and $\vG(\emptyset)=f(x')-f(x')=0$,
$\vG(N)=f(x)-f(x')$ hold automatically.

\begin{remark}[Why not build $\vG$ from the continuous index space]
An alternative construction would define $\vG$ by integrating a
scalarized marginal contribution over the preimages of a measurable
projection $\rho:\mathcal Q\to2^N$, as in Appendix~\ref{app:projection}.
That construction requires the
apparatus of Appendix~\ref{app:projection} ($\sigma$-finite measures,
measurable partitions, push-forward measures) to even state, and --- as
discussed in Remark~\ref{rem:siv-scope} --- collapses to a feature-blind
game once Definition~\ref{def:canonical}'s constitutive constraint is
taken into account. Definition~\eqref{eq:vG-def} sidesteps both issues:
it is the standard coalition-value game already used implicitly by
Lemma~A (Theorem~\ref{thm:completeness}) and by
Corollary~\ref{cor:mobius-correction}, whose hypothesis already refers
to ``the discrete game $v(S):=F(x_S)-F(x')$ used in
Theorem~\ref{thm:siv} and Appendix~\ref{app:mobius}.''
\end{remark}

\subsection{Möbius Transform}

For every set function $v : 2^N \to \R$, the Möbius transform is:
\[
  \mG(T) \;=\; \sum_{A \subseteq T} (-1)^{|T|-|A|}\, \vG(A), \qquad \forall T \subseteq N,
\]
with inverse $v(S) = \sum_{T \subseteq S} \mG(T)$.
The coefficient $\mG(T)$ represents the pure contribution of coalition $T$,
i.e.\ the part of $\vG(T)$ not explainable by proper sub-coalitions.

Applying the decomposition to the GRALIS game and substituting~\eqref{eq:vG-def}:
\begin{equation}
  \mG(T) \;=\; \sum_{A \subseteq T} (-1)^{|T|-|A|}\bigl[f(x_A)-f(x')\bigr].
  \label{eq:mobius-integral}
\end{equation}
No exchange of sum and integral is needed here: $\vG$ is already a
finite-valued function on the finite lattice $2^N$ ($\le 2^n$ terms), so
$\mG(T)$ is a finite sum of point evaluations of $f$.

\subsection{Second-Order Interactions}
\label{sec:interazioni}

For two players $i, j \in N$, the interaction conditioned on $S \subseteq N\setminus\{i,j\}$:
\[
  \Delta_{ij}\vG(S) \;=\; \vG(S\cup\{i,j\}) - \vG(S\cup\{i\}) - \vG(S\cup\{j\}) + \vG(S).
\]
Using the Möbius representation $\vG(S) = \sum_{T \subseteq S} \mG(T)$:
\[
  \Delta_{ij}\vG(S) \;=\; \sum_{\substack{T : \{i,j\} \subseteq T \subseteq S\cup\{i,j\}}} \mG(T).
\]
Terms not containing both $i$ and $j$ simultaneously cancel out.
This identity shows that the second-order discrete difference captures
exactly the pure contributions involving both $i$ and $j$ simultaneously.

\subsection{Recovery of Shapley Interaction Values}

The SIV of pair $(i,j)$:
$\IijSh(\vG) = \sum_{S \subseteq N\setminus\{i,j\}} \pi_n(|S|)\cdot\Delta_{ij}\vG(S)$.
Substituting the expression in terms of Möbius and exchanging the sums:
\begin{equation}
  \IijSh(\vG) \;=\; \sum_{\substack{T \subseteq N \\ \{i,j\} \subseteq T}}
  \frac{1}{|T|-1}\cdot\mG(T),
  \label{eq:siv-mobius}
\end{equation}
where $\alpha_T = 1/(|T|-1)$ is the weight with which the pure contribution $\mG(T)$
participates in the interaction (Grabisch \& Roubens~\citep{grabisch1999}, Thm.~3.1).

\subsection{Interpretation for GRALIS}

In this construction:
\begin{enumerate}
  \item $f(x_S)-f(x')$ is the value of coalition $S$ acting alone (all
        other features held at baseline);
  \item this directly defines the coalition value $\vG(S)$, Eq.~\eqref{eq:vG-def};
  \item the Möbius transform decomposes $\vG$ into pure effects $\mG(T)$;
  \item the SIV aggregates the pure effects involving both $i$ and $j$ simultaneously.
\end{enumerate}
The projection $\rho$ of Appendix~\ref{app:projection} plays no role in
this construction; it is used solely for the weight-mass diagnostic of
Remark~\ref{rem:weight-mass}. (The prior, separate question of which raw
pixels constitute each discrete feature in $\mathcal F=N$ --- e.g.\ via SLIC
segmentation, Algorithm~\ref{alg:mc} --- is handled by the segmentation map
$\rho_{\mathrm{slic}}:\Omega\to N$ of the ``Canonical operators for
imaging'' remark above, a different, single-feature-valued map from the
$\rho:\mathcal Q\to2^N$ formalized in Appendix~\ref{app:projection}.)

\subsection{Final Proposition}

\begin{proposition}
\label{prop:mobius}
For every pair $i, j \in N$:
\[
  \IijSh(\vG)
  \;=\; \sum_{\substack{T \subseteq N \\ \{i,j\} \subseteq T}}
  \frac{1}{|T|-1}\cdot\mG(T).
\]
\end{proposition}
\begin{proof}
Formula~\eqref{eq:siv-mobius} follows by the standard Möbius inversion formula
for set functions~\cite[Prop.~3.1]{grabisch1999}: substitute the representation
$\vG(S)=\sum_{T\subseteq S}\mG(T)$ into $\Delta_{ij}\vG(S)$, expand, and collect
terms by support $T\supseteq\{i,j\}$; the weight $1/(|T|-1)$ arises from the
symmetrization of the Shapley-interaction weights over coalitions $S\subseteq T\setminus\{i,j\}$.
\end{proof}

\subsection{Concluding Remark}

The Möbius transform shows that computing Shapley Interaction Values on
$\vG$ is not a heuristic simplification but the exact, standard
construction for a finite cooperative game.

The correct formulation is not:
\begin{center}
  \textit{GRALIS approximates the SIVs,}
\end{center}
but:
\begin{center}
  \textit{GRALIS associates to $(f,x,x')$ the standard coalition game $\vG(S)=f(x_S)-f(x')$
  and computes the SIVs exactly on $\vG$} --- with the relationship
  between $\mathrm{Sh}(\vG)$ and $\phi^{\GRALIS}$ itself given precisely
  by Corollary~\ref{cor:siv-gralis-relation}, not by identity.
\end{center}

The link between the Möbius coefficients $\mG(T)$ and the components of the
Hoeffding decomposition (Theorem~\ref{thm:anova}) holds under additional
hypotheses --- in particular feature independence ($\mu = \bigotimes_i \mu_i$).
In general, GRALIS provides a common construction from which both
ANOVA-type decompositions and Shapley interaction indices can be derived
under appropriate assumptions, without one automatically coinciding with the other.

\section{Deferred Proof of Theorem~\ref{thm:anova}, Conditional General-\texorpdfstring{$F$}{F} Extension of T5/T6, and Deferred Material for T7}
\label{app:anova-general}

\subsection{Full proof of Theorem~\ref{thm:anova} (affine case)}

\begin{proof}
\emph{Step 1 (kernel-independence for affine $F$, and the singleton
values).} For affine
$F(x)=a+\sum_j b_j x_j$, $\nabla_i F\equiv b_i$ is constant, so the
conditioned IG term $\mathrm{IG}_i(S)=(x_i-x'_i)\int_0^1\nabla_iF(\tilde
x^{(S,\alpha)})\,d\alpha = b_i(x_i-x'_i)$ is the \emph{same value for
every coalition $S$}, independent of the locality kernel. Hence in
Definition~\ref{def:explicit}'s explicit formula the common factor
$b_i(x_i-x'_i)$ pulls out of the weighted sum and cancels against the
identical sum in the denominator $Z_{x,i}$:
\[
  \phi_i^{\GRALIS}(x)
  \;=\;
  \frac{\sum_S w_{\mathrm{Sh}}(|S|)\,\pi_x(x_S)\cdot b_i(x_i-x'_i)}
       {\sum_S w_{\mathrm{Sh}}(|S|)\,\pi_x(x_S)}
  \;=\; b_i(x_i-x'_i),
\]
for \emph{any} strictly positive kernel $\pi_x$ (in particular any fixed
$\sigma>0$) and any baseline $x'$. By Definition~\ref{def:phi-T-affine},
$\Phi_{\{i\}}^{\GRALIS}(x)=b_i(x_i-x'_i)$; all higher-order interaction
terms are \emph{set to} $0$ by that same definition, which is
justified here by the constancy of $\nabla_iF$ (no curvature is present
in $F$ for a higher-order term to capture, so the convention
$\Phi_T^{\GRALIS}:=0$, $|T|\ge2$, loses nothing).

\emph{Step 2 (exact completeness, verified directly from Step~1 ---
not assumed).} By Step~1 and affineness of $F$, for \emph{any} fixed
baseline $x'$:
\begin{align*}
  F(x') + \sum_{i=1}^n \Phi_{\{i\}}^{\GRALIS}(x) + \sum_{|T|\ge2}\Phi_T^{\GRALIS}(x)
  &\;=\;
  \Bigl(a+\textstyle\sum_j b_jx'_j\Bigr) + \textstyle\sum_i b_i(x_i-x'_i) + 0 \\
  &\;=\; a+\textstyle\sum_j b_jx_j \;=\; F(x).
\end{align*}
This is a two-line algebraic consequence of Definition~\ref{def:phi-T-affine}
and affineness alone. We deliberately do \emph{not} derive it by
invoking ``the Shapley efficiency axiom'' as a general property of
GRALIS attributions at every order: such a general claim would be
\emph{false} for Algorithm~\ref{alg:mc}'s actual output whenever $F$ is
not affine (the toy example of Remark~\ref{rem:mobius-scope} gives
$\sum_i\phi_i^{\GRALIS}=8.5\ne11=F(x)-F(x')$), and stating completeness
here as an unqualified efficiency axiom would risk being read as a
general practical-completeness claim in tension with
Remark~\ref{rem:t2-scope} and Remark~\ref{rem:approx_completeness}. The
affine case needs no such general claim: Step~2 above is self-contained.
Setting $x'=\mathbb E_\mu[X]$ gives $F(x')=\mathbb E_\mu[F]$ by
linearity of $F$, used in Step~4 below.

\emph{H1 (existence and uniqueness).}
The Hoeffding decomposition is the unique orthogonal decomposition of
$F \in L^2(\mu)$ such that $\int f_S\,d\mu_i = 0$ for every $i \in S$
(\citep{hoeffding1948}, Efron \& Stein~\citep{efron1981}); its constant term is
$\mathbb{E}_\mu[F]$ by construction, not $F$ evaluated at any particular point.
Under product $\mu$, the projectors $P_S : L^2(\mu) \to L^2_S(\mu)$ are
well defined via marginal integration. This is exactly why the theorem's
hypothesis is stated as $\mathbb E_\mu[\|X\|^2]<\infty$: for affine $F$,
$\mathbb E_\mu[\|X\|^2]<\infty$ is a sufficient condition for
$F\in L^2(\mu)$ (equivalently $\mathbb E_\mu[F^2]<\infty$), by the same
argument as in the theorem's statement --- it is not necessary in
general (the same counterexample applies: for $F(X_1,X_2)=X_1$,
$F\in L^2(\mu)$ already holds whenever $\mathbb E[X_1^2]<\infty$, even if
$\mathbb E[X_2^2]=\infty$). The weaker requirement $F\in L^2(\mu)$ would
itself suffice for the Hoeffding decomposition invoked here; the stated
second-moment hypothesis is used only because it gives a simple,
feature-wise sufficient condition. A merely finite \emph{first} moment
does not imply $F\in L^2(\mu)$ (e.g.\ independent heavy-tailed
coordinates can have $\mathbb E[|X_i|]<\infty$ while
$\mathbb E[X_i^2]=\infty$) --- a finite first moment alone would not
license the $L^2(\mu)$ orthogonal decomposition used here.

\emph{H2 (zero-mean property, exact for affine $F$).}
We verify that $\int \Phi_T^{\GRALIS}(x)\,d\mu_i(x_i) = 0$ for every $i \in T$.
By Step~1, $\Phi_{\{i\}}^{\GRALIS}(x) = b_i(x_i - x'_i)$ exactly, with
$x'=\mathbb E_\mu[X]$ fixed as hypothesized (no limit involved), so
$\int \Phi_{\{i\}}^{\GRALIS}\,d\mu_i = b_i\int(x_i-\mathbb{E}_\mu[x_i])\,d\mu_i = 0$;
for $|T|\ge2$, $\Phi_T^{\GRALIS}\equiv0$ by Definition~\ref{def:phi-T-affine},
so the zero-mean condition holds trivially. Note that for affine $F$
this also makes the constant term unambiguous:
$F(x')=F(\mathbb{E}_\mu[X])=\mathbb{E}_\mu[F(X)]$ by
linearity, so $F(x')$ and $\mathbb{E}_\mu[F]$ coincide in this case and the two
possible ways of writing the decomposition's constant agree. This coincidence is
special to the affine case (see Remark~\ref{rem:anova-baseline}); the statement above
is written with $\mathbb{E}_\mu[F]$ because that is what H1's uniqueness argument
actually forces in general.

\emph{Step 4 (uniqueness closes the argument).} By H1, the Hoeffding
decomposition is the unique family $\{f_T\}_{T\ne\emptyset}$ of
zero-mean, mutually orthogonal terms with
$F(x)=\mathbb E_\mu[F]+\sum_T f_T(x)$. Step~2 (with $x'=\mathbb E_\mu[X]$,
so $F(x')=\mathbb E_\mu[F]$) exhibits $\{\Phi_T^{\GRALIS}\}_T$ as
exactly such a family: it reproduces $F$ exactly, with constant term
$\mathbb E_\mu[F]$ (Step~2), and each term is zero-mean (H2). By
uniqueness, $\Phi_T^{\GRALIS}=f_T$ for every $T$, which is the
theorem's conclusion.

\emph{Note.} Orthogonality requires product $\mu$.
For non-product measures the terms $f_S$ are not orthogonal and the decomposition
loses uniqueness; Theorem~\ref{thm:anova} applies only under the feature-independence
assumption stated in the hypothesis.
\end{proof}

\subsection{General smooth-\texorpdfstring{$F$}{F} extension of the Hoeffding/Sobol correspondence (T5/T6)}

Theorem~\ref{thm:anova}'s affine case is unconditional and fully derived
(main text). A general smooth-$F$ analogue is possible only under an
additional projector-identity hypothesis, close to definitional, as
discussed in the main text preceding this appendix.

\begin{proposition}[Coincidence with Hoeffding Decomposition --- general smooth $F$, conditional; epistemic status: near-definitional]
\label{prop:anova-general}
\textbf{Status.} A conditional characterization, not an independently
derived result: it states what a general-$F$ analogue of
Theorem~\ref{thm:anova} would require, adding no theorem-level content
beyond its own hypothesis. It is not used elsewhere to derive a further
consequence (T6 cites it only under the same caveat; the experiments and
Appendix~\ref{app:toy-verification} exercise only the affine/multilinear
machinery).

Under the same setup as Theorem~\ref{thm:anova} (product measure $\mu$,
limit $\pi_x\to\mu$), suppose $F\in C^2$ and the projector identity
$\Phi_T^{\GRALIS} = P_T F$ holds at every order $T$, where $P_T$ is the
Hoeffding orthogonal projector under $\mu$ and $\Phi_T^{\GRALIS}$ for
$|T|\ge2$ --- undefined independently elsewhere in this paper, since
Algorithm~\ref{alg:mc} only outputs order-1 quantities and
Theorem~\ref{thm:siv}'s interaction machinery lives on the separate
auxiliary game $\vG$ (Remark~\ref{rem:siv-scope}), not on a Hoeffding
projection of $F$ --- is understood as \emph{introduced by} this
identity. Then Theorem~\ref{thm:anova}'s conclusion,
$F(x)=\mathbb{E}_\mu[F]+\sum_{\emptyset\ne T\subseteq\mathcal
F}\Phi_T^{\GRALIS}(x)$, holds trivially given the hypothesis. Compatible
with, but not established by, the Owen~\citep{owen2014} Shapley--ANOVA
connection (which addresses only the affine/multilinear regime already
covered unconditionally by Theorem~\ref{thm:anova}); an independently
derived nonlinear version is left as future work. Unlike
Theorem~\ref{thm:anova}, this should be read as a conditional restatement
of its own hypothesis, not an unconditionally derived result.
\end{proposition}

\subsection{Deferred proof and correlated-layer discussion for T7 (MS-GRALIS)}
\label{app:t7-deferred}

\begin{proof}[Proof of Theorem~\ref{thm:ms}]
Under the assumption that the layer-wise attributions
$\{\mathrm{GRALIS}_i^{(\ell)}\}$ are mutually independent,
$\Var[\mathrm{GRALIS}_i^{\mathrm{MS}}] = \sum_\ell \lambda_\ell^2
\sigma_\ell^2$. Minimizing with Lagrange multiplier $\nu$ subject to
$\sum_\ell\lambda_\ell=1$:
\[
  \frac{\partial}{\partial \lambda_\ell}
  \Bigl[\sum_\ell \lambda_\ell^2 \sigma_\ell^2 + \nu\bigl(\sum_\ell \lambda_\ell - 1\bigr)\Bigr] = 0
  \;\Rightarrow\; 2\lambda_\ell \sigma_\ell^2 + \nu = 0
  \;\Rightarrow\; \lambda_\ell \propto \sigma_\ell^{-2}.
\]
The objective is strictly convex in $\boldsymbol{\lambda}$, so the KKT conditions are
sufficient for a global minimum.
\end{proof}

\begin{remark}[Independence assumption in MS-GRALIS: correlated layers]
\label{rem:ms-independence}
The independence assumption $\Cov[\mathrm{GRALIS}_i^{(\ell)}, \mathrm{GRALIS}_i^{(\ell')}]=0$
for $\ell\neq\ell'$ is a simplifying hypothesis.
In practice, multi-scale attributions of the same model are computed on the same input
with shared weights and therefore exhibit positive correlation.
Under positive correlation, $\Var[\mathrm{GRALIS}_i^{\mathrm{MS}}]
= \sum_\ell\lambda_\ell^2\sigma_\ell^2 + 2\sum_{\ell<\ell'}\lambda_\ell\lambda_{\ell'}\,\mathrm{Cov}_{\ell\ell'}$,
and since $\lambda_\ell>0$ and $\mathrm{Cov}_{\ell\ell'}>0$ the second term is
positive; the independence-based quantity $\sum_\ell\lambda_\ell^2\sigma_\ell^2$
is therefore a \emph{lower bound} on the true variance (it \emph{underestimates}
$\Var[\mathrm{GRALIS}_i^{\mathrm{MS}}]$ whenever the layers are positively
correlated). The weights $\lambda^*_\ell$ remain a well-motivated heuristic even
when independence does not hold exactly, but the resulting variance estimate
should be treated as optimistic (a lower bound), not conservative, under
positive inter-layer correlation.
The inverse-variance form is consistent with the classical result in meta-analysis
and weighted least squares, where the same formula is derived under a fixed-effects
model without requiring independence of estimators.
For general correlated layers with covariance matrix $\Sigma=(\mathrm{Cov}_{\ell\ell'})$
(positive definite), minimizing $\boldsymbol\lambda^\top\Sigma\boldsymbol\lambda$
subject to $\mathbf 1^\top\boldsymbol\lambda=1$ (via the same Lagrange
argument, now with a matrix quadratic form) gives
$\boldsymbol\lambda^*=\Sigma^{-1}\mathbf 1/(\mathbf 1^\top\Sigma^{-1}\mathbf 1)$,
which reduces to the inverse-variance formula above exactly when $\Sigma$ is
diagonal (independence); in general $\boldsymbol\lambda^*$ is \emph{not}
proportional to the marginal inverse variances $\sigma_\ell^{-2}$ alone,
and if the sign constraint $\lambda_\ell\ge0$ is imposed, $\Sigma^{-1}\mathbf1$
need not be nonnegative and the problem becomes a quadratic program with
no closed form in general.
\end{remark}

Layers with more stable attributions (low variance) receive greater weight.
GradCAM++~\citep{chattopadhyay2018} applied to multiple layers is a
\emph{related heuristic multi-scale construction}, not a derived special
case of MS-GRALIS: both aggregate per-layer attributions with
layer-specific weights, and GradCAM++'s scheme can be read informally as
(i) $\mathrm{GRALIS}_i^{(\ell)}$ approximated with a single gradient, and
(ii) $\lambda_\ell$ chosen heuristically rather than via
Theorem~\ref{thm:ms}'s variance-minimization criterion; but GradCAM++ is
not itself an instance of Definition~\ref{def:explicit}'s
$\phi_i^{\GRALIS}$ (it does not satisfy conditions (a)--(b) of
Theorem~\ref{thm:canonical} in the way linearized GradCAM does), so no
equation-level derivation of GradCAM++ from MS-GRALIS is claimed.

\section{Formalization of the Projection \texorpdfstring{$\rho$}{rho}
         (Secondary Diagnostic)}
\label{app:projection}

\begin{remark}[The push-forward weight mass $\tilde w$, and what it is not]
\label{rem:weight-mass}
This is a diagnostic quantity, not part of the definition of
$\phi^{\GRALIS}$ or of the auxiliary game $\vG(S)=f(x_S)-f(x')$
(Appendix~\ref{app:mobius}), both of which are defined directly on
$2^{\mathcal F}$ and bypass it entirely; it is included only to
formalize, and preempt the objection to, an otherwise-implicit
projection. Let $\rho:\mathcal Q\to2^N$, $\rho(S,\alpha):=S$, forget the
path parameter of GRALIS's sampling index space $\mathcal Q$
(Definition~\ref{def:canonical}), keeping the predecessor coalition ---
unrelated to the segmentation map $\rho_{\mathrm{slic}}:\Omega\to N$ of
Remark~\ref{rem:canonical-imaging}. Since $\rho$ is measurable,
$\{\rho^{-1}(S)\}_{S\subseteq N}$ is a measurable partition of
$\mathcal Q$, so $\tilde w(S):=\int_{\rho^{-1}(S)}w\,d\mu$ defines a
probability measure on $2^N$ ($\tilde w\ge0$,
$\sum_S\tilde w(S)=\int_{\mathcal Q}w\,d\mu=1$), invariant under
relabeling the coalitions up to the induced permutation, and the
underlying operator $P_\rho f(S):=\int_{\rho^{-1}(S)}f\,d\mu$ is bounded,
positive, and linear ($\|P_\rho\|\le1$). $\tilde w(S)$ answers ``how much
sampling weight is spent on coalitions inside $S$'' --- useful for
auditing whether the LIME kernel $\pi_x$ concentrates mass on a few
coalitions --- not ``what is the model's output change attributable to
$S$'' (that is $\vG(S)$). $\phi^{\GRALIS}$ is \emph{not} obtainable by
composing $\tilde w$ with the Shapley operator: no identity relates
$\phi^{\GRALIS}$ to $\mathrm{Sh}(\tilde w)$, and $\vG$ carries no
$\rho$-dependence to begin with.
\end{remark}

\section{Numerical Verification on a Toy Multilinear Example}
\label{app:toy-verification}

This appendix complements the direct proofs of T1--T3 and the algebraic
argument of Appendix~\ref{app:mobius} (T4) with an explicit, fully worked
numerical example, computed with exact rational/symbolic arithmetic (no
floating-point rounding). All computations are reproducible; scripts are
included in the supplementary code (Section~\ref{sec:data-availability}).

\textbf{Setup.} $n=3$,
$F(x_1,x_2,x_3) = 2x_1+3x_2+x_3+4x_1x_2-x_1x_3+2x_2x_3$ (multilinear,
degree $\le1$ per coordinate), $x'=(0,0,0)$, $x=(1,1,1)$, so
$F(x)-F(x')=11$. Möbius coefficients: $m(\{1\})=2$, $m(\{2\})=3$,
$m(\{3\})=1$, $m(\{1,2\})=4$, $m(\{1,3\})=-1$, $m(\{2,3\})=2$,
$m(\{1,2,3\})=0$.

\begin{table}[ht]
\centering
\small
\caption{Summary of exact symbolic checks on the toy example above.
Every ``predicted'' value is the closed-form output of the cited
theorem/corollary; every ``computed'' value is an independent direct
symbolic evaluation (exact integration, or direct combinatorial sum).
All match to machine (rational) precision.}
\label{tab:toy-verification}
\begin{tabular}{>{\raggedright\arraybackslash}p{4.3cm}>{\raggedright\arraybackslash}p{5.2cm}>{\raggedright\arraybackslash}p{5.2cm}}
\toprule
\textbf{Check} & \textbf{Predicted} & \textbf{Computed (independent)} \\
\midrule
T2 completeness deficit (Cor.~\ref{cor:mobius-correction}) &
  $11-[4\cdot\tfrac12-1\cdot\tfrac12+2\cdot\tfrac12+0\cdot\tfrac23]=8.5$ &
  $\sum_i\phi_i^{\GRALIS}=\tfrac{11}{4}+\tfrac{9}{2}+\tfrac{5}{4}=8.5$ (direct path integration) \\
Within-coalition FTC identity (Lemma~A) &
  $\sum_{j\in S}\mathrm{IG}^S_j=F(x_S)-F(x')$ &
  exact match on all $2^3=8$ coalitions $S$ \\
T4 SIV via Möbius (Prop.~\ref{prop:mobius}) &
  $I_{12}=4,\ I_{13}=-1,\ I_{23}=2$ (Möbius formula) &
  same values from the direct Grabisch--Roubens formula \\
T4 $\mathrm{Sh}(\vG)$ vs.\ $\phi^{\GRALIS}$ (Cor.~\ref{cor:siv-gralis-relation}) &
  term-by-term ratio $1/|T|$ &
  $\mathrm{Sh}(\vG)=\{3.5,6,1.5\}$ (sums to $11$) vs.\ $\phi^{\GRALIS}=\{2.75,4.5,1.25\}$
  (sums to $8.5$); \emph{not} a uniform rescaling (per-feature ratios
  $0.786,0.75,0.833$), but each order-$T$ term matches $1/|T|$ exactly
  (e.g.\ feature 1: $2,2,-0.5\to2,1,-0.25$, ratios $1,\tfrac12,\tfrac12$) \\
T5--T6 Hoeffding/Sobol machinery &
  $\E[F]+\sum_Tf_T(x)=F(x)$; $\Var[F]=\sum_T\Var[f_T]$ &
  $\mu=\mathrm{Unif}[0,1]^3$: $\E[F]=\tfrac{17}{4}$, exact match; $\Var[F]=\tfrac{209}{48}=\sum_T\Var[f_T]$,
  $S_1{=}\tfrac{49}{209},S_2{=}\tfrac{144}{209},S_3{=}\tfrac{9}{209}$ \\
T7 minimum-variance weights (Thm.~\ref{thm:ms}) &
  $\lambda^*_\ell\propto1/\sigma_\ell^2$, $\Var^*=1/\sum_\ell\sigma_\ell^{-2}$ &
  $\sigma^2=(1.0,4.0,0.25)$: $\lambda^*=(0.190,0.048,0.762)$,
  $\Var^*\approx0.190$; beats uniform ($\Var\approx0.583$) and $2000$
  random simplex weightings ($400{,}000$ MC draws) \\
\bottomrule
\end{tabular}
\end{table}

\begin{remark}[Scope of the T5--T6 check]
This verifies the classical Hoeffding/Sobol machinery that
Theorems~\ref{thm:anova}--\ref{thm:sobol} invoke, and is consistent with
it. It does \emph{not} independently verify GRALIS's own convergence
$\Phi_T^{\GRALIS}\to f_T$ as $\pi_x\to\mu$, an asymptotic statement about
the limiting kernel rather than a finite quantity a toy example can check
directly; given the finite-$\pi_x$ attenuation of
Corollary~\ref{cor:mobius-correction}, this convergence is worth a
dedicated numerical study (e.g.\ tracking $|\Phi_T^{\GRALIS}-f_T|$ as
$\sigma\to\infty$) rather than being asserted from the toy example alone.
\end{remark}


\end{document}